\documentclass[12pt,english]{article}
\usepackage[T1]{fontenc}
\usepackage[latin9]{inputenc}
\usepackage{xcolor}
\usepackage{units}
\usepackage{url}
\usepackage{algorithm2e}
\usepackage{amsmath}
\usepackage{amsthm}
\usepackage{amssymb}
\usepackage{stmaryrd}
\usepackage{graphicx}
\usepackage[authoryear,numbers, compress]{natbib}
\PassOptionsToPackage{normalem}{ulem}
\usepackage{ulem}

\makeatletter

\providecolor{lyxadded}{rgb}{0,0,1}
\providecolor{lyxdeleted}{rgb}{1,0,0}

\DeclareRobustCommand{\lyxsout}[1]{\ifx\\#1\else\sout{#1}\fi}

\theoremstyle{plain}
\newtheorem{thm}{\protect\theoremname}
\theoremstyle{plain}
\newtheorem{cor}{\protect\corollaryname}
\theoremstyle{definition}
 \newtheorem{example}{\protect\examplename}
\theoremstyle{plain}
\newtheorem{prop}{\protect\propositionname}
\theoremstyle{plain}
\newtheorem{lem}{\protect\lemmaname}
\theoremstyle{plain}
\newtheorem{assumption}{\protect\assumptionname}
\theoremstyle{remark}
\newtheorem{claim}{\protect\claimname}
\theoremstyle{plain}
\newtheorem{fact}{\protect\factname}
\theoremstyle{remark}
\newtheorem{rem}{\protect\remarkname}

\usepackage[pagebackref]{hyperref}
\usepackage{enumitem}
\usepackage{breqn}
\usepackage{bbm} 
\usepackage[letterpaper, margin=1in]{geometry}
\usepackage{xspace}
\DeclareMathOperator{\Tr}{Tr}  
 \DeclareMathOperator*{\argmin}{argmin}
\DeclareMathOperator*{\argmins}{argmin^*}

\DeclareMathOperator{\sgn}{sgn} 
\DeclareMathOperator{\rank}{rank}
\global\long\def\Dkl{\mathrm{D_{KL}}}
\global\long\def\dtv{\mathrm{d_{{TV}}}}
\global\long\def\precision{\mathrm{precision}}
\global\long\def\recall{\mathrm{recall}}
\global\long\def\P{\mathbb{P}}
\global\long\def\E{\mathbb{E}}
\global\long\def\I{\mathbbm{1}}
\global\long\def\d{\mathrm{d}}

\global\long\def\trre[#1,#2]{\overset{{\scriptstyle (#2)}}{#1}} 

\DeclareRobustCommand\onedot{\futurelet\@let@token\@onedot}
\def\@onedot{\ifx\@let@token.\else.\null\fi\xspace}
\def\eg{\emph{e.g}\onedot} 
\def\ie{\emph{i.e}\onedot}

\def\vs{\emph{vs}\onedot}
\def\wrt{w.r.t\onedot}
\def\WHP{w.h.p\onedot}

\def\rv{r.v\onedot}
\def\WLOG{w.l.o.g\onedot}
\def\RHS{r.h.s\onedot}

\def\IID{i.i.d\onedot}
\allowdisplaybreaks

\makeatother

\usepackage{babel}
\providecommand{\assumptionname}{Assumption}
\providecommand{\claimname}{Claim}
\providecommand{\corollaryname}{Corollary}
\providecommand{\examplename}{Example}
\providecommand{\factname}{Fact}
\providecommand{\lemmaname}{Lemma}
\providecommand{\propositionname}{Proposition}
\providecommand{\remarkname}{Remark}
\providecommand{\theoremname}{Theorem}

\begin{document}
\title{Statistical curriculum learning: An elimination algorithm achieving
an oracle risk}
\author{Omer Cohen, Ron Meir and Nir Weinberger\thanks{The authors are with the Department of Electrical and Computer Engineering,
Technion -- Israel Institute of Technology. Emails: \{\texttt{omercohen7640@gmail.com},
\texttt{rmeir@ee.technion.ac.il},\texttt{ nirwein@technion.ac.il}\}.
The work of RM was partially supported by the Israel Science Foundation
grant 1693/22, by the Skillman chair in biomedical sciences and by
the Ollendorf Center of the Viterbi Faculty of Electrical and Computer
Engineering at the Technion. The work of NW was supported by the Israel
Science Foundation (ISF), grant no. 1782/22.}}
\date{\today}
\maketitle
\begin{abstract}
We consider a statistical version of curriculum learning (CL) in a
parametric prediction setting. The learner is required to estimate
a target parameter vector, and can adaptively collect samples from
either the target model, or other source models that are similar to
the target model, but less noisy. We consider three types of learners,
depending on the level of side-information they receive. The first
two, referred to as strong/weak-oracle learners, receive high/low
degrees of information about the models, and use these to learn. The
third, a fully adaptive learner, estimates the target parameter vector
without any prior information. In the single source case, we propose
an elimination learning method, whose risk matches that of a strong-oracle
learner. In the multiple source case, we advocate that the risk of
the weak-oracle learner is a realistic benchmark for the risk of adaptive
learners. We develop an adaptive multiple elimination-rounds CL algorithm,
and characterize instance-dependent conditions for its risk to match
that of the weak-oracle learner. We consider instance-dependent minimax
lower bounds, and discuss the challenges associated with defining
the class of instances for the bound. We derive two minimax lower
bounds, and determine the conditions under which the performance weak-oracle
learner is minimax optimal.

\end{abstract}

\section{Introduction}

\emph{Curriculum learning }(CL) \citep{bengio2009curriculum} is a
machine learning paradigm that is successfully utilized in various
applications, \eg, weakly supervised object localization, object
detection, neural machine translation, and reinforcement learning
\citep{soviany2022curriculum,wang2021survey}. The distinctive feature
of CL is the ability to control the learning \emph{order}, \eg, to
choose which task to begin with, to begin from easy training examples,
or to begin from a smooth loss function. This leads to an \emph{optimization
benefit.}\textbf{ }In parallel, various\textbf{ }\emph{lifelong learning
}methods also involve learning from multiple tasks, but aim to achieve
a \emph{statistical benefit}, \eg, \emph{multi-task learning, transfer
learning}, and c\emph{ontinual learning }\citep{caruana1997multitask,ruder2017overview,zhang2021survey,crawshaw2020multi,zenke2017continual,van2019three,de2021continual},
but without the challenging aspect of optimization of the order.\textbf{
}Theoretical aspects of these methods were explored in, \eg \citep{baxter2000model,argyriou2006mulfti,maurer2016benefit,du2020few,tripuraneni2020theory,tripuraneni2021provable,evron2023continual,goldfarb2023analysis,li2023fixed,lin2023theory,evron2024joint}.
In comparison, CL was much less explored \citet{weinshall2018curriculum,hacohen2019power,weinshall2020theory,saglietti2022analytical,cornacchia2023mathematical},
and much less from the statistical aspect. We survey these and other
related works in detail in Appendix \ref{sec:Related-work}.\looseness=-1

Recently, \citet{xu2022statistical} advocated the \emph{statistical
benefits} of CL, by considering a setting in which $T$ source models
indexed by $t\in[T]:=\{1,2,\ldots,T\}$ are available to the learner,
and model $t=0$ is a \emph{target} task to be learned. The learner
has a budget of $N$ samples for all models. Each model can be rated
by both its \emph{similarity} to the target task $t=0$, and its \emph{statistical
difficulty}. A source model is useful for learning the target task
if it is both similar to the target task and statistically easier
to learn than the target task. Since the similarity of the tasks (and
sometimes the difficulty) are unknown in advance, the algorithmic
challenge is to set up a curriculum that optimally utilizes the source
tasks to learn the target task. In this paper, we focus on a Gaussian
mean estimation problem in $\mathbb{R}^{d}$, in which similarity
between models is measured by the distance between parameter vectors,
and difficulty by the noise variance. Since the mean parameters are
unknown in advance, so is the similarity between them. \citet{xu2022statistical}
introduced a \emph{strong-oracle}, which reveals to a learner the
source model that optimally balances between similarity and difficulty.
They proved a lower bound on its risk \citep[Theorem 1]{xu2022statistical},
which is also a lower bound on the risk of any CL algorithm (which
operates without this knowledge).

\paragraph{Contributions}

We address both the algorithmic and statistical challenges of CL,
and show that even in a parametric learning setting, inherent challenges
and intriguing aspects are encountered:
\begin{itemize}
\item We show that the risk of the strong-oracle learner can be achieved
for $T=1$. Interestingly, the learner does not know that its risk
is \emph{strictly} better than the risk obtained when just using samples
from the target model, even if this is the case.
\item We introduce a notion of a \emph{weak-oracle learner,} and postulate
that the risk of an ideal CL algorithm should match the risk of the
weak-oracle learner when $T>1$. 
\item We develop an adaptive\emph{ multiple-rounds source-elimination }CL
algorithm, where in each round source models are eliminated based
on estimated similarity. The possible elimination of source models
allows for improved similarity estimation in the following round,
and thus the elimination of additional source models. We introduce
the \emph{elimination curve}, which determines whether the risk of
the weak-oracle learner will be achieved by the algorithm. We reveal
that the risk is \emph{not} monotonic in the similarity of the source
models to the target model. 
\item We derive two lower bounds on the minimax risk of any CL algorithm.
We reveal \emph{intrinsic} hurdles in choosing a set of problem instances
of \emph{homogeneous} difficulty for the minimax bound. The first
bound is for the set proposed by \citet[Theorem 2]{xu2022statistical}
and strictly improves it, and the second bound is for a newly introduced
set, and matches the risk of the weak-oracle learner.
\end{itemize}

\paragraph*{Paper outline}

In Section \ref{sec:setting}, we formulate the CL problem for general
learning problems, and then focus on the mean estimation parametric
setting. In Section \ref{sec:The-similarity=002013noise-variance-tr}
we introduce the the similarity--difficulty balance and present the
strong-oracle learner. We consider the case of a single source model,
develop an elimination method, and show that its risk essentially
matches to that of the strong-oracle learner (Theorem \ref{thm: Single source CL algorithm}).
We then introduce the weak-oracle learner as a plausible benchmark
for CL algorithms. In Section \ref{sec:A-CL-elimination-algorithm},
we present the main algorithmic contribution and theoretical guarantee.
We develop a multiple elimination rounds CL algorithm (Algorithm \ref{alg: CL multiple sources})
and upper bound its risk (Theorem \ref{thm: Multiple source CL algorithm}).
In Section \ref{sec:Minimax-risk-lower} we state the minimax lower
bounds (Theorem \ref{thm: lower bound global risk constant dimension}
and Theorem \ref{thm: lower local risk high dimension}). In Section
\ref{sec:Summary-and-open}, we summarize the paper and present a
few directions for future research.

\section{Problem setting \label{sec:setting}}

\paragraph*{The statistical curriculum learning problem}

We use standard notation conventions outlined in Appendix \ref{sec:Notation-conventions}.
Let $(\Omega,\mathfrak{B}(\Omega))$ be a Borel space, $\Phi_{t}$
an index set, and ${\cal M}_{t}:=\{P_{\phi_{t}}\}_{\phi_{t}\in\Phi_{t}}$
a statistical model, where $\phi_{t}$ is an index in $\Phi_{t}$
and $P_{\phi_{t}}$ is a probability measure on the Borel space. Let
${\cal M}_{0}$ be the \emph{target model}, and let $\{{\cal M}_{t}\}_{t\in[T]}$
be $T$ \emph{source models}. Let $\boldsymbol{\Phi}:=\Phi_{0}\times\Phi_{1}\times\cdots\Phi_{T}$
be the product set of the index sets, and $\boldsymbol{\phi}:=(\phi_{0},\phi_{1},\ldots,\phi_{T})\in\boldsymbol{\Phi}$
be the collection of the indices of the $T+1$ models. The learner
may adaptively collect samples from each of these models, up to a
total of $N$ samples. To wit, the learner chooses the next model
to sample from based on previous model choices and observations. Let
$A_{i}\in\llbracket T\rrbracket:=\{0\}\cup[T]$ denote the model index
chosen by the learner for the $i$th sample, and let $S_{i}$ be the
$i$th sample. Thus, the observations made by the learner are $H_{N}:=(A_{1},S_{1},A_{2},S_{2},\ldots,A_{N},S_{N})$.
We adopt the \emph{random table model }from the closely-related \emph{multi-armed
bandit} problem \citep[Section 4.6]{lattimore2020bandit}, in which
$N\cdot(T+1)$ independent \rv's $\boldsymbol{Z}:=\{Z_{i,t}\}_{i\in[N],t\in\llbracket T\rrbracket}$
are defined over the Borel space $({\cal Z}^{N(T+1)},\mathfrak{B}({\cal Z}^{N(T+1)}))$,
so that the law of $Z_{i,t}$ is $P_{\phi_{t}}$. The $i$th sample
collected by the learner is $S_{i}=Z_{i,A_{i}}$, where $A_{i}$ is
$\sigma(H_{i-1})$-measurable, \ie, $A_{i}$ only depends on the
\emph{history} $H_{i-1}$, and $S_{i}$ is sampled from the model
${\cal M}_{A_{i}}$. The policy $\pi_{i}$ of the learner for the
$i$th sample is the distribution of $A_{i}$ conditioned on $H_{i-1}$,
and its \emph{policy} is $\pi:=(\pi_{i})_{i\in[N]}$. We consider
the well-specified setting, in which the learner produces an hypothesis
$\hat{\phi}\colon((\llbracket T\rrbracket)\times{\cal Z})^{N}\to\Phi_{0}$
based on the collected samples $(A_{i},S_{i})_{i\in[N]}$. Let $\ell\colon\Phi_{0}\times\Phi_{0}\to\mathbb{R}_{+}$
be a loss function. A CL algorithm is ${\cal A}:=(\pi,\hat{\phi})$,
and its excess risk of after collecting $N$ samples when the parameters
are $\boldsymbol{\phi}$ is 
\[
L_{N}(\boldsymbol{\phi},{\cal A}):=\E_{\boldsymbol{\phi},{\cal A}}\left[\ell\left(\hat{\phi}\left((A_{i},S_{i})_{i\in[N]}\right),\phi_{0}\right)-\ell\left(\phi_{0},\phi_{0}\right)\right],
\]
where the expectation is \wrt the randomness of the samples $\boldsymbol{Z}$,
the policy, and possibly the estimator. The \emph{minimax excess risk}
of $\boldsymbol{\Psi}\subseteq\boldsymbol{\Phi}$ is given by (with
a slight abuse of notation)

\begin{equation}
L_{N}(\boldsymbol{\Psi}):=\inf_{{\cal A}}\sup_{\boldsymbol{\phi}\in\boldsymbol{\Psi}}L_{N}(\boldsymbol{\phi},{\cal A}).\label{eq: minimax risk adaptive}
\end{equation}
The CL problem is to derive tight bounds on $L_{N}(\boldsymbol{\Psi})$,
and a develop CL algorithms that achieve it. 

\paragraph*{The mean estimation statistical CL problem}

In order to focus on the intrinsic aspects of the statistical CL problem,
we focus on a simple parametric mean estimation problem, prototypical
to parametric learning problems. Specifically, the sample space is
${\cal Z}=\mathbb{R}^{d}$ and $\Phi_{t}:=\mathbb{R}^{d}\times\mathbb{R}_{+}\times\overline{\mathbb{S}}_{d}^{++}$,
where $\overline{\mathbb{S}}_{d}^{++}$ is the positive semidefinite
cone of matrices whose trace is normalized to $d$, and the parameters
are $\phi_{t}:=(\theta_{t},\sigma_{t}^{2},\overline{\Sigma}_{t})$.
The $t$th model, $t\in\llbracket T\rrbracket$, is given by 
\begin{equation}
{\cal M}_{t}\colon\quad Y_{t}=\theta_{t}+\epsilon_{t},\label{eq: mean estimation model}
\end{equation}
where $\theta_{t}\in\mathbb{\mathbb{R}}^{d}$ is an unknown parameter
vector and $\epsilon_{t}\sim{\cal N}(0,\sigma_{t}^{2}\cdot\overline{\Sigma}_{t})$
is a Gaussian noise with $\Tr[\overline{\Sigma}_{t}]=d$. The goal
of the learner is to estimate the \emph{target parameter} $\phi_{0}$
under the squared error loss function $\ell(\phi_{0},\phi)=\|\theta_{0}-\theta\|^{2}$,
that is, the parameters $(\sigma_{t}^{2},\overline{\Sigma}_{t})$
are \emph{nuisance}. Thus, the excess risk is $L_{N}(\boldsymbol{\phi},{\cal A})=\E[\|\hat{\theta}-\theta_{0}\|^{2}]$.
We assume for the most part that $\{(\sigma_{t}^{2},\overline{\Sigma}_{t})\}_{t\in\llbracket T\rrbracket}$
are known to the learner in advance (and extend to unknown $\{(\sigma_{t}^{2},\overline{\Sigma}_{t})\}_{t\in\llbracket T\rrbracket}$
in Appendix \ref{sec:unknown noise covariance}). The CL algorithm
will be denoted by ${\cal A}=(\hat{\theta},\pi)$, where $\hat{\theta}$
is an estimator for $\theta_{0}$. We assume \WLOG that $\sigma_{t}^{2}<\sigma_{0}^{2}$
for all $t\in[T]$, since otherwise there is no reason for the learner
to sample from the $t$th source model -- it is both mismatched,
and has larger noise variance than the target model. We also make
the mild assumption that $\max_{t\in\llbracket T\rrbracket}\|\theta_{t}\|\leq C_{\theta}$,
for some constant $C_{\theta}>0$. As is well-known (\eg, \citet[Theorem 2.2]{rigollet2019high}),
estimators based on $N$ \IID samples from a parametric model lead
to a \emph{parametric error rate}. Specifically, there exists an estimator
$\overline{\theta}(N)$ and $g(\delta)\colon[0,1]\to\mathbb{R}_{+}$
so that for any $\delta\in(0,1)$, 
\begin{equation}
\|\bar{\theta}(N)-\theta\|^{2}\leq g(\delta)\cdot\frac{d\sigma^{2}}{N}\label{eq: parametric error rate definition of c_delta}
\end{equation}
holds with probability at least $1-\delta$. For the model \eqref{eq: mean estimation model},
the empirical mean estimator achieves that with $g(\delta)=c\log(\nicefrac{e}{\delta})$
for some universal constant $c>1$ (see Lemma \ref{lem: high probability MSE for empirical average}
in Appendix \ref{sec:High-probability-bounds}). 

\citet{xu2022statistical} considered the closely-related linear regression
problem, as we describe in Appendix \ref{sec:The-linear-regression}.
We show that under the assumptions on the covariance matrix of the
covariates made by \citet{xu2022statistical} (Assumption \ref{assu: condition number on Sigma}
in Appendix \ref{sec:The-linear-regression}), accurate estimation
of the parameter vector is necessary for low risk (and not just sufficient).
We then show using the small-ball method of \citet{koltchinskii2015bounding},
that \eqref{eq: parametric error rate definition of c_delta} holds
for the projected least squares (LS) estimator with $g(\delta)=\Theta(\log^{2}(\delta))$,
under a mild condition on $(N,d)$. So, under Assumption \ref{assu: condition number on Sigma},
the mean estimation and linear regression problems are not very different,
and we thus focus on the former, which is simpler. Our results hold
almost verbatim for the latter, with slight adjustments. 

Finally, whenever we invoke $\bar{\theta}(K)$ for some $K\in[N]$,
we will assume $K$ fresh \IID samples from the model, in the sense
that they are independent of all other samples previously taken from
that specific model and any other model (source or target). This is
mainly assumed for simplicity of exposition and proofs, and can be
avoided in practice. 

\section{The similarity--difficulty balance and oracle-based learners \label{sec:The-similarity=002013noise-variance-tr}}

The gain in estimating $\theta_{0}$ based on samples from a source
model is related to their similarity and their difficulty. For parametric
models, similarity can be gauged by the distances $\{Q_{t}\}_{t\in[T[}$
where $Q_{t}:=\|\theta_{t}-\theta_{0}\|$ (and $Q_{0}:=0$), and difficulty
by the noise variances $\{\sigma_{t}^{2}\}_{t\in\llbracket T\rrbracket}$.
Assuming, for now, that $\{\sigma_{t}^{2}\}_{t\in\llbracket T\rrbracket}$
are known to the learner, we follow \citet{xu2022statistical}, and
consider a \emph{strong oracle} that knows $\{Q_{t}\}_{t\in[T]}$.
We later contrast it with a more realistic \emph{weak oracle}. 

\paragraph*{Strong-oracle learners and their loss}

Let $\bar{\theta}_{t}(N)$ be the empirical mean estimator for $\theta_{t}$
using ${\cal M}_{t}$. Then, with high probability (\WHP), its loss
for estimating $\theta_{0}$ is upper bounded as
\begin{equation}
\|\bar{\theta}_{t}(N)-\theta_{0}\|^{2}=\|\bar{\theta}_{t}(N)-\theta_{t}+\theta_{t}-\theta_{0}\|^{2}\lesssim\|\bar{\theta}_{t}(N)-\theta_{t}\|^{2}+\|\theta_{t}-\theta_{0}\|^{2}\lesssim\frac{d\sigma_{t}^{2}}{N}+Q_{t}^{2}\label{eq: MSE of source model}
\end{equation}
(which is typically tight in high dimensions, \eg, \citet[Remark 3.2.5]{vershynin2018high}).
Following \citet[Theorem 1]{xu2022statistical}, we define a strong-oracle
learner as one that knows the task that minimizes this upper bound,
\ie, $t^{*}\in\argmin_{t\in\llbracket T\rrbracket}\{\nicefrac{d\sigma_{t}^{2}}{N}+Q_{t}^{2}\}$,
and thus knows the optimal trade-off between similarity and difficulty.
Such a learner can allocate all its $N$ samples to the model ${\cal M}_{t^{*}}$
(though it does not have to do so), and estimate $\theta_{0}$ as
$\bar{\theta}_{t^{*}}(N)$. This results in 
\begin{equation}
\|\bar{\theta}_{t^{*}}(N)-\theta_{0}\|^{2}\lesssim\min_{t\in\llbracket T\rrbracket}\left\{ \frac{d\sigma_{t}^{2}}{N}+Q_{t}^{2}\right\} .\label{eq: strong oracle}
\end{equation}
We refer to the \RHS of \eqref{eq: strong oracle} as the loss of
the \emph{strong-oracle learner}. \citet[Theorem 1]{xu2022statistical}
showed that this loss is a lower bound on the risk of any CL algorithm
(see Appendix \ref{sec:Comments-on-Xu-Tewari} for a detailed comparison
with \citet{xu2022statistical}). Since we focus on non-asymptotic
analysis, we define for a constant $\kappa>0$ the \emph{strong oracle
set (of source models)} as
\begin{equation}
{\cal T}_{\text{s.o.}}\equiv{\cal T}_{\text{s.o.}}(\kappa):=\left\{ t\in\llbracket T\rrbracket\colon Q_{t}^{2}+\frac{d\sigma_{t}^{2}}{N}\leq\kappa\cdot\left[\frac{d\sigma_{t^{*}}^{2}}{N}+Q_{t^{*}}^{2}\right]\right\} ,\label{eq: T strong oracle set}
\end{equation}
so that any $t\in{\cal T}_{\text{s.o.}}(\kappa)$ is essentially as
efficient as $t^{*}$. A natural question is whether the strong-oracle
set can be identified by a learner that samples from $\{{\cal M}_{t}\}_{t\in\llbracket T\rrbracket}$,
without prior knowledge of $\{Q_{t}\}_{t\in[T]}$. If this is possible
then \eqref{eq: strong oracle} can be achieved by a CL algorithm
that splits half of the samples for finding ${\cal T}_{\text{s.o.}}(\kappa)$,
and the rest for sampling from ${\cal M}_{\tilde{t}}$ with $\tilde{t}\in{\cal T}_{\text{s.o.}}(\kappa)$
and estimating $\theta_{0}$ as $\bar{\theta}_{\tilde{t}}(\nicefrac{N}{2})$.
We will argue that it is in generally impossible to identify ${\cal T}_{\text{s.o.}}(\kappa)$
from samples. We will propose a larger set of models, the \emph{weak-oracle
set}, and explain why identifying this set is a more realistic goal
for CL algorithms. To motivate the weak oracle, we first address the
simple case of $T=1$.

\paragraph*{A single source model and the source elimination lemma}

When $T=1$, a CL algorithm only needs to decide whether $\nicefrac{d\sigma_{1}^{2}}{N}+Q_{1}^{2}\lesssim\nicefrac{d\sigma_{0}^{2}}{N}$
or not. Since $\sigma_{1}^{2}<\sigma_{0}^{2}$, this is equivalent
to whether $Q_{1}^{2}\lesssim\nicefrac{d\sigma_{0}^{2}}{N}$ or not,
so that if $Q_{1}^{2}$ is large then $\theta_{0}$ and $\theta_{1}$
are too different, and samples from ${\cal M}_{0}$ should be used
to estimate $\theta_{0}$, and otherwise from ${\cal M}_{1}$. To
decide this, the learner can estimate $Q_{1}^{2}=\|\theta_{0}-\theta_{1}\|^{2}$
based on initial estimates $\bar{\theta}_{0}(\nicefrac{N}{2})$ and
$\bar{\theta}_{1}(\nicefrac{N}{2})$, for which it holds \WHP that
\begin{equation}
\hat{Q_{1}^{2}}:=\|\bar{\theta}_{0}(\nicefrac{N}{2})-\bar{\theta}_{1}(\nicefrac{N}{2})\|^{2}\lesssim\frac{d\sigma_{1}^{2}}{N}+Q_{1}^{2}+\frac{d\sigma_{0}^{2}}{N}\asymp Q_{1}^{2}+\frac{d\sigma_{0}^{2}}{N}.\label{eq: MSE of Q1 estimation}
\end{equation}
The learner then decides that if $\hat{Q_{1}^{2}}\gtrsim\nicefrac{d\sigma_{0}^{2}}{N}$
then so is $Q_{1}^{2}\gtrsim\nicefrac{d\sigma_{0}^{2}}{N}$ and chooses
${\cal M}_{0}$. The source model ${\cal M}_{1}$ is thus \emph{eliminated},
and the final estimate is $\hat{\theta}=\bar{\theta}_{0}(\nicefrac{N}{2})$.
Otherwise, $\hat{\theta}=\bar{\theta}_{1}(\nicefrac{N}{2})$ and the
loss is $O(Q_{1}^{2}+\nicefrac{d\sigma_{1}^{2}}{N})$, which is guaranteed
to be \emph{not larger} (orderwise) than the loss of the target model
estimate, given by $O(\nicefrac{d\sigma_{0}^{2}}{N})$. Optimistically,
it \emph{might} also hold that the estimation loss is \emph{much smaller}
when using the source model, that is, $Q_{1}^{2}+\nicefrac{d\sigma_{1}^{2}}{N}=o(\nicefrac{d\sigma_{0}^{2}}{N})$.
However, the learner cannot know this, since it does not have any
knowledge of $\theta_{0}$ at a resolution smaller than $\sqrt{\nicefrac{d\sigma_{0}^{2}}{N}}$.
The following theorem is a rigorous theoretical guarantee on the loss
and risk of this method. 
\begin{thm}
\label{thm: Single source CL algorithm}Let $\tilde{\theta}_{0}=\bar{\theta}_{0}(\nicefrac{N}{2})$
(resp. $\tilde{\theta}_{1}=\bar{\theta}_{1}(\nicefrac{N}{2})$) be
an initial estimate of $\theta_{0}$ using $N/2$ \IID samples from
the target model ${\cal M}_{0}$ (resp. source model ${\cal M}_{1}$).
Let $\delta\in(0,1)$ be given, and let 
\begin{equation}
\hat{\theta}=\begin{cases}
\tilde{\theta}_{0}, & \|\tilde{\theta}_{0}-\tilde{\theta}_{1}\|^{2}\geq10g(\delta)\frac{d\sigma_{0}^{2}}{\nicefrac{N}{2}}\\
\tilde{\theta}_{1}, & \text{\emph{otherwise}}
\end{cases}\label{eq: choice of final estimate for single source algorithm}
\end{equation}
be the final estimate of $\theta_{0}$. Then, there exists $\nu\in[\nicefrac{1}{27},1)$
such with probability at least $1-\delta$
\begin{equation}
\|\hat{\theta}-\theta_{0}\|^{2}\leq\frac{8g(\nicefrac{\delta}{2})}{\nu}\cdot\min\left\{ Q_{1}^{2}+\frac{d\sigma_{1}^{2}}{N},\frac{d\sigma_{0}^{2}}{N}\right\} .\label{eq: single source high probability bound}
\end{equation}
Assume further that $Q_{1}^{2}\leq4C_{\theta}^{2}$ for some constant
$C_{\theta}>0$. Then, setting in \eqref{eq: choice of final estimate for single source algorithm}
\[
\delta=\delta_{*}:=\left[\frac{\sigma_{1}^{2}}{\sigma_{0}^{2}+\sigma_{1}^{2}}\vee\frac{1}{4}\vee\frac{d\sigma_{0}^{2}}{8C_{\theta}^{2}N}\right]^{2}
\]
results in
\begin{equation}
\E\left[\|\hat{\theta}-\theta_{0}\|^{2}\right]\leq\frac{28c}{\nu}\log\left(\frac{2}{\delta_{*}}\right)\cdot\min\left\{ Q_{1}^{2}+\frac{d\sigma_{1}^{2}}{N},\frac{d\sigma_{0}^{2}}{N}\right\} .\label{eq: single sources bound in expectation}
\end{equation}
\end{thm}
Remarkably, the high-probability loss achieved by this method has
the same order as the loss achieved by the strong-oracle learner in
\eqref{eq: strong oracle}, and their risk nearly match, up to a $\Theta(\log(\nicefrac{N}{d}))$
factor. 

The proof of Theorem \ref{thm: Single source CL algorithm} appears
in Appendix \ref{sec:Proofs-for-single-source-model}, and is based
on a simple, yet crucial, Lemma \ref{lem: iden}, called the \emph{source
elimination lemma}. This lemma establishes that the learner can eliminate
the worse of two models, with a constant loss factor (though larger
than $2$). Its proof is a consequence of the convexity and the approximate
triangle-inequality of the squared Euclidean norm loss function. An
equivalent rephrase of Lemma \ref{lem: iden}, which will be used
for $T>1$, is as follows: 
\begin{cor}[to the source elimination lemma, Lemma \ref{lem: iden}]
\label{cor: source identification}Suppose an algorithm eliminates
model ${\cal M}_{t}$ if $\|\tilde{\theta}_{0}-\tilde{\theta}_{t}\|^{2}\geq10\lambda^{2}$
for some $\lambda>0$. Then, there exists a numerical constant $\nu\in[\nicefrac{1}{27},1)$
so that if $Q_{t}^{2}\geq\nicefrac{\lambda^{2}}{\nu}$ then ${\cal M}_{t}$
will be eliminated \WHP.
\end{cor}
Letting $\lambda_{t}^{2}\asymp\nicefrac{d\sigma_{t}^{2}}{N}$, we
use this corollary with $\lambda^{2}=\lambda_{\text{max}}^{2}:=\lambda_{0}^{2}\vee\lambda_{1}^{2}\asymp\nicefrac{d\sigma_{0}^{2}}{N}$.
The proof of the upper bound on the risk (expected loss) in \eqref{eq: single sources bound in expectation}
is more technical, and is obtained by controlling of the fourth-moment
of the error. 

\paragraph*{Weak-oracle learners and their loss}

As is reasonably intuitive, and as proved by \citet{david2010impossibility},
transfer learning from a source to the target is impossible without
samples from the target model. In our method, $\nicefrac{N}{2}$ samples
(say) were allocated to the target model. However, $Q_{t}^{2}$ cannot
be estimated at a resolution smaller than $\Theta(\sqrt{\nicefrac{d\sigma_{0}^{2}}{N}})$,
even using $N$ samples (see \eqref{eq: MSE of Q1 estimation}). Thus,
even if the learner is informed with the models for which $Q_{t}^{2}\lesssim\nicefrac{d\sigma_{0}^{2}}{N}$,
it cannot further discriminate between them. Therefore, allocating
$\Theta(N)$ samples to the source model with minimal $Q_{t}^{2}+\nicefrac{d\sigma_{t}^{2}}{N}$,
as can be done using a strong oracle, cannot be guaranteed without
prior knowledge of $\{Q_{t}\}_{t\in[T]}$.

Hence, it is reasonable that a CL algorithm will be able to match
the performance of a \emph{weak-oracle learner}.\emph{ }The weak oracle
only knows which source models provide more accurate samples than
the target model itself, that is, for some constant $\kappa>0$, the
models in the set
\begin{equation}
{\cal T}_{\text{w.o.}}\equiv{\cal T}_{\text{w.o.}}(\kappa):=\left\{ t\in[T]\colon Q_{t}^{2}\leq\kappa\cdot\frac{d\sigma_{0}^{2}}{N}\right\} .\label{eq: weak oracle set}
\end{equation}
The weak-oracle learner will choose one of the source models in ${\cal T}_{\text{w.o.}}$
for its final estimate (if ${\cal T}_{\text{w.o.}}=\emptyset$ then
it will simply use the target model). If $|{\cal T}_{\text{w.o.}}|>1$,
then the tightest uniform bound over all problem instances with the
same ${\cal T}_{\text{w.o.}}$ is obtained by selecting the one with
minimal noise variance. Concretely, for any given set ${\cal T}\subseteq[T]$,
let
\begin{equation}
\overline{t}({\cal T}):=\begin{cases}
0, & {\cal T}=\emptyset\\
\argmins_{t\in{\cal T}}\sigma_{t}^{2}, & |{\cal T}|\geq1
\end{cases},\label{eq: choise of model within oracle set}
\end{equation}
where $\argmins$ pessimistically chooses the model with the maximal
$Q_{t}^{2}$ within the set $\argmin_{t\in{\cal T}}\sigma_{t}^{2}$
(thus $\overline{t}({\cal T})$ also depends on $\{Q_{t}^{2}\}_{t\in{\cal T}}$,
but we omit this from the notation for brevity). Given ${\cal T}_{\text{w.o.}}$,
the weak-oracle learner collects $N$ samples from $\overline{t}({\cal T}_{\text{w.o.}})$,
and its final estimate is then $\hat{\theta}_{\text{w.o.}}=\overline{\theta}_{\overline{t}({\cal T}_{\text{w.o.}})}(N)$.
As in \eqref{eq: MSE of source model}, the loss is bounded with \WHP
as
\begin{equation}
\|\hat{\theta}_{\text{w.o.}}-\theta_{0}\|^{2}\lesssim\frac{d\sigma_{\overline{t}({\cal T}_{\text{w.o.}})}^{2}}{N}+Q_{\overline{t}({\cal T}_{\text{w.o.}})}^{2}.\label{eq: weak oracle}
\end{equation}
Since $\Theta(N)$ samples allow for $Q_{t}$ to be estimated with
resolution $\Theta(\sqrt{\nicefrac{d\sigma_{0}^{2}}{N}})$, it is
plausible that ${\cal T}_{\text{w.o.}}$ can be identified from samples
without knowing $\{Q_{t}^{2}\}_{t\in[T]}$. We next explore this possibility.\textbf{ }

We finally remark that if $\sigma_{t}^{2}\equiv\sigma^{2}$ for all
source models, $t\in[T]$ one may argue that it is better to uniformly
``combine the advice'' of the models in ${\cal T}_{\text{w.o.}}$,
and to choose the final estimator to be the average of $\hat{\theta}_{\text{w.o.,avg}}:=\frac{1}{|{\cal T}_{\text{w.o.}}|}\sum_{t\in{\cal T}_{\text{w.o.}}}\overline{\theta}_{t}(\nicefrac{N}{|{\cal T}_{\text{w.o.}}|})$.
Due to the convexity of the squared error, this reduces the loss,
however, it does not lead to a reduced upper bound on the loss (since
it is possible that $\theta_{t}$ is the same for all $t\in{\cal T}_{\text{w.o.}}$).

\section{A CL elimination algorithm and risk upper bound \label{sec:A-CL-elimination-algorithm}}

As for $T=1$, source models can be eliminated by estimation of $Q_{t}^{2}$,
but each estimate may be inaccurate, since there are possibly many
sources ($T>1$) and sampling all of them is necessary. For example,
if one uniformly allocate $\Theta(\nicefrac{N}{T})$ samples to each
of the source models, the number of samples will be smaller by a factor
of $T$ compared to the $T=1$ case. However, if some source models
have already been eliminated, the learner can re-estimate $Q_{t}^{2}$
only for the retained models, and obtain more accurate estimates that
will possibly allow to eliminate additional source models. We thus
propose using \emph{multiple} elimination rounds, as depicted in Algorithm
\ref{alg: CL multiple sources}. In this algorithm the average noise
variances in a set of models ${\cal T}\subseteq[T]$ is denoted as
$\overline{\sigma}^{2}({\cal T}):=\frac{1}{|{\cal T}|}\sum_{t\in{\cal T}}\sigma_{t}^{2}$. 

\RestyleAlgo{ruled}

\begin{algorithm}
\SetAlgoLined

\SetNlSty{texttt}{(}{)}

\SetKwInOut{Input}{input}

\Input{($N,d,\{\sigma_{t}^{2}\}_{t\in\llbracket T\rrbracket},\bar{r},\delta$)}

\nl \textbf{\label{line: algorithm delta_bar}}Set $\bar{N}\leftarrow\nicefrac{N}{(\bar{r}+2)}$
and $\bar{\delta}\leftarrow\nicefrac{\delta}{(T\bar{r}+2)}$\;

Compute  $\tilde{\theta}_{0}=\bar{\theta}_{0}(\bar{N})$ by sampling
$\bar{N}$ samples from the target model ${\cal M}_{0}$\;

Set  ${\cal T}_{0}\leftarrow[T]$ and $T_{0}\leftarrow T$ \;

\For{$r=1$ \emph{to} $\bar{r}$}{

Set ${\cal T}_{r}\leftarrow{\cal T}_{r-1}$\;

\For{$t\in\mathcal{T}_{r-1}$}{

\nl\textbf{\label{line: algorithm sample allocation}}Compute $\tilde{\theta}_{t}=\bar{\theta}_{t}(\bar{N}_{t,r})$
by sampling $\bar{N}_{t,r}=\frac{\bar{N}}{T_{r-1}}\cdot\frac{\sigma_{t}^{2}}{\overline{\sigma}^{2}({\cal T}_{r-1})}$
samples from ${\cal M}_{t}$\;

\nl\textbf{\label{line: algorithm elimination rule}}\If{ $\|\tilde{\theta}_{0}-\tilde{\theta}_{t}\|^{2}\geq10g(\bar{\delta})\cdot\left[\frac{d\overline{\sigma}^{2}({\cal T}_{r-1})}{\nicefrac{\bar{N}}{T_{r-1}}}\vee\frac{d\sigma_{0}^{2}}{\bar{N}}\right]$
}{

Set ${\cal T}_{r}\leftarrow{\cal T}_{r}\backslash\{t\}$\;

}

}

Set ${\cal T}_{\text{alg}}\leftarrow{\cal T}_{r}$ and $T_{r}\leftarrow|{\cal T}_{r}|$\;

\nl\textbf{\label{line: stopping criterion}}\If{$T_{r-1}\leq\nicefrac{\sigma_{0}^{2}}{\overline{\sigma}^{2}({\cal T}_{r-1})}$
\emph{or} $T_{r}=0$ \emph{or} $T_{r}=T_{r-1}$}{

Break out of the for loop\;

}

}

\nl\textbf{\label{line: algorithm elimination final output}}Return
$\hat{\theta}=\overline{\theta}_{t^{*}}(\bar{N})$ for some $t^{*}\in\overline{t}({\cal T}_{\text{alg}})$\;

\caption{A CL algorithm for a multiple source models}\label{alg: CL multiple sources}
\end{algorithm}

\paragraph*{Description of Algorithm \ref{alg: CL multiple sources}}

We assume that $N\gtrsim T$, so that all source models can be explored.
The algorithm will use at most $T\overline{r}+2$ empirical mean estimates
($T$ for each of the $\overline{r}$ rounds, one for the initial
target estimate, and one for the final estimate). So, setting $\bar{\delta}:=\nicefrac{\delta}{(T\bar{r}+2)}$
and requiring each estimate to be accurate with probability $1-\overline{\delta}$,
assures that all estimates are accurate with probability $1-\delta$
(via the union bound). We next assume that this high probability event
holds. Let $\overline{r}$ be a total number of elimination rounds
and let $\bar{N}=\nicefrac{N}{(\bar{r}+2)}$ be the number of samples
per round (plus an initial and a final round). The learner first uses
$\overline{N}$ samples from the target model, to obtain an initial
estimate $\tilde{\theta}_{0}=\overline{\theta}_{0}(\overline{N})$,
and then starts to eliminate models. Let ${\cal T}_{r}\subseteq[T]$
be the set of retained source models \emph{after }the $r$th round
of elimination, initialized with ${\cal T}_{0}:=[T]$, and denote
$T_{r}:=|{\cal T}_{r}|$. Thus, ${\cal T}_{r}\subseteq{\cal T}_{r-1}$
and $T_{r}\leq T_{r-1}$. At each elimination round, $\bar{N}_{t,r}$
samples are allocated to the source models in ${\cal T}_{r}$, which
are used to estimate $\theta_{t}$ as $\bar{\theta}_{t}(\bar{N}_{t,r})$,
and then to estimate $Q_{t}^{2}$. Based on the estimated $Q_{t}^{2}$,
it is decided (line \ref{line: algorithm elimination rule}) whether
to eliminate the model from ${\cal T}_{r}$, using a condition similar
to the source elimination lemma (Corollary \ref{cor: source identification}).
Then, it is decided (line \ref{line: stopping criterion}) if to stop
after the $r$th elimination round, based on several possible criteria:
(1) If $T_{r-1}\leq\nicefrac{\sigma_{0}^{2}}{\overline{\sigma}^{2}({\cal T}_{r-1})}$
then this allows the algorithm to eliminate in the current round all
models outside ${\cal T}_{\text{w.o.}}$. (2) If ${\cal T}_{r}=\emptyset$
then all source models have been eliminated, which implies that using
the target model estimator $\overline{\theta}_{0}(\overline{N})$
is best. (3) If $T_{r}=T_{r-1}$ then no additional source model is
eliminated. If none of these conditions is fulfilled then $\overline{r}$
elimination rounds are performed. Finally, in line \ref{line: algorithm elimination final output},
one of the models $t^{*}$ in $\overline{t}({\cal T}_{\text{alg}})$
is chosen (where $t^{*}=0$ if ${\cal T}_{\text{alg}}=\emptyset$),
and the final estimate is $\overline{\theta}_{t^{*}}(\bar{N})$.

\paragraph*{Sample allocation at each round}

Assume that at round $r$ the $\overline{N}$ samples are allocated
to the retained source models in ${\cal T}_{r-1}$ as $\{\bar{N}_{t,r}\}_{t\in{\cal T}_{r-1}}$
with $\sum_{t\in{\cal T}_{r-1}}\bar{N}_{t,r}=\overline{N}$. Then,
for any $t\in{\cal T}_{r-1}$
\[
\|\overline{\theta}_{t}(N/T)-\theta_{t}\|^{2}\leq g(\bar{\delta})\frac{d\sigma_{t}^{2}}{\bar{N}_{t,r}}:=\lambda_{t}^{2}
\]
(with probability $1-\overline{\delta}$), and according to the source
elimination lemma (Corollary \ref{cor: source identification}), the
$t$th model will be eliminated if $Q_{t}^{2}\geq\nicefrac{(\lambda_{t}^{2}\vee\lambda_{0}^{2})}{\nu}$.
Since the retained source models in ${\cal T}_{r-1}$ were not discriminated
in round $r-1$, a plausible sample allocation will induce an equal
elimination barrier for all $t\in{\cal T}_{r-1}$, \ie, $\lambda_{t_{1}}^{2}=\lambda_{t_{2}}^{2}$
for any $t_{1},t_{2}\in{\cal T}_{r-1}$. This results in the sample
allocation $\bar{N}_{t,r}=\nicefrac{\bar{N}}{T_{r-1}}\cdot\nicefrac{\sigma_{t}^{2}}{\overline{\sigma}^{2}({\cal T}_{r-1})}$
used in line \ref{line: algorithm sample allocation}, for which $\lambda_{t}^{2}=\nicefrac{g(\bar{\delta})d\overline{\sigma}^{2}({\cal T}_{r-1})}{(\nicefrac{\bar{N}}{T_{r-1}})}$.
For example, if $\sigma_{t}^{2}\equiv\sigma^{2}$ for all source models
$t\in[T]$, then the sample allocation is uniform over $T_{r-1}$
and $\lambda_{t}^{2}=\nicefrac{g(\bar{\delta})d\sigma^{2}}{(\nicefrac{\bar{N}}{T_{r-1}})}$.
The barrier $\nicefrac{(\lambda_{t}^{2}\vee\lambda_{0}^{2})}{\nu}$
thus decreases at each round, allowing for additional source models
to be eliminated. Now, if the number of currently retained models
$T_{r-1}$ is \emph{small} enough, specifically, $T_{r-1}\leq\nicefrac{\sigma_{0}^{2}}{\overline{\sigma}^{2}({\cal T}_{r-1})}$,
then the barrier for elimination is $\lambda_{t}^{2}\leq\nicefrac{g(\bar{\delta})}{\nu}\cdot\nicefrac{d\sigma_{0}^{2}}{\overline{N}}$,
which is the same barrier used by the weak oracle \eqref{eq: weak oracle set},
with $\kappa\leq\nicefrac{g(\bar{\delta})}{\nu}\cdot(\overline{r}+2)$.
Thus, if $\overline{r}$ is not very large, whenever $T_{r-1}\leq\nicefrac{\sigma_{0}^{2}}{\overline{\sigma}^{2}({\cal T}_{r-1})}$,
the set ${\cal T}_{\text{w.o.}}$ is identified by the algorithm \WHP.
In fact, if $T\leq\nicefrac{\sigma_{0}^{2}}{\overline{\sigma}^{2}({\cal T}_{0})}$
then this occurs already in the first round. Multiple rounds are needed
for \emph{large} number of source models. 

\paragraph*{Number of rounds and elimination dynamics}

The elimination dynamics depend in an intricate way on $\{Q_{t}^{2},\sigma_{t}^{2}\}_{t\in\llbracket T\rrbracket}$.
To allow for simple tracking, we introduce the \emph{elimination curve}
$\beta_{\delta}(\tau)$. Let $\tau\in(0,1)$, and assume that there
are $\tau T$ retained models. Then, $T\beta_{\delta}(\tau)$ is the
number of models that will be eliminated next by the algorithm (\WHP).
Concretely, letting $\tau_{\text{min}}:=\frac{1}{T}\lceil\nicefrac{\sigma_{0}^{2}}{\overline{\sigma}^{2}([T])}\rceil,$
and setting the required reliability to $\delta\in(0,1]$, we obtain
\begin{equation}
\beta_{\delta}(\tau):=\begin{cases}
\frac{\left|\left\{ t\in[T]\colon Q_{t}^{2}\leq\frac{g(\bar{\delta})}{\nu}\cdot\frac{d\overline{\sigma}^{2}([\tau T])}{N/(\tau T)}\right\} \right|}{T}, & \tau\in(\tau_{\text{min}},1)\\
\frac{\left|\left\{ t\in[T]\colon Q_{t}^{2}\leq\frac{g(\bar{\delta})}{\nu}\cdot\frac{d\sigma_{0}^{2}}{N}\right\} \right|}{T}, & \tau\in(0,\tau_{\text{min}})
\end{cases}.\label{eq: defintion beta function - unequal noise variances}
\end{equation}
Hence, at the first round $T_{1}=T\cdot\beta_{\delta}(1)$ source
models will be retained by the algorithm \WHP, after the second $T_{2}=T\cdot\beta_{\delta}(\beta_{\delta}(1))$,
and so on. After $\bar{r}$ rounds, the number of retained source
models is \WHP $T_{\bar{r}}=T\cdot\beta_{\delta}(\beta_{\delta}(\cdots(1))=T\cdot\beta_{\delta}^{(\bar{r})}(1)$,
where $\beta_{\delta}^{(r)}$ denotes the $r$th order functional
power of $\beta_{\delta}$, and if $T_{r}\leq\tau_{\text{min}}T$
then the algorithm stops since it eliminates all source models outside
${\cal T}_{\text{w.o.}}$. That is, if the algorithm satisfies either
$T_{r-1}\leq\nicefrac{\sigma_{0}^{2}}{\overline{\sigma}^{2}({\cal T}_{r-1})}$
or $T_{r}=0$, then ${\cal T}_{\text{alg}}={\cal T}_{\bar{r}}={\cal T}_{\text{w.o.}}(\nicefrac{g(\bar{\delta})}{\nu})$,
and the algorithm identifies the weak oracle set with constant $\nicefrac{g(\bar{\delta})}{\nu}$.
However, it may happen that the elimination stops becuase $T_{r}=T_{r-1}$.
As we next show, this is easily identified by $\beta_{\delta}(\tau)$. 
\begin{example}
\label{exa: beta function}Consider $T=10^{3}$ and assume the normalized
values $\nicefrac{g(\bar{\delta})}{\nu}\cdot\nicefrac{(d\sigma_{t}^{2})}{(\nicefrac{N}{T})}=1$,
for which $\beta_{\delta}(\tau)=\frac{1}{T}|\{t\in[\tau]\colon Q_{t}^{2}\leq\tau\}|$
for $\tau\in(\tau_{\text{min}},1)$ and $\beta_{\delta}(\tau)=\beta_{\delta}(\tau_{\text{min}})$
for $\tau\in(0,\tau_{\text{min}})$. We assume $\tau_{\text{min}}=0.2$.
We demonstrate the sensitivity of CL to the values of $\{Q_{t}^{2}\}_{t\in[T]}$
by comparing two very similar cases -- in the first case $Q_{t}^{2}=\nicefrac{t^{1/3}}{9}$
and in the second case $Q_{t}^{2}=\nicefrac{t^{1/3}}{12}$. In both
cases, if $\sigma_{0}^{2}$ is sufficiently large, then ${\cal T}_{\text{w.o.}}$
is not empty. Figure \ref{fig:beta and simulation} (left panel) shows
that in the first case (red curve) $\beta_{\delta}(1)\approx0.72$,
$\beta_{\delta}(0.62)\approx0.28$, and $\beta_{\delta}(0.242)\approx0.017$,
so that after $3$ rounds, all models outside ${\cal T}_{\text{w.o.}}$
are eliminated. By contrast, in the second case (blue curve), $\beta_{\delta}(1)=1$,
and while all source models are better than the target model, \emph{none}
of them is eliminated. Thus, the risk of the weak-oracle learner is
\emph{not} achieved by Algorithm \ref{alg: CL multiple sources}. 
\end{example}
Example \ref{exa: beta function} reveals an intricate property: The
distances $\{Q_{t}\}_{t\in[T]}$ are smaller in the second case compared
to the first case, but Algorithm \ref{alg: CL multiple sources} has
a lower loss guarantee in the first case -- thus the loss is \emph{not}
monotonic with the distances. This appears to be an inherent property
of the statistical CL problem: Models with low similarity (large $Q_{t}$)
can be easily eliminated. Nonetheless, generalizing the observation
of Example \ref{exa: beta function}, if $\beta_{\delta}(\tau)$ is
bounded away from the identity line (and thus does not have a fixed
point), then Algorithm \ref{alg: CL multiple sources} will identify
the weak-oracle set using $r=\Theta(\log T)$ rounds, as follows (the
proof is obvious from the discussion above):
\begin{prop}
\label{prop: beta strictly bounded }Suppose $\beta_{\delta}(\tau)\leq\overline{\beta}\tau$
for some $\overline{\beta}\in(0,1)$ and all $\tau\in(\tau_{\text{min}},1]$.
Then, ${\cal T}_{\text{alg}}={\cal T}_{\text{w.o.}}(\nicefrac{g(\bar{\delta})}{\nu})$
if
\[
\bar{r}=\frac{\log\tau_{\text{min}}}{\log\overline{\beta}}=\frac{\log\left(T\cdot\nicefrac{\overline{\sigma}^{2}([T])}{\sigma_{0}^{2}}\right)}{\log(\nicefrac{1}{\overline{\beta}})}.
\]
\end{prop}

\paragraph*{The final estimate}

Similarly to the strong/weak-oracle learners, we choose the final
estimate to be based on $\overline{N}$ samples from $t_{\text{alg}}:=\overline{t}({\cal T}_{\text{alg}})$,
which favors the model with minimal noise variance, see \eqref{eq: choise of model within oracle set}.
The final estimate is thus $\hat{\theta}_{\text{alg}}=\overline{\theta}_{t_{\text{alg}}}(\overline{N})$,
and the resulting loss is bounded \WHP, as in \eqref{eq: MSE of source model},
by $\|\hat{\theta}_{\text{alg}}-\theta_{0}\|^{2}\lesssim\nicefrac{d\sigma_{t_{\text{alg}}}^{2}}{N}+Q_{t_{\text{alg}}}^{2}$.
The following theorem provides a rigorous high probability guarantee
on the loss and the risk of Algorithm \ref{alg: CL multiple sources}.
The proof appears in Appendix \ref{sec:Proofs-for-Section-algorithm}
and is a direct consequence of the source elimination lemma. 
\begin{thm}
\label{thm: Multiple source CL algorithm}Assume \WLOG that\footnote{The learner does not know $\{Q_{t}^{2}\}_{t\in[T]}$, and does not
know that this order holds.} $0=Q_{0}^{2}\leq Q_{1}^{2}\leq\cdots\leq Q_{T}^{2}$. Let $\beta_{\delta}(\tau)$
be as in \eqref{eq: defintion beta function - unequal noise variances}
and let ${\cal T}_{\text{alg}}=[T_{\bar{r}}]$ with $T_{\bar{r}}=T\beta_{\delta}(\beta_{\delta}(\cdots(1)))=T\beta_{\delta}^{(\bar{r})}(1)$.
Let $\bar{\delta}=\nicefrac{\delta}{(T\bar{r}+2)}$ for $\delta\in(0,1)$
and let $\hat{\theta}$ be the output of Algorithm \ref{alg: CL multiple sources}.
Then, with probability at least $1-\delta$
\begin{equation}
\|\hat{\theta}-\theta_{0}\|^{2}\leq2(\bar{r}+2)\cdot\left(g(\bar{\delta})\vee1\right)\cdot\left[\frac{d\sigma_{\overline{t}({\cal T}_{\text{alg}})}^{2}}{N}+Q_{\overline{t}({\cal T}_{\text{alg}})}^{2}\right].\label{eq: MSE multiple source high probability}
\end{equation}
 Assume further that $\max_{t\in[T]}Q_{t}^{2}\leq4C_{\theta}^{2}$
for some constant $C_{\theta}>0$. Then, setting 
\[
\delta=\delta_{*}:=\left[\frac{1}{2(T+1)}\cdot\left[\min_{t,t'\in\llbracket T\rrbracket}\frac{\sigma_{t}^{2}}{\sigma_{t'}^{2}}\vee\min_{t\in\llbracket T\rrbracket}\frac{d\sigma_{t}^{2}}{8NC_{\theta}^{2}}\right]\right]^{2}
\]
as input to Algorithm \ref{alg: CL multiple sources} results in
\begin{equation}
\E\left[\|\hat{\theta}-\theta_{0}\|^{2}\right]\leq8c(\bar{r}+2)\cdot\log\left(\frac{(T\bar{r}+2)}{\delta_{*}}\right)\cdot\left[\frac{d\sigma_{\overline{t}({\cal T}_{\text{alg}})}^{2}}{N}+Q_{\overline{t}({\cal T}_{\text{alg}})}^{2}\right].\label{eq: multiple sources bound in expectation}
\end{equation}
\end{thm}
For mean estimation $g(\bar{\delta})=\Theta(\log(\nicefrac{T\bar{r}}{\delta}))$,
and so if ${\cal T}_{\text{alg}}={\cal T}_{\text{w.o.}}(\nicefrac{g(\bar{\delta})}{\nu})$
then the loss/risk of the weak-oracle learner is achieved using $\bar{r}=\Theta(\log T)$
rounds, up to a mild $\Theta(\log^{2}(\nicefrac{TN}{d}))$ factor. 

\paragraph*{Extension to unknown variances/covariance-matrices}

In Appendix \ref{sec:unknown noise covariance}, we extend Algorithm
\ref{alg: CL multiple sources} and Theorem \ref{thm: Multiple source CL algorithm}
to the case the variances $\{\sigma_{t}^{2}\}_{t\in\llbracket T\rrbracket}$
are unknown (with $\overline{\Sigma}_{t}=I_{d}$ for all $t\in\llbracket T\rrbracket$
), and then to unknown covariance matrices. The extension is based
on a preliminary step of variance/covariance estimation, and plug-in
of the estimated values instead of the exact ones in Algorithm \ref{alg: CL multiple sources}.
Assuming that the mild condition $dN=\Omega(T\log T)$ holds for unknown
variances and $dN\gtrsim d\log d\cdot T\log T$ holds for unknown
covariance matrices, the loss/risk of this extended Algorithm \ref{alg: CL multiple sources}
is roughly the same as in Theorem \ref{thm: Multiple source CL algorithm}
in both cases.

\paragraph*{Numerical experiments}

We empirically validated Algorithm \ref{alg: CL multiple sources}
in simulations, and provide the results in Appendix \ref{sec:Simulations}.
Here, we highlight one experiment in which the sources are partitioned
to three types, according to the distance of their parameter from
the target parameter, designated as ``close'', ``medium'', and
``far''. Figure \ref{fig:beta and simulation} (right panel) displays
the loss of Algorithm \ref{alg: CL multiple sources} for different
mixtures of types $(T_{\text{far}},T_{\text{close}})$ and $T_{\text{med}}=T-T_{\text{close}}-T_{\text{far}}$,
for $T=100$ source models, averaged over $200$ experiments. One
can track the loss over various interesting paths. For example, since
Algorithm \ref{alg: CL multiple sources} can eliminate the ``far''
models, the loss is consistently low when traversing the straight
line from $(0,100)$ to $(100,0)$, as long as there is at least one
close source model. The loss increases when traversing the straight
line from $(50,50)$ to $(0,0)$ due to the additional ``medium''
parameters, which are not eliminated. The loss of Algorithm \ref{alg: CL multiple sources}
in $(0,a)$ has lower loss compared to $(a,0)$, since the ``close''
parameters in the former case are replaced by ``far'' parameters
in the latter. 

\begin{figure}
\begin{centering}
\includegraphics[scale=0.5]{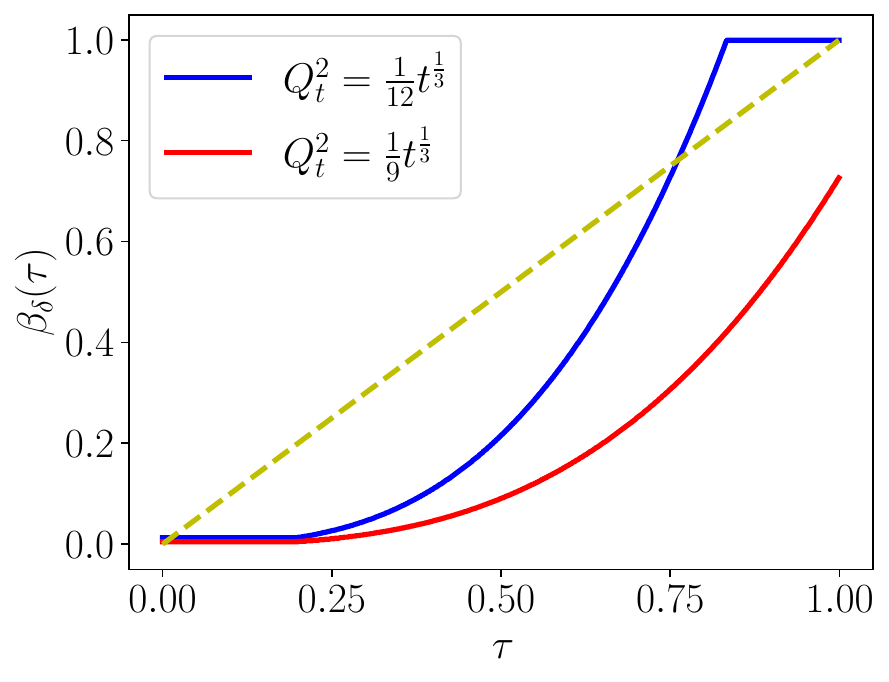}\includegraphics[scale=0.5]{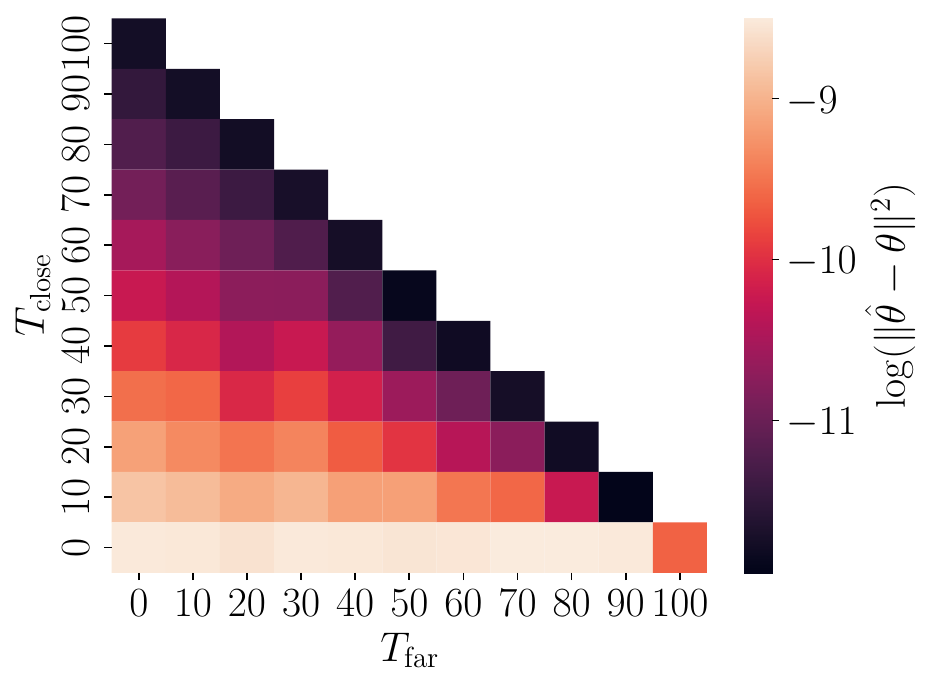}
\par\end{centering}
\caption{\label{fig:beta and simulation}Left: The elimination curves $\beta_{\delta}(\tau)$
for Example \ref{exa: beta function} (the identity line in dashed
yellow). Right: The loss $\|\hat{\theta}-\theta_{0}\|^{2}$ on a log-scale.
Parameters are $N=10^{5},\;d=2,\;\sigma^{2}=0.1,\;\sigma_{0}^{2}=1,\;\tilde{Q}_{\text{close}}^{2}=0,\;\tilde{Q}_{\text{medium}}^{2}=10,\;\tilde{Q}_{\text{far}}^{2}=2\cdot10^{4}$
where $\tilde{Q}_{t}^{2}:=\nicefrac{Q_{t}^{2}}{(\nicefrac{d\sigma_{0}^{2}}{N})}$
is the normalized distance. }
\end{figure}

\section{Minimax risk lower bounds \label{sec:Minimax-risk-lower}}

We next derive lower bounds on the minimax risk, which can be compared
with the risk of the weak-oracle learner \eqref{eq: weak oracle},
and with Algorithm \ref{alg: CL multiple sources}. Since the learner
can typically learn the noise statistics (see Appendix \ref{sec:unknown noise covariance}),
we assume that $\{\overline{\Sigma}_{t}=I_{d}\}_{t\in\llbracket T\rrbracket}$
and $\{\sigma_{t}^{2}\}_{t\in\llbracket T\rrbracket}$ are known to
the learner. In this case, we may assume \WLOG that $\sigma_{0}^{2}\geq\sigma_{1}^{2}\geq\sigma_{2}^{2}\geq\cdots\geq\sigma_{T}^{2}$,
and let $\boldsymbol{\Gamma}:=(\sigma_{0}^{2},\sigma_{1}^{2},\ldots,\sigma_{T}^{2})$.
We take the noise variances as fixed, thus specify a problem-instance
by $\boldsymbol{\theta}=(\theta_{0},\theta_{1},\ldots,\theta_{T})$,
and let $\boldsymbol{\Psi}\subset(\mathbb{R}^{d})^{T+1}$ be a set
of such parameters. The minimax risk is then 
\begin{equation}
L_{N,d}(\boldsymbol{\Psi},\boldsymbol{\Gamma}):=\inf_{{\cal A}}\sup_{\boldsymbol{\theta}\in\boldsymbol{\Psi}}\E_{\boldsymbol{\phi},{\cal A}}\left[\left\Vert \hat{\theta}\left((A_{i},S_{i})_{i\in[N]}\right)-\theta_{0}\right\Vert ^{2}\right].\label{eq: minimax risk Gaussian location model}
\end{equation}
To lower bound the minimax risk, we follow the standard method of
reducing the learning problem to an hypothesis testing problem between
$K$ hypotheses $\boldsymbol{\Theta}_{\text{test}}:=\{\theta^{(j)}\}_{j\in[K]}\subset\boldsymbol{\Psi}$
where $\theta^{(j)}=(\theta_{0}^{(j)},\theta_{1}^{(j)},\ldots,\theta_{T}^{(j)})$,
and then bounding the error probability in the latter \citep{yang1999information}
\citep[Chapter 15]{wainwright2019high}. However, this reduction is
different from the standard one in two aspects, which we next discuss,
before deriving two minimax lower bounds. 

\paragraph*{A general reduction to hypothesis testing}

The algorithm ${\cal A}$ in \eqref{eq: minimax risk Gaussian location model}
collects samples adaptively, and so its $N$ samples are \emph{not}
\IID. Thus, in principle, the reduction to hypothesis testing should
be made to a similarly adaptive tester. However, lower bounds for
adaptive testers are not as widely available, compared to Le-Cam's
and Fano's bounds for \IID samples. There are a few strategies to
circumvent this. First, we may naively assume that each of the $T+1$
models is sampled $N$ times, and so there is no need for adaptive
allocation. This leads to non-trivial lower bounds, but does not capture
the correct dependence on $T$. Second, we may choose the set $\boldsymbol{\Theta}_{\text{test}}$
in a way that an optimal learner will surely only sample from one
of the models. Since the noise is Gaussian, this is the case, \eg,
if the $\theta_{t_{1}}^{(j)}\propto\theta_{t_{2}}^{(j)}$ for some
$t_{1},t_{2}\in[T]$ and all $j\in[K]$. Note that making an analogous
claim for estimation, rather than testing, is non-trivial. We discuss
this in detail in Appendix \ref{subsec:A-general-reduction}.

\paragraph*{The choice of the localization set and associated challenges}

The \emph{localization set} $\boldsymbol{\Psi}$ localizes the lower
bound on a set of problem instances. For the most informative bound,
these instances should have equal difficulty, where here, as in \citep{mousavi2020minimax,xu2022statistical},
the difficulty is determined by $\{Q_{t}^{2}\}_{t\in\llbracket T\rrbracket}$
(beyond$\{\sigma_{t}^{2}\}_{t\in\llbracket T\rrbracket}$, which are
assumed fixed). In fact, for our problem, some localization is essential,
since if $\boldsymbol{\Psi}=(\mathbb{R}^{d})^{T+1}$ then the worst
case is when $Q_{t}^{2}\gg\nicefrac{d\sigma_{0}^{2}}{N}$, and all
source models are useless for the target task; this trivializes the
bound. \citet{xu2022statistical} proposed a set we call a \emph{semi-local
set}, given by 
\begin{equation}
\boldsymbol{\Psi}\equiv\boldsymbol{\Psi}_{\leq}(\boldsymbol{q}):=\bigcup_{\zeta\in\mathfrak{S}_{T}}\left\{ \boldsymbol{\theta}\in\boldsymbol{\Theta}\colon Q_{t}^{2}:=\|\theta_{t}-\theta_{0}\|^{2}\leq q_{\zeta(t)}^{2},\text{ for all }t\in[T]\right\} ,\label{eq: semi-local set}
\end{equation}
where $\mathfrak{S}_{T}$ is the set of all permutations of $[T]$
and $\boldsymbol{q}=\{q_{t}\}_{t\in[T]}$ satisfies, \WLOG, $0<q_{1}\leq\cdots\leq q_{T}$.
So, the learner only knows an \emph{unordered multiset} of upper bounds
on $\{Q_{t}\}_{t\in[T]}$ (where $Q_{t}:=\|\theta_{0}-\theta_{t}\|$).
However, the risks of the weak-oracle learner and Algorithm \ref{alg: CL multiple sources}
are \emph{not} \emph{necessarily monotonic} with $Q_{t}$, and so
the worst-case may achieved in the interior of this set, \ie, for
$Q_{t}<q_{\zeta(t)}$. Hence, the set $\boldsymbol{\Psi}_{\leq}(\boldsymbol{q})$
contains problem instances of \emph{heterogeneous} difficulty. A reasonable
solution is to replace the $\leq$ in \eqref{eq: semi-local set}
with $=$, thus obtaining $\boldsymbol{\Psi}_{=}(\boldsymbol{q})$,
a \emph{fully localized set}. However, restricting $\theta$ to $\boldsymbol{\Psi}_{=}(\boldsymbol{q})$
may reveal too much information to the learner, which allows unrealistically
low risk, and thus to a loose bound. Indeed, assume that $\sigma_{t}^{2}=0$
for all $t\in[T]$, that $T$ is odd, that $\nicefrac{(T+1)}{2}$
of the source parameters are identical to the target parameter, and
that the other $\nicefrac{(T-1)}{2}$ source parameters are at a larger,
equal distance, which is still within the target noise resolution.
That is $q_{0}^{2}=q_{1}^{2}=\cdots q_{\nicefrac{(T+1)}{2}}^{2}=0$
and $q_{\nicefrac{(T+3)}{2}}^{2}=\cdots=q_{T}^{2}:=\overline{q}^{2}\lesssim\nicefrac{d\sigma_{0}^{2}}{N}$.
The loss of the weak-oracle learner in \eqref{eq: weak oracle} is
$\|\hat{\theta}_{\text{w.o.}}-\theta_{0}\|^{2}\lesssim\overline{q}^{2}$,
and is achieved by Algorithm \ref{alg: CL multiple sources}. However,
just knowing the \emph{multiset} $\{q_{t}^{2}\}_{t\in[T]}$, but not
the permutation $\zeta$, the learner can collect one sample from
each source model, and estimate $\theta_{0}$ as the observed sample
that is common to $\nicefrac{(T+1)}{2}$ models (or more). Thus, no
lower bound can be better than $0$. In Appendix \ref{subsec:The-choice-of}
we provide two more examples of exceeding the weak-oracle learner
risk, either by a dimensionality reduction to $d=1$ (when $T=2$),
or by a reduction to two possibilities (when $T=1$). We also explain
why this occurs even for noise with a strictly positive variance.
The problems of both $\boldsymbol{\Psi}_{\leq}(\boldsymbol{q})$ (lack
of monotonicity) and $\boldsymbol{\Psi}_{=}(\boldsymbol{q})$ (possibly
trivial or loose lower bound) reveal the delicate nature of the statistical
CL problem. 

\paragraph*{Semi-local minimax risk lower bound in low dimensions}

For the \emph{semi-local }set $\boldsymbol{\Psi}_{\leq}(\boldsymbol{q})$
we derive a lower bound based on a one-dimensional construction, which,
as such, does not capture the correct dependence on $d$. For brevity,
we let $t_{\text{w.o.}}:=\overline{t}({\cal T}_{\text{w.o.}}(\kappa))$
where ${\cal T}_{\text{w.o.}}(\kappa)$ is defined for the identity
permutation, \ie, $Q_{t}^{2}=\|\theta_{t}-\theta_{0}\|^{2}=q_{t}^{2}$.
We also let $q_{\text{\text{w.o.}}}:=q_{t_{\text{w.o.}}}$ as well
as $\sigma_{\text{\text{w.o.}}}^{2}:=\sigma_{t_{\text{w.o.}}}^{2}$,
and refer to $t_{\text{w.o.}}$ as the index of the \emph{weak-oracle
model}. We let the index of the task with median distance $q_{t}$
among the ``close'' models $t\in[t_{\text{w.o.}}-1]$ be $t_{\text{med}}:=\lceil\nicefrac{(t_{\text{w.o.}}+1)}{2}\rceil$,
and let $q_{\text{med}}:=q_{t_{\text{med.}}}$. We thus also refer
to ${\cal M}_{\text{med}}:={\cal M}_{t_{\text{med}}}$ as the \emph{median
model} (within the ``close'' source models).
\begin{thm}
\label{thm: lower bound global risk constant dimension}Suppose \WLOG
that $q_{0}^{2}=0\leq q_{1}^{2}\leq\cdots\leq q_{T}^{2}$ and $\sigma_{0}^{2}\geq\sigma_{1}^{2}\geq\cdots\geq\sigma_{T}^{2}$.
Then
\[
L_{N,d}\left(\boldsymbol{\Psi}_{\leq}(\boldsymbol{q}),\boldsymbol{\Gamma}\right)\geq\frac{1}{720}\cdot\left(\frac{\sigma_{T}^{2}}{N}+q_{\text{med}}^{2}\right),
\]
with $\kappa=\nicefrac{1}{(4d)}$ in ${\cal T}_{\text{w.o.}}(\kappa)$
of \eqref{eq: weak oracle set}.
\end{thm}
The lower bound is tighter than the strong-oracle learner risk of
\citet[Theorem 1]{xu2022statistical} and also tighter than its improved
version \citep[Theorem 2]{xu2022statistical} (see Appendix \ref{subsec:Proof-of-Theorem- first-minimax}
for a detailed comparison). However, in general it is lower than the
weak-oracle learner risk since it uses $q_{\text{med}}^{2}$ instead
of the larger $q_{\text{w.o.}}^{2}$, and $\sigma_{T}^{2}$ instead
of $\sigma_{t_{\text{med}}}^{2}$. Interestingly, the gap between
them depends on $\{q_{t}^{2}\}_{t\in[T]}$ in an intricate way, just
like the intricate way that the upper bound on the risk of Algorithm
\ref{alg: CL multiple sources} depends on the distances (i.e., the
elimination curve $\beta_{\delta}(\tau)$). In Appendix \ref{subsec:Proof-of-Theorem- first-minimax}
we discuss this in detail, and provide a proof based on Le-Cam's two-point
method. We prove it for $T=2$, and then reduce $T>2$ to this case.
Nonetheless, as the proof shows, the bound matches the weak-oracle
learner risk for $T\leq2$.

\paragraph*{Single localized-source minimax risk lower bound in high dimensions}

We propose 
\[
\boldsymbol{\Psi}_{*}(q):=\bigcup_{\zeta\in\mathfrak{S}_{T}}\left\{ \boldsymbol{\theta}\in\boldsymbol{\Theta}\colon\exists t\in[T]\text{ such that }Q_{t}^{2}=\|\theta_{t}-\theta_{0}\|^{2}=q^{2}\right\} ,
\]
which is \emph{a single localized-source set}, and\emph{ }offers a
reasonable balance between localization of the instance and the information
revealed to the learner. Given that $\boldsymbol{\theta}\in\boldsymbol{\Psi}_{*}(q)$,
the learner only knows that there exists one source model whose parameter
is at distance $q$ from the target parameter.
\begin{thm}
\label{thm: lower local risk high dimension}Assume that $d\geq3$,
that $q_{0}^{2}=0\leq q_{1}^{2}\leq\cdots\leq q_{T}^{2}$ and that
$\sigma_{0}^{2}\geq\sigma_{1}^{2}\geq\cdots\geq\sigma_{T}^{2}$. Set
$\kappa=1$ in the definition of ${\cal T}_{\text{w.o.}}(\kappa)$
in \eqref{eq: weak oracle set}. There exists $c(d)\geq e^{-5}$ so
that 
\[
L_{N,d}(\boldsymbol{\Psi}_{*}(q_{\text{w.o.}}),\boldsymbol{\Gamma})\ge\begin{cases}
L_{N,d}(\boldsymbol{\Psi}_{=}(\boldsymbol{q}),\boldsymbol{\Gamma})\geq\frac{c^{2}(d)}{2}\cdot\frac{d\sigma_{\text{w.o.}}^{2}}{N}, & \nicefrac{d\sigma_{\text{w.o.}}^{2}}{N}\geq q_{\text{w.o.}}^{2}\\
\frac{c^{2}(d)}{8}\cdot q_{\text{w.o.}}^{2}, & \nicefrac{d\sigma_{\text{w.o.}}^{2}}{N}<q_{\text{w.o.}}^{2}
\end{cases}.
\]
\end{thm}
Thus, if the noise variance $\nicefrac{d\sigma_{\text{w.o.}}^{2}}{N}$
of the weak oracle model dominates $q_{\text{w.o.}}^{2}$ the minimax
risk lower bound holds even for the localized set $\boldsymbol{\Psi}_{=}(\boldsymbol{q})$.
Otherwise, it holds for the single localized-source set $\boldsymbol{\Psi}_{*}(q_{\text{w.o.}})$.
The proof is in Appendix \ref{subsec:Proof-of-Theorem-second-minimax}
and utilizes Fano's method to obtain a bound that correctly depends
on $d$. The construction of $\boldsymbol{\Theta}_{\text{test}}$
is designed so that the optimal policy $\pi$ is non-adaptive. 

\section{Future research \label{sec:Summary-and-open}}

While the elimination approach is natural for CL, it would interesting
to investigate other algorithmic approaches to compete with the weak-oracle
learner. For lower bounds, it is interesting to develop and directly
utilize bounds on adaptive hypothesis testing, or to utilize Bayesian
methods. Additional interesting settings are the CL problem under
a low-dimensional common structure of the parameters (as considered
by \citet{du2021fewshot,xu2022statistical}), or other similarity
measures between parameters, \eg, the sparsity level $\|\theta_{0}-\theta_{t}\|_{0}$.
Finally, it would be interesting to develop algorithms and extend
the analysis to non-parametric and distribution-free settings.

\appendix

\paragraph*{Organization of the appendix}

In Appendix \ref{sec:Related-work}, we review related work in detail.
In Appendix \ref{sec:Notation-conventions}, we set up notation conventions.
In Appendix \ref{sec:High-probability-bounds} we derive, for completeness,
a high probability bound on the empirical mean estimator. In Appendix
\ref{sec:The-linear-regression}, we formulate the parametric linear
regression setting, considered by \citet{xu2022statistical} instead
of our similar parametric setting of Gaussian mean estimation, and
derive, for completeness, a high probability bound on the projected
LS estimator. In Appendix \ref{sec:Comments-on-Xu-Tewari}, we make
a detailed comparison with the results of \citet{xu2022statistical}.
In Appendix \ref{sec:Proofs-for-single-source-model}, we provide
the proof for Theorem \ref{thm: Single source CL algorithm}, which
bounds the loss/risk for our elimination method in the case of a single
source model $(T=1)$, and to this end we introduce the \emph{source
elimination lemma} (Lemma \ref{lem: iden}). In Appendix \ref{sec:Proofs-for-Section-algorithm},
we provide the proof of Theorem \ref{thm: Multiple source CL algorithm}
that establishes theoretical guarantees on the loss/risk of Algorithm
\ref{alg: CL multiple sources}, used for multiple source models ($T>1$).
In Appendix \ref{sec:unknown noise covariance}, we generalize Algorithm
\ref{alg: CL multiple sources} and Theorem \ref{thm: Multiple source CL algorithm}
to the case that the noise variances or covariance matrices are unknown,
and show that the loss/risk only slightly increases, assuming mild
conditions on the problem parameters hold. In Appendix \ref{sec:Simulations},
we empirically validate Algorithm \ref{alg: CL multiple sources}
on several simulated settings, and show that its risk matches the
theoretical predictions. In Appendix \ref{sec:Proofs-for-Section-Minimax},
we discuss in detail the two aspects of the minimax lower bound presented
in the main text, and then prove the two minimax lower bounds (Theorem
\ref{thm: lower bound global risk constant dimension} and Theorem
\ref{thm: lower local risk high dimension}). 

\section{Related work \label{sec:Related-work}}

\citet{bengio2009curriculum} explained the optimization benefits
of CL by attributing to them regularization capabilities that reduce
the generalization error. One way to achieve this is employing CL
as a \emph{continuation method} \citep{allgower2012numerical}, in
which the difficulty of a task is defined through the smoothness of
its objective function. A continuum of risk functions $\{L_{\lambda}(\theta)\}_{\lambda\in[0,1]}$
is defined, in which $L_{0}(\theta)$ is easiest to optimize over
$\theta$ (\eg, it is very smooth or convex), and $L_{1}(\theta)$
is the target risk function, which is difficult to optimize. The learner
first optimizes $L_{0}(\theta)$ and then increases $\lambda$ while
continuously tracking the minimizer of the risk. Alternatively, improved
generalization error can be achieved by placing weights on the training
samples. At first, larger weights assigned to easy samples (\eg,
far from the margin, or without irrelevant inputs), and then gradually
more weight is assigned to more difficult ones, until all samples
are assigned a uniform weight. The experiments by \citet{bengio2009curriculum}
were the first to suggest that models trained with a simple CL strategy
achieve better convergence rate and accuracy.

Following \citet{bengio2009curriculum}, the CL approach has been
experimented and utilized in various applications, such as computer
vision, natural language processing, computer vision, speech processing,
medical imaging, reinforcement learning and others \citep{portelas2020automatic,narvekar2020curriculum}.
However, despite its intuitive plausibility, CL is not widely used
in machine learning, and the empirical results occasionally exhibit
only moderate effectiveness, if at all. In an attempt to map the advantages
of CL, a few recent papers \citep{wang2021survey,soviany2022curriculum}
have comprehensively surveyed CL methods and their applications. Specifically,
\citet{soviany2022curriculum} classified CL methods to those applied
on the data, the model, the task and the performance measure. A taxonomy
of CL methods was created, based on seven categories representing
aspects such as determinations of the CL pace, diversity in the training
samples, the complexity of the model, and architectures such as teacher-student.
That being said, CL was successfully recently used in foundation models
\citep{brown2020language}, and this is a strong timely motivation
to theoretically explore its benefits.

The statistical benefits of CL were considered by \citet{xu2022statistical},
which introduced the parametric learning problem that we adopt here,
showed that the risk of the strong-oracle learner is a lower bound
on the risk of a CL algorithm \citep[Theorem 1]{xu2022statistical},
and then offered an improved lower bound \citep[Theorem 2]{xu2022statistical}.
Then, a single-round elimination CL algorithm was proposed. Under
a restricted setting, an upper bound on the risk was derived \citep[Theorem 3]{xu2022statistical},
which was revised by \citet[Theorem 3.3]{xu2023benefits}. We compare
the results here with the results of \citet{xu2022statistical} both
in the paper, and in Appendix \ref{sec:Comments-on-Xu-Tewari} in
detail. 

The statistical benefits of CL rely on a quantitative measure of the
similarity between the target task and the source task. Such measures
were extensively addressed in \emph{multitask learning and meta-learning}
(\emph{learning-to-learn})\emph{ }\citep{vilalta2002perspective,baxter2000model,argyriou2006mulfti,maurer2016benefit,du2020few,tripuraneni2020theory,tripuraneni2021provable,hospedales2021meta},
where similarity is obtained, \eg, via a common low-dimensional representation.
In these methods, improved performance on a target task is achieved
by sampling multiple source tasks. Similarity between tasks is also
considered in \emph{transfer learning} \citep{ben2010theory,mohri2012new,yang2013theory,germain2013pac,hanneke2019value,zhuang2020comprehensive,kpotufe2021marginal},
and \emph{domain adaptation} \citep{mansour2009domain,david2010impossibility},
in which the goal is to transfer knowledge between tasks, and the
resulting risk is measured by the task similarity. Finally, in \emph{continual
learning} \citep{allgower2012numerical,zenke2017continual,van2019three,de2021continual},
tasks are presented to the learner in succession, and it should adapt
to the new task while avoiding \emph{forgetting} of previous tasks.
CL refines these methods, as it allows optimizing the order in which
source tasks are presented to the learner (the setting of a common
low-dimensional representation was also explored by \citet{xu2022statistical}). 

The ability of the learner to use its past experience in selecting
an action -- selecting the next model to sample from in the CL setting
-- directly relates the CL problem to online decision-making problems,
such as the \emph{multi-armed bandit} (MAB) problem \citep{bubeck2012regret,slivkins2019introduction,lattimore2020bandit}.
However, in our CL setting, the outcome of the action is the sample
of the chosen model, and unlike the MAB problem, this does not define
an explicit or an immediate reward that is linearly accumulated. Also,
unlike typical MAB problems, it is not obvious \emph{a priori }that
the optimal strategy is to sample from just a single model, even if
the problem-instance is known to the learner (via an oracle). So the
CL problem is also different from a best-arm identification problem
\citep{even2006action,audibert2010best,soare2014best,jamieson2014best,russo2016simple,garivier2016optimal}.
Nonetheless, \citet{graves2017automated} proposed MAB algorithms
for mini-batch scheduling of training of deep neural networks (DNNs).
They discussed reward signals driven by accuracy prediction and network
complexity, and discussed CL for multitask representation learning.
\citet{xu2022statistical} proposed prediction-gain driven task schedulers
that are inspired by MAB approaches. In a similar spirit, the CL problem
is also related to \emph{active learning }(\eg, \citep{hanneke2013statistical,hanneke2014theory,hino2020active})\emph{,
}say, in a supervised setting, in which the learner can choose the
next input to observe a label from based on past observations. From
the perspective of the CL setting, each input defines a task, and
thus a continuum of tasks is available. In active learning, however,
the goal is to minimize the average loss over all inputs (tasks),
and typically there is no specific target task.

Direct theoretical investigations of CL methods are rather scarce.
Beyond \citet{xu2022statistical}, they include \citet{weinshall2018curriculum,hacohen2019power,weinshall2020theory}
and the recent work by \citet{cornacchia2023mathematical,saglietti2022analytical}.
In the sequence of papers \citep{weinshall2018curriculum,hacohen2019power,weinshall2020theory},
the difficulty of a training point is measured by the loss of the
optimal hypothesis, and various relations between difficulty and convergence
rates of a stochastic gradient descent (SGD) training algorithm were
derived, both for linear regression and binary classification. For
example, it was proved that the convergence rate decreases with difficulty,
and for a step size small enough and a fixed loss, the convergence
rate increases with the loss. In practice, ranking the difficulty
of training points is hardly available in advance, and so \citet{weinshall2018curriculum,hacohen2019power,weinshall2020theory}
have proposed a teacher-student architecture, in which a teacher network
transfers the training point difficulty to the student network. A
training algorithm that is a variation of SGD was proposed, in which
at the first steps the input training point is chosen randomly with
bias towards easier examples, and this bias decays with the training
iterations. They empirically showed in a DNN setup that CL increase
the rate of convergence at the beginning of training, and improves
generalization when either the task is difficult, the network is small,
or when strong regularization is enforced.

\citet{cornacchia2023mathematical} explored whether CL can learn
concepts that are computationally hard to learn without curriculum.
It thus focused on learning $k$-parity functions of $d$ uniformly
chosen (unbiased) input bits, which is difficult for iterative algorithms,
such as SGD. Nonetheless, when the input bits are biased, SGD converges
fast to low error, and thus implicitly identifies the support of the
parity function. A CL approach is proposed based on domain adaptation,
in which the input distribution is gradually shifted from deterministic
(input bit is $1$ with probability $1$) to uniform (input bit is
$1$ with probability $1/2$). It is shown that the CL algorithm reduces
the computational complexity from $d^{\Omega(k)}$ to $d^{\Omega(1)}$.
\citet{saglietti2022analytical} proposed a CL model in which the
task includes both relevant and irrelevant features, and a sample
is difficult if it has irrelevant features with large variance. Analytical
expressions were derived for the average learning trajectories of
simple neural networks on this task, and it was shown that convergence
is accelerated at the early iterations of training. A multi-stage
CL algorithm in which the iterations of each stage are regularized
with the solution of the previous stage was proposed, and the algorithm
was analyzed in the high-dimensional limit. This showed the CL is
most effective when there is a small number of relevant features.

\section{Notation conventions \label{sec:Notation-conventions}}

For a positive integer $T$ we denote the sets $[T]:=\{1,2\ldots,T\}$
as well as $\llbracket T\rrbracket:=\{0\}\cup[T]$. For $a,b\in\mathbb{R}$,
we denote $a\wedge b:=\min\{a,b\}$ and $a\vee b:=\max\{a,b\}$. We
use standard Bachmann-Landau asymptotic notation, where, specifically,
we use $\tilde{O}(\cdot)$ and $\tilde{\Omega}(\cdot)$ to hide poly-logarithmic
factors. We also use the asymptotic relation $a\lesssim b$ to indicate
that $a=O(b)$ and $a\asymp b$ to indicate that $a=\Theta(b)$. We
denote the probability of an event ${\cal E}$ by $\P[{\cal E}]$,
the indicator function for this event by $\I[{\cal E}]$, and the
complement of this event by ${\cal E}^{c}$. We denote the $d$-dimensional
Euclidean ball with radius $r$ by $\mathbb{B}^{d}(r):=\{\theta\in\mathbb{R}^{d}\colon\|\theta\|\leq r\}$
and the $(d-1)$-dimensional unit sphere by $\mathbb{S}^{d-1}:=\{\theta\in\mathbb{R}^{d}\colon\|\theta\|=1\}$.
We denote equivalence by $\equiv$ which is usually just a local simplification
of notation. For a pair of probability distributions $P$ and $Q$
on a common alphabet ${\cal X}$ with densities $p$ and $q$ \wrt
to a base measure $\nu$, we denote the total variation distance by
$\dtv(P,Q):=\frac{1}{2}\int|p-q|\d\nu$ and the Kullback-Leibler (KL)
divergence by $\Dkl(P\mid\mid Q):=\int p\log(\nicefrac{p}{q})\d\nu$.
For a pair of random variables (\rv's) $(X,Y)\sim P_{XY}$, the mutual
information is denoted by $I(X;Y):=\Dkl(P_{XY}\mid\mid P_{X}\otimes P_{Y})$. 

\section{A high probability bound on the mean estimation loss \label{sec:High-probability-bounds}}

The following lemma is standard, and bounds the error of the empirical
mean estimator. It is re-stated and proved here for completeness. 
\begin{lem}
\label{lem: high probability MSE for empirical average}Consider $N$
\IID samples $\{Y(i)\}_{i\in[N]}$ from the model 
\[
Y=\theta+\epsilon
\]
where $\epsilon\sim{\cal N}(0,\sigma^{2}\overline{\Sigma})$ where
$\theta\in\mathbb{R}^{d}$ and $\overline{\Sigma}\in\overline{\mathbb{S}}_{d}^{++}$
($\overline{\mathbb{S}}_{d}^{++}$ is the positive semidefinite cone
of matrices whose trace is normalized to $d$). Let $\overline{\theta}(N)=\frac{1}{N}\sum_{i=1}^{n}Y(i)$
be the empirical mean estimator. Then, there exists a numerical constant
$c>1$ so that
\[
\|\overline{\theta}(N)-\theta\|^{2}\leq c\log\left(\frac{e}{\delta}\right)\cdot\frac{d\sigma^{2}}{N}
\]
with probability at least  $1-\delta$, for any $\delta\in(0,1)$. 
\end{lem}
\begin{proof}
It was shown by \citet[Lecture 4]{rigollet2020math} that with probability
at least $1-\delta$
\[
\left|\|\overline{\theta}(N)-\theta\|^{2}-\frac{\Tr[\sigma^{2}\overline{\Sigma}]}{N}\right|\leq c\left(\frac{\|\sigma^{2}\overline{\Sigma}\|_{\text{op}}}{N}\log\left(\frac{2}{\delta}\right)+\frac{\|\sigma^{2}\overline{\Sigma}\|_{\text{F}}}{N}\sqrt{\log\left(\frac{2}{\delta}\right)}\right)
\]
for some numerical constant $c>0$, which can be taken to be $c\geq1$.
Since $\|\sigma^{2}\overline{\Sigma}\|_{\text{op}}\leq\|\sigma^{2}\overline{\Sigma}\|_{\text{F}}\leq\Tr[\sigma^{2}\overline{\Sigma}]$
it holds that \textbf{
\begin{align*}
\|\overline{\theta}(N)-\theta\|^{2} & \leq\frac{\Tr[\sigma^{2}\overline{\Sigma}]}{N}\cdot c\left(1+\log\left(\frac{e}{\delta}\right)+\sqrt{\log\left(\frac{e}{\delta}\right)}\right)\\
 & \leq\frac{\Tr[\sigma^{2}\overline{\Sigma}]}{N}\cdot3c\log\left(\frac{e}{\delta}\right).
\end{align*}
}
\end{proof}

\section{The linear regression setting and high probability bounds on the
loss \label{sec:The-linear-regression}}

Rather than the parametric mean estimation problem considered in this
paper, \citet{xu2022statistical} considered the parametric linear
regression problem. We show that under the assumptions of \citet{du2021fewshot,xu2022statistical},
and for the purpose of the analysis in this paper, the parametric
mean estimation and linear regression problems are similar, which
allows to translate our results for the former to the latter, with
slight modifications. 

In the linear regression setting, the sample space is ${\cal Z}=\mathbb{R}^{d}\times\mathbb{R}$
and 
\[
\Phi_{t}:=\mathbb{R}^{d}\times\mathbb{R}_{+}\times\overline{\mathbb{S}}_{d}^{++}
\]
where $\phi_{t}:=(\theta_{t},\sigma_{t}^{2},\overline{\Sigma}_{t})$,
and $\overline{\mathbb{S}}_{d}^{++}$ is the positive semidefinite
cone of matrices whose trace is normalized to $d$. The $t$th model,
$t\in\llbracket T\rrbracket$, is given by 
\begin{equation}
{\cal M}_{t}\colon\quad Y_{t}=\langle X_{t},\theta_{t}\rangle+\epsilon_{t},\label{eq: linear regression models}
\end{equation}
where $X_{t}\sim{\cal N}(0,\overline{\Sigma}_{t})$ is a random feature
vector, where $\Tr[\overline{\Sigma}_{t}]=d$, $\theta_{t}\in\mathbb{\mathbb{R}}^{d}$
is the unknown parameter vector, and $\epsilon_{t}\sim{\cal N}(0,\sigma_{t}^{2})$
is a Gaussian noise. Upon samplinf from ${\cal M}_{t}$, the learner
observes the $Z_{t}=(X_{t},Y_{t})$. The loss function is given by
the squared prediction error 
\[
\ell\left(\phi,\tilde{\phi}\right):=\left(Y_{t}-\langle X_{t}^{\top},\tilde{\theta}\rangle\right)^{2}=\left(\langle X_{t}^{\top},\theta-\tilde{\theta}\rangle+\epsilon_{t}\right)^{2},
\]
and the excess risk is 
\[
L_{N}(\boldsymbol{\phi},{\cal A})=\E\left[\left(\langle X_{t}^{\top},\theta-\hat{\theta}\rangle+\epsilon_{0}\right)^{2}-\epsilon_{0}^{2}\right]=(\hat{\theta}-\theta_{0})^{\top}\Sigma_{0}(\hat{\theta}-\theta_{0}),
\]
where $\hat{\theta}\equiv\hat{\theta}((A_{i},S_{i})_{i\in[N]})$.
In addition to the assumption that $\|\theta_{t}\|\leq C_{\theta}$
for all $t\in\llbracket T\rrbracket$, we follow \citet{du2021fewshot,xu2022statistical}
and make the following assumption:
\begin{assumption}
\label{assu: condition number on Sigma}For the linear regression
setting, there exists constants $0<\underline{C}_{\Sigma}<\overline{C}_{\Sigma}$
such that
\[
\underline{C}_{\Sigma}\cdot I_{d}\preceq\Sigma_{t}\preceq\overline{C}_{\Sigma}\cdot I_{d}
\]
holds for any $t\in\llbracket T\rrbracket$. 
\end{assumption}
Under Assumption \ref{assu: condition number on Sigma} it can be
asserted that (see the proof of Lemma \ref{lem: high probability MSE for least squares}
in what follows) that
\begin{equation}
L_{N}(\boldsymbol{\phi},{\cal A})\asymp\E\left[\|\hat{\theta}-\theta_{0}\|^{2}\right],\label{eq: risk is equivalent to MSE}
\end{equation}
that is, the excess risk is proportional to the MSE of the parameter,
as in the mean estimation problem, for which $L_{N}(\boldsymbol{\phi},{\cal A})=\E[\|\hat{\theta}-\theta_{0}\|^{2}]$.
Therefore, while, in general, an accurate estimation of the parameter
\emph{suffices} for a low prediction risk, under Assumption \ref{assu: condition number on Sigma},
low prediction risk also \emph{necessitates} accurate estimation.
Therefore, we focused in this paper on the mean estimation setting
in which the risk itself is the MSE of an estimator, rather than the
linear regression setting. This estimation problem is simpler because
it does not involve the randomness of the features, only the noise. 

We next show that, similarly to the mean estimation setting, there
exists an estimator $\overline{\theta}(N)$ and $g(\delta)\colon[0,1]\to\mathbb{R}_{+}$
such that 
\[
\|\bar{\theta}(N)-\theta\|^{2}\leq g(\delta)\cdot\frac{d\sigma^{2}}{N}
\]
holds with probability at least $1-\delta$, for any $\delta\in(0,1)$.
Specifically, we will show next (Lemma \ref{lem: high probability MSE for least squares})
that if $\bar{\theta}(N)$ is the LS estimator, projected on $\mathbb{B}^{d}(C_{\theta})$,
then there exist universal constants $c,c_{1},c_{2}>0$ such that
if 
\begin{equation}
N\geq\frac{1}{c_{1}}\vee\frac{c_{0}\overline{C}_{\Sigma}}{\underline{C}_{\Sigma}}d\label{eq: minimal N for LS estimator analysis}
\end{equation}
then \eqref{eq: parametric error rate definition of c_delta} holds
with $g(\delta)=\nicefrac{c}{\underline{C}_{\Sigma}}\log^{2}(\nicefrac{4}{\delta})$.
Evidently, the error in linear regression and mean estimation behaves
roughly the same, where the former requires the condition \eqref{eq: minimal N for LS estimator analysis},
which can be assumed for the target model and all the $T$ source
models. The dependence on the error is also slightly different, and
is $\Theta(\log^{2}(\nicefrac{1}{\delta}))$ for the linear regression
setting compared to $\Theta(\log(\nicefrac{1}{\delta}))$ (sub-exponential)
for the mean estimation setting. 

To prove this, we will utilize the following high probability lower
bound on the minimal singular value of a random matrix, due to \citet[Theorem 3.1]{koltchinskii2015bounding}:
\begin{thm}[{\citet[Theorem 3.1]{koltchinskii2015bounding}}]
\label{thm:minimal singular values}Let $\boldsymbol{X}\in\mathbb{R}^{N\times d}$
be a random matrix comprised of \IID rows $\{X^{\top}(i)\}_{i\in[N]}$,
which are equal in distribution to $X\in\mathbb{R}^{N}$. Assume that
there exist constants $a,A,B$ so that for every $u\in\mathbb{S}^{d-1}$
\[
a\leq\sqrt{\E\left[\langle u,X\rangle^{2}\right]}\leq A,\quad\sqrt{\E\left[\langle u,X\rangle^{2}\right]}\leq B\cdot\E\left|\langle u,X\rangle\right|.
\]
Then, there exist constants $c_{0},c_{1},c_{2}>0$ so that for any
$N\geq c_{0}B^{4}(\frac{A}{a})^{2}d$
\[
\lambda_{\min}\left(\frac{1}{N}\boldsymbol{X}^{\top}\boldsymbol{X}\right)\geq c_{2}^{2}\frac{a^{2}}{B^{4}}
\]
with probability at least $1-\exp(-c_{1}B^{4}N)$. 
\end{thm}
We may now prove the high probability bound on the loss of the projected
LS estimator:
\begin{lem}
\label{lem: high probability MSE for least squares}Consider $N$
\IID samples $\{X(i),Y(i)\}_{i\in[N]}$ from the model 
\[
Y=\langle X,\theta\rangle+\epsilon
\]
 where $X\sim{\cal N}(0,\overline{\Sigma})$, with $\overline{\Sigma}\in\overline{\mathbb{S}}_{d}^{++}$,
$\theta\in\mathbb{B}^{d}(C_{\theta})$, and $\epsilon\sim{\cal N}(0,\sigma^{2})$.
Let 
\[
\boldsymbol{X}=\left[\begin{array}{ccc}
- & X^{\top}(1) & -\\
- & X^{\top}(2) & -\\
- & \vdots & -\\
- & X^{\top}(N) & -
\end{array}\right]\in\mathbb{R}^{N\times d}
\]
and $\boldsymbol{Y}=(Y(1),Y(2),\cdots,Y(N))^{\top}\in\mathbb{R}^{N}$
and $\boldsymbol{\epsilon}=(\epsilon(1),\epsilon(2),\cdots,\epsilon(N))^{\top}\in\mathbb{R}^{N}$
so that \textbf{$\boldsymbol{Y}=\boldsymbol{X}\theta+\boldsymbol{\epsilon}$}.
Consider the LS estimator
\[
\overline{\theta}_{\text{LS}}(N)=\left(\boldsymbol{X}^{\top}\boldsymbol{X}\right)^{\dagger}\boldsymbol{X}^{\top}\boldsymbol{Y},
\]
and let $\overline{\theta}(N)$ be $\overline{\theta}_{\text{LS}}(N)$
projected on $\mathbb{B}^{d}(C_{\theta})$. Then, there exist numerical
constants $c>0$ and $c_{0}\vee c_{2}<1<c_{1}$ such that if 
\begin{equation}
N\geq\underline{N}:=\frac{1}{c_{1}}\vee\frac{c_{0}\overline{C}_{\Sigma}}{\underline{C}_{\Sigma}}d\label{eq: condition on minimal N}
\end{equation}
then
\begin{equation}
\|\overline{\theta}(N)-\theta\|^{2}\le\frac{c}{c_{2}^{2}\cdot\underline{C}_{\Sigma}}\log^{2}\left(\frac{4}{\delta}\right)\frac{d\sigma^{2}}{N}\label{eq: high probability upper bound on linear regression estimation}
\end{equation}
with probability at least $1-\delta$, for any $\delta\in(0,1)$.
\end{lem}
\begin{proof}
It was shown by \citet[Lecture 6]{rigollet2020math} that with probability
at least $1-\nicefrac{\delta}{2}$
\[
\left|\frac{1}{N}\|\boldsymbol{X}\overline{\theta}(N)-\boldsymbol{X}\theta\|^{2}-\frac{\sigma^{2}\cdot\rank(\boldsymbol{X}^{\top}\boldsymbol{X})}{N}\right|\leq c\frac{\sigma^{2}}{N}\log\left(\frac{2}{\delta}\right)
\]
for some numerical constant $c>0$, which can be taken to be $c\geq1$.
Thus, with probability at least $1-\nicefrac{\delta}{2}$
\[
\frac{1}{N}\left\Vert \boldsymbol{X}\left(\overline{\theta}(N)-\theta\right)\right\Vert ^{2}\leq2c\frac{d\sigma^{2}}{N}\log\left(\frac{2}{\delta}\right).
\]
Moreover,
\begin{align*}
\frac{1}{N}\left\Vert \boldsymbol{X}\left(\overline{\theta}(N)-\theta\right)\right\Vert ^{2} & =\frac{1}{N}\left(\overline{\theta}(N)-\theta\right)^{\top}\boldsymbol{X}^{\top}\boldsymbol{X}\left(\overline{\theta}(N)-\theta\right)\\
 & \geq\frac{1}{N}\lambda_{\text{min}}(\boldsymbol{X}^{\top}\boldsymbol{X})\cdot\left\Vert \overline{\theta}(N)-\theta\right\Vert ^{2}
\end{align*}
and so with probability at least $1-\nicefrac{\delta}{2}$
\begin{equation}
\left\Vert \overline{\theta}(N)-\theta\right\Vert ^{2}\le\frac{2c\frac{\sigma^{2}d}{N}\log\left(\frac{2}{\delta}\right)}{\frac{1}{N}\lambda_{\text{min}}(\boldsymbol{X}^{\top}\boldsymbol{X})}.\label{eq: upper bound on LS error with minimal eigenvalue}
\end{equation}
Let $\sigma_{\min}(\boldsymbol{X})$ be the minimal singular value
of a matrix $\boldsymbol{X}$. We next use \citep[Theorem 3.1]{koltchinskii2015bounding},
reproduced as Theorem \ref{thm:minimal singular values} above. To
use this bound, we first verify the qualifying conditions of \citet[Assumption 3.1]{koltchinskii2015bounding}.
For brevity, let $X$ be distributed as $X(1)\sim{\cal N}(0,\Sigma)$.
Now, first, for any $u\in\mathbb{S}^{d-1}$
\[
\E\left[\langle X,u\rangle^{2}\right]=\E\left[u^{\top}XX^{\top}u\right]=u^{\top}\Sigma_{t}u
\]
and so Assumption \ref{assu: condition number on Sigma} results
\[
\sqrt{\underline{C}_{\Sigma}}\leq\sqrt{\E\left[\langle X,u\rangle^{2}\right]}\leq\sqrt{\overline{C}_{\Sigma}}.
\]
Second, since $X_{t}\sim{\cal N}(0,\Sigma)$ it holds that $\langle X,u\rangle\sim{\cal N}(0,u^{\top}\Sigma u)$
and so $\sqrt{\E[\langle X,u\rangle^{2}]}=\sqrt{u^{\top}\Sigma u}$
and $\E|\langle X,u\rangle|=\sqrt{\frac{2}{\pi}u^{\top}\Sigma u}$.
Hence, for any $u\in\mathbb{S}^{d-1}$ and $B\geq\sqrt{\nicefrac{2}{\pi}}$
it holds that 
\[
\sqrt{\E[\langle X,u\rangle^{2}]}\leq B\cdot\E\left|\langle X,u\rangle\right|.
\]
Given these properties, \citet[Theorem 3.1]{koltchinskii2015bounding}
proved that there exist numerical constants $c_{0},c_{1},c_{2}>0$,
so that the following holds: If 
\begin{equation}
N\geq c_{0}\frac{\overline{C}_{\Sigma}}{\underline{C}_{\Sigma}}B^{4}d\label{eq: condition on minimal N for minimal singular value}
\end{equation}
then
\begin{equation}
\frac{1}{N}\lambda_{\text{min}}(\boldsymbol{X}^{\top}\boldsymbol{X})\geq\frac{c_{2}^{2}\cdot\underline{C}_{\Sigma}}{B^{4}}\label{eq: high probability lower bound on the minimal singular value}
\end{equation}
with probability at least $1-\exp(-c_{1}B^{4}N)$. 

Assume that $N\geq\underline{N}=(\nicefrac{1}{c_{1}})\vee(\nicefrac{c_{0}\overline{C}_{\Sigma}d}{\underline{C}_{\Sigma}}).$
Let us choose 
\[
B^{4}=1\vee\left[\frac{1}{c_{1}N}\log\left(\frac{2}{\delta}\right)\right]
\]
for which $B\geq\sqrt{\nicefrac{2}{\pi}}$ holds. Consider two cases
based on the value of $\delta\in(0,1)$: First, assume that $\delta$
is large enough so that
\[
N\geq\sqrt{\frac{c_{0}\overline{C}_{\Sigma}}{c_{1}\underline{C}_{\Sigma}}\log\left(\frac{2}{\delta}\right)d}.
\]
Then, the condition \eqref{eq: condition on minimal N for minimal singular value}
of \citet[Theorem 3.1]{koltchinskii2015bounding}, that is, 
\[
N\geq\frac{c_{0}\overline{C}_{\Sigma}}{\underline{C}_{\Sigma}}\left\{ 1\vee\left[\frac{1}{c_{1}N}\log\left(\frac{2}{\delta}\right)\right]\right\} d,
\]
holds and so \eqref{eq: high probability lower bound on the minimal singular value}
also holds, with probability at least $1-\nicefrac{\delta}{2}$ .
Combining with \eqref{eq: upper bound on LS error with minimal eigenvalue}
and the union bound implies that for all $N\geq\nicefrac{1}{c_{1}}$
\begin{align*}
\left\Vert \overline{\theta}(N)-\theta\right\Vert ^{2} & \le2c\frac{1\vee\left[\frac{1}{c_{1}N}\log\left(\frac{2}{\delta}\right)\right]}{c_{2}^{2}\cdot\underline{C}_{\Sigma}}\log\left(\frac{2}{\delta}\right)\frac{\sigma^{2}d}{N}\\
 & \leq\frac{2c}{c_{2}^{2}\cdot\underline{C}_{\Sigma}}\log^{2}\left(\frac{4}{\delta}\right)\frac{\sigma^{2}d}{N},
\end{align*}
with probability at least $1-\delta$. Second, assume that $\delta$
is too small, so that 
\[
N\leq\sqrt{\frac{c_{0}\overline{C}_{\Sigma}}{c_{1}\underline{C}_{\Sigma}}\log\left(\frac{2}{\delta}\right)d}.
\]
then the condition \eqref{eq: condition on minimal N for minimal singular value}
of \citet[Theorem 3.1]{koltchinskii2015bounding} does not hold. However,
this case is equivalent to 
\begin{equation}
\log\left(\frac{2}{\delta}\right)\geq\frac{N^{2}}{d}\frac{c_{1}\underline{C}_{\Sigma}}{c_{0}\overline{C}_{\Sigma}}.\label{eq: regime of small delta for linear regression analysis}
\end{equation}
Now, since $\|\theta\|\vee\|\overline{\theta}(N)\|\leq C_{\theta}$
it holds with probability $1$ that 
\[
\left\Vert \overline{\theta}(N)-\theta\right\Vert ^{2}\leq4C_{\theta}^{2}.
\]
If we further assume the condition 
\[
N\geq2\cdot\frac{c_{0}^{3/4}c_{2}^{1/2}}{c_{1}^{3/4}c^{1/2}}\frac{\overline{C}_{\Sigma}^{3/4}}{\underline{C}_{\Sigma}^{1/4}}C_{\theta}\frac{d^{1/4}}{\sigma},
\]
then, it holds that 
\begin{align*}
4C_{\theta}^{2} & \leq\frac{c_{1}^{3/2}c\underline{C}_{\Sigma}^{1/2}}{c_{0}^{3/2}c_{2}\overline{C}_{\Sigma}^{3/2}}\cdot\frac{N^{3}}{d^{3/2}}\cdot\frac{\sigma^{2}d}{N}\\
 & \leq\frac{c_{1}^{3/2}c\underline{C}_{\Sigma}^{1/2}}{c_{0}^{3/2}c_{2}\overline{C}_{\Sigma}^{3/2}}\frac{c_{0}^{3/2}\overline{C}_{\Sigma}^{3/2}}{c_{1}^{3/2}\underline{C}_{\Sigma}^{3/2}}\log^{3/2}\left(\frac{2}{\delta}\right)\cdot\frac{\sigma^{2}d}{N}\\
 & =\frac{c}{c_{2}\underline{C}_{\Sigma}}\log^{3/2}\left(\frac{2}{\delta}\right)\cdot\frac{\sigma^{2}d}{N}
\end{align*}
where the second inequality follows from \eqref{eq: regime of small delta for linear regression analysis}. 

Combining both cases, we obtain 
\[
\left\Vert \overline{\theta}(N)-\theta\right\Vert ^{2}\leq\frac{2c}{c_{2}^{2}\cdot\underline{C}_{\Sigma}}\log^{2}\left(\frac{4}{\delta}\right)\frac{\sigma^{2}d}{N}
\]
using $c_{2}<1$ (so that $\nicefrac{1}{c_{2}^{2}}\leq\nicefrac{1}{c_{2}}$).
This implies that \eqref{eq: high probability upper bound on linear regression estimation}
holds with probability at least $1-\delta$. 
\end{proof}

\section{Comments on \citet{xu2022statistical} \label{sec:Comments-on-Xu-Tewari}}

\subsection{The minimax risk lower bound of \citet[Theorem 1]{xu2022statistical}\label{subsec: Xu Theorem 1}}

\citet[Theorem 1]{xu2022statistical} derived a lower bound on the
minimax risk of the strong-oracle learner. The bound is derived for
the linear regression model, but is essentially the same for the mean
estimation problem we consider, and so we describe that result in
the former setting (using our notation). This bound addresses the
minimax risk of a CL algorithm for the class $\boldsymbol{\Psi}_{\leq}(\boldsymbol{q})$
(defined in \eqref{eq: semi-local set}). The minimax risk is derived
for the strong oracle, which knows $\{Q_{t}^{2}\}_{t\in[T]}$, where
$Q_{t}:=\|\theta_{t}-\theta_{0}\|$, and is given by 
\begin{equation}
\frac{d\sigma_{0}^{2}}{N}\wedge\min_{t\in[T]}\left\{ \frac{d\sigma_{t}^{2}}{N}+Q_{t}^{2}\right\} .\label{eq: claimed lower bound for strong oracle Xu}
\end{equation}
This manifests that it is rate optimal for a strong oracle to allocate
all samples to a single model. The proof of this result appears in
\citet[Appendix A]{xu2022statistical}, and is based on the proof
of \citet{mousavi2020minimax}, which provided a lower bound for the
risk of a target task based on a single source model, in a transfer
learning problem. Both proofs reduce the learning problem to an hypothesis
testing problem, and use Fano's method to obtain the lower bound.
Nonetheless, \citet{mousavi2020minimax} considered a transfer learning
problem, in which the number of samples allocated to each of the models
is fixed in advanced. By contrast, the strong oracle utilized in the
proof of \citet[Theorem 1]{xu2022statistical} may optimize the number
of samples allocated to each of the models based on a (partial) knowledge
of $\{\theta_{t}\}_{t\in\llbracket T\rrbracket}$. Thus, the CL problem
should be reduced to an hypothesis testing problem in which the number
of samples allocated to each of the models is determined adaptively.
In turn, Fano's lower bound should be adapted to this setting, as
its standard version \citep{cover2012elements,wainwright2019high}
provides a lower bound on the error probability in an hypothesis testing
problem with a fixed number of samples for each of the models. This
is not reflected in the proof of \citet[Appendix A]{xu2022statistical},
which writes the lower bound resulting from Fano's method as (in our
notation) 
\begin{equation}
U^{2}\left(1-\frac{\log2+\sum_{t=1}^{T}N_{t}\cdot I(J;Y_{t})}{\log K}\right)\label{eq: claimed Fano's inequality Xu}
\end{equation}
where $U^{2}$ is the claimed lower bound on the minimax rate, $K$
is the number of hypotheses, $J\sim\text{Uniform}[K]$. Importantly,
in this bound, $N_{t}$, which is the number of samples allocated
to each of the models, statistically depends on $J$, but then it
is not obvious that \eqref{eq: claimed Fano's inequality Xu} is a
valid application of Fano's bound. That being said, the specific construction
of hypotheses used in the proof, indeed results the lower bound \eqref{eq: claimed lower bound for strong oracle Xu},
and we adopted a similar construction in our proofs. In short, source
task parameters with $q_{t}\gtrsim U$ are chosen to be the same for
all hypotheses, making them independent of $J$, and source task parameters
with $q_{t}\lesssim U$ are chosen as a $U$-packing set, with equal
parameters across these ``close'' models, for each hypothesis. Thus,
the allocation of samples to these source models can be chosen \WLOG
as allocating all samples to just one of them. This reduces the problem
to a single source task, and then it can be assumed, with only a constant
factor in the lower bound, that the source task and the target task
are each allocated $\nicefrac{N}{2}$ of the samples (with probability
$1$). Thus, the lower bound is intact.

In Appendix \ref{subsec:Proof-of-Theorem- first-minimax} we compare
our minimax lower bound (Theorem \ref{thm: lower bound global risk constant dimension})
with the bound of \citet[Theorem 2]{xu2022statistical}. 

\subsection{The algorithm and risk upper bound of \citet[Theorem 3]{xu2022statistical}
\label{subsec: Xu Theorem 3}}

\citet[Section 3.3]{xu2022statistical} proposed and analyzed a two-step
CL algorithm. The setting assumes that there is at least one source
model $t^{*}$ whose parameter equals exactly to the target parameter,
namely $Q_{t^{*}}=0$. Translating their algorithm for linear regression
to our mean estimation setting, the algorithm can be described as
first allocating $\nicefrac{N}{(2T)}$ samples to each of the $T$
source models, to obtain an initial estimates $\{\overline{\theta}_{t}(\nicefrac{N}{(2T)}\}_{t\in[T]}$,
using $\nicefrac{N}{2}$ to obtain an initial estimate for the target
parameter $\overline{\theta}_{0}(\nicefrac{N}{2})$ and then choosing
\[
t^{*}=\argmin_{t\in[T]}\left\Vert \overline{\theta}_{t}(\nicefrac{N}{(2T)})-\overline{\theta}_{0}(\nicefrac{N}{2})\right\Vert 
\]
so $t^{*}$ is an \rv. This algorithm was analyzed by \citet[Appendix C]{xu2022statistical},
though its statement in \citep[Theorem 3]{xu2022statistical} is inaccurate,
and thus was revised by \citet[Theorem 3.3]{xu2023benefits} to the
bound 
\begin{equation}
\|\overline{\theta}_{t^{*}}(N/(2T)-\theta_{0}\|\lesssim\log\left(\frac{Td}{\delta}\right)\cdot\left(\frac{\sigma_{0}^{2}}{N}+\frac{d\sigma_{t^{*}}^{2}}{N/T}+\sqrt{\frac{d}{N}}\right),\label{eq: bound of Xu Thereom 3}
\end{equation}
which holds with probability at least $1-\delta$ (In \citep[Theorem 3]{xu2022statistical}
only the first two terms appear). In \eqref{eq: bound of Xu Thereom 3},
the third term, $\sqrt{\nicefrac{d}{N}}$, prevents this bound from
achieving the typical fast parametric rate $\nicefrac{d\sigma^{2}}{N}$
associated with risks under mean squared error (MSE), and one which
does not vanish when the noise variances vanish, i.e., $\max_{t\in\llbracket T\rrbracket}\sigma_{t}\downarrow0$.
The first term, $\nicefrac{\sigma_{0}^{2}}{N}$, does not appear in
our bound, and its proof of \citet[Appendix C, Lemma 4]{xu2022statistical}
is somewhat unclear. In our setting and notation, the claim therein
can be roughly rephrased as ``\emph{martingale concentration inequality
on the sum $\|\hat{\theta}_{t}-\theta_{0}^{*}\|\cdot\sum_{i=1}^{\nicefrac{N}{2}}\epsilon_{0,i}$}''
where $\{\epsilon_{0,i}\}$ are \IID samples from ${\cal N}(0,\sigma_{0}^{2})$.
However, $\sum_{i=1}^{\nicefrac{N}{2}}\epsilon_{0,i}$ are actually
bounded \WHP as $O(\nicefrac{\sigma_{0}}{\sqrt{N}})$, and not as
$O(\nicefrac{\sigma_{0}^{2}}{N})$. We next discuss the second term
in \eqref{eq: bound of Xu Thereom 3}, $\nicefrac{d\sigma_{t^{*}}^{2}}{(\nicefrac{N}{T})}$.
Due to the assumption that there exists $t$ with $Q_{t}=0$ one may
speculate that this term would generalize to 
\[
\tilde{O}\left(Q_{t^{*}}^{2}+\frac{d\sigma_{t^{*}}^{2}}{\nicefrac{N}{T}}\right)
\]
when this assumption does not hold. It is not clear what can be assured
on $t^{*}$, and so it cannot be directly compared to our bound, in
which $t^{*}$ is chosen from a subset ${\cal T}\subseteq\llbracket T\rrbracket$,
where under favorable conditions it holds that ${\cal T}={\cal T}_{\text{w.o.}}$.
More importantly, our bound is 
\[
\min_{t\in{\cal T}}\tilde{O}\left(Q_{t}^{2}+\frac{d\sigma_{t}^{2}}{N}\right)
\]
and the noise term in our bound is $\nicefrac{d\sigma_{t^{*}}^{2}}{N}$,
which is a factor of $T$ smaller than the second term in \eqref{eq: bound of Xu Thereom 3}.
Finally, we prove a high probability bound as well as a bound that
holds for the risk (expected loss).

\section{Proof of Theorem \ref{thm: Single source CL algorithm} (a single
source model $T=1$) \label{sec:Proofs-for-single-source-model}}

The success of the method proposed in Theorem \ref{thm: Single source CL algorithm}
it is based on the ability of the learner to reliably eliminate the
source model whenever $Q_{1}^{2}$ is large. This ability is proved
in the following \emph{source elimination lemma}. It is stated in
slightly greater generality than needed for the $T=1$ setting, in
the sense that it also considers the case in which the noise variance
in the source model is \emph{larger} than the noise variance in the
target model (that is, the case $\lambda_{1}^{2}>\lambda_{0}^{2}$
in what follows). This form will be necessary for the analysis of
the multiple source models $T>1$. 
\begin{lem}[The source elimination lemma]
\label{lem: iden}For $t=0,1$, let $\theta_{t},\tilde{\theta}_{t}\in\mathbb{R}^{d}$
and $\lambda_{t}\in\mathbb{R}_{+}$ be such that 
\[
\|\tilde{\theta}_{t}-\theta_{t}\|^{2}\leq\lambda_{t}^{2}.
\]
Set $Q_{1}^{2}:=\|\theta_{1}-\theta_{0}\|^{2}$ and $\lambda_{\text{max}}:=\lambda_{0}\vee\lambda_{1}$.
Then, there exists a numerical constant $\nu\in[\nicefrac{1}{27},1]$
so that

\begin{equation}
\begin{cases}
Q_{1}^{2}\geq\nu\cdot\lambda_{\text{max}}^{2}, & \text{\emph{if}}\,\|\tilde{\theta}_{0}-\tilde{\theta}_{1}\|^{2}\geq10\lambda_{\text{max}}^{2}\\
Q_{1}^{2}\leq\frac{\lambda_{\text{max}}^{2}}{\nu}, & \text{\emph{otherwise}}
\end{cases}.\label{eq: basic identification}
\end{equation}
\end{lem}
\begin{proof}
We consider three cases. First suppose that $Q_{1}^{2}\leq\lambda_{\text{max}}^{2}$.
Then, 
\begin{align*}
\|\tilde{\theta}_{0}-\tilde{\theta}_{1}\|^{2} & =\|\tilde{\theta}_{0}-\theta_{0}+\theta_{1}-\tilde{\theta}_{1}+\theta_{0}-\theta_{1}\|^{2}\\
 & \trre[\leq,*]2\|\tilde{\theta}_{0}-\theta_{0}\|^{2}+2\|\theta_{1}-\tilde{\theta}_{1}+\theta_{0}-\theta_{1}\|^{2}\\
 & \trre[\leq,*]2\|\tilde{\theta}_{0}-\theta_{0}\|^{2}+4\|\theta_{1}-\tilde{\theta}_{1}\|^{2}+4\|\theta_{0}-\theta_{1}\|^{2}\\
 & \leq2\lambda_{0}^{2}+4\lambda_{1}^{2}+4Q_{1}^{2}\\
 & \leq10\lambda_{\text{max}}^{2},
\end{align*}
using $(a+b)^{2}\leq2a^{2}+2b^{2}$ in $(*)$.\footnote{The constant can be slightly improved if we use the refined inequality
$(a+b)^{2}\leq(1+\zeta^{-1})a^{2}+(1+\zeta)b^{2}$ and optimize over
$\zeta>0$. Note also that $\E[\|\tilde{\theta}_{0}-\tilde{\theta}_{1}\|^{2}]=\lambda_{0}^{2}+\lambda_{1}^{2}+Q_{1}^{2}\leq3\lambda_{\text{max}}^{2}$,
and since $\langle\tilde{\theta}_{0}-\theta_{0},\tilde{\theta}_{1}-\theta_{1}\rangle\lesssim\nicefrac{\lambda_{0}^{2}}{\sqrt{d}}$
\WHP (and similarly for the other mixed terms), the random error
concentrates fast around this expected value. Thus, the constant $10$
appearing in the bound can be improved by at most a factor of $\nicefrac{10}{3}$.} This shows that the first case in \eqref{eq: basic identification}
holds even for $\nu=1$. Second, suppose that $Q_{1}^{2}\geq27\cdot\lambda_{\text{max}}^{2}$.
Then,
\begin{align*}
\|\tilde{\theta}_{0}-\tilde{\theta}_{1}\| & \trre[\geq,a]\|\tilde{\theta}_{1}-\theta_{0}\|-\|\tilde{\theta}_{0}-\theta_{0}\|\\
 & \trre[\geq,a]\|\theta_{1}-\theta_{0}\|-\|\tilde{\theta}_{1}-\theta_{1}\|-\|\tilde{\theta}_{0}-\theta_{0}\|\\
 & \geq Q_{1}-\lambda_{1}-\lambda_{0}\\
 & \trre[\geq,b]\sqrt{27}\cdot\lambda_{\text{max}}-2\lambda_{\text{max}}\\
 & \geq\sqrt{10}\lambda_{\text{max}},
\end{align*}
where $(a)$ follows from the triangle inequality (twice), $(b)$
follows since $Q_{1}^{2}\geq27\cdot\lambda_{\text{max}}^{2}$ was
assumed. This shows that the second case in \eqref{eq: basic identification}
holds for $\nu=\nicefrac{1}{27}$. Thirdly, if $\lambda_{\text{max}}^{2}\leq Q_{1}^{2}\leq27\lambda_{\text{max}}^{2}$
then both statements in \eqref{eq: basic identification} hold (for
$\nu=\nicefrac{1}{27}$). Note that the numerical constant $\nu$
is not optimized, and so there exists some maximal $\nu\in[\nicefrac{1}{27},1]$
for which the statement also holds. 
\end{proof}
We next prove Theorem \ref{thm: Single source CL algorithm} using
the source elimination lemma, and show that the estimated output essentially
achieves the loss \eqref{eq: strong oracle} of the strong-oracle
learner.
\begin{proof}[of Theorem \ref{thm: Single source CL algorithm}]
Let ${\cal E}_{\delta}$ be the event in which 
\[
\|\bar{\theta}_{t}(\nicefrac{N}{2})-\theta\|^{2}=\|\tilde{\theta}_{t}-\theta\|^{2}\leq g\left(\frac{\delta}{2}\right)\frac{d\sigma_{t}^{2}}{\nicefrac{N}{2}}=:\lambda_{t}^{2}
\]
holds for both $t=0$ and $t=1$. By the union bound and the definition
of $g(\cdot)$, it holds that $\P[{\cal E}_{\delta}]\geq1-\delta$.
Under our assumptions $\lambda_{\text{max}}=\lambda_{0}\geq\lambda_{1}$.
Assuming ${\cal E}_{\delta}$ occur, Lemma \ref{lem: iden} implies
that if $\|\tilde{\theta}_{0}-\tilde{\theta}_{1}\|^{2}\geq10\lambda_{0}^{2}$
then $Q_{1}^{2}+\lambda_{1}^{2}\geq\nu\cdot\lambda_{0}^{2}$ and so
by the choice in \eqref{eq: choice of final estimate for single source algorithm}
$\hat{\theta}=\tilde{\theta}_{0}$ and 
\[
\|\hat{\theta}-\theta_{0}\|^{2}=\|\tilde{\theta}_{0}-\theta_{0}\|^{2}\leq\lambda_{0}^{2}=\lambda_{0}^{2}\wedge\frac{Q_{1}^{2}+\lambda_{1}^{2}}{\nu}\leq\frac{1}{\nu}\left[\lambda_{0}^{2}\wedge(Q_{1}^{2}+\lambda_{1}^{2})\right].
\]
Otherwise, if $\|\tilde{\theta}_{0}-\tilde{\theta}_{1}\|^{2}\leq10\lambda_{0}^{2}$
then Lemma \ref{lem: iden} implies that $\lambda_{0}^{2}\geq\nu Q_{1}^{2}$
and so $2\lambda_{0}^{2}\geq\nu(Q_{1}^{2}+\lambda_{1}^{2})$. So by
the choice in \eqref{eq: choice of final estimate for single source algorithm}
$\hat{\theta}=\tilde{\theta}_{1}$ and so 
\begin{align*}
\|\hat{\theta}-\theta_{0}\|^{2} & =\|\tilde{\theta}_{1}-\theta_{0}\|^{2}\\
 & =\|\theta_{1}-\theta_{0}+\tilde{\theta}_{1}-\theta_{1}\|^{2}\\
 & \leq2\|\theta_{1}-\theta_{0}\|^{2}+2\|\tilde{\theta}_{1}-\theta_{1}\|^{2}\\
 & \leq2Q_{1}^{2}+2\lambda_{1}^{2}\\
 & \leq\frac{4}{\nu}\left[\lambda_{0}^{2}\wedge(Q_{1}^{2}+\lambda_{1}^{2})\right].
\end{align*}
Combining both cases shows the claimed bound in \eqref{eq: single source high probability bound}.

We next prove the bound on the risk (expected loss). It holds by Lemma
\ref{lem: high probability MSE for empirical average} that 
\[
\|\tilde{\theta}_{t}-\theta_{t}\|^{2}=\|\bar{\theta}_{t}(\nicefrac{N}{2})-\theta_{t}\|^{2}\leq c\log\left(\frac{e}{\delta}\right)\cdot\frac{d\sigma_{t}^{2}}{\nicefrac{N}{2}}
\]
with probability at least $1-\delta$, both for $t\in\{0,1\}$. Thus,
both are sub-exponential \rv's. Consequently, it holds that $\E[\|\tilde{\theta}_{t}-\theta_{t}\|^{4}]\lesssim[\nicefrac{d\sigma_{t}^{2}}{(\nicefrac{N}{2})}]^{2}$.
More accurately, note that for any $r\geq0$ 
\[
\P\left[\|\tilde{\theta}_{t}-\theta_{t}\|^{4}\geq r\right]\leq e\exp\left[-\frac{\sqrt{r}}{c}\frac{\nicefrac{N}{2}}{d\sigma_{t}^{2}}\right],
\]
and so
\begin{align}
\E\left[\|\tilde{\theta}_{t}-\theta_{t}\|^{4}\right] & \trre[=,a]\int_{0}^{\infty}\d r\cdot\P\left[\|\tilde{\theta}_{t}-\theta_{t}\|^{4}\geq r\right]\nonumber \\
 & \leq e\int_{0}^{\infty}\d r\cdot\exp\left[-\frac{\sqrt{r}}{c}\frac{\nicefrac{N}{2}}{d\sigma_{t}^{2}}\right]\nonumber \\
 & \trre[=,b]e\left(\frac{cd\sigma_{t}^{2}}{\nicefrac{N}{2}}\right)^{2}2\int_{0}^{\infty}\d s\cdot s\exp\left[-s\right]\nonumber \\
 & =2ec^{2}\left(\frac{d\sigma_{t}^{2}}{\nicefrac{N}{2}}\right)^{2},\label{eq: fourth moment of estimation error}
\end{align}
where $(a)$ is by tail probability integration, and $(b)$ is using
the change of variables 
\[
s=\frac{\sqrt{r}}{c}\frac{\nicefrac{N}{2}}{d\sigma_{t}^{2}},
\]
 for which 
\[
\frac{\d s}{\d r}=\frac{\nicefrac{N}{2}}{cd\sigma_{t}^{2}}\frac{1}{2\sqrt{r}}=\left(\frac{\nicefrac{N}{2}}{cd\sigma_{t}^{2}}\right)^{2}\frac{1}{2s}.
\]
Hence, 
\begin{align*}
\E\left[\|\tilde{\theta}_{1}-\theta_{0}\|^{4}\right] & =\E\left[\|\tilde{\theta}_{1}-\theta_{1}+\theta_{1}-\theta_{0}\|^{4}\right]\\
 & \trre[\leq,a]8\E\left[\|\tilde{\theta}_{1}-\theta_{1}\|^{4}\right]+8\E\left[\|\theta_{1}-\theta_{0}\|^{4}\right]\\
 & \trre[\leq,b]48c^{2}\left(\frac{d\sigma_{1}^{2}}{\nicefrac{N}{2}}\right)^{2}+8Q_{1}^{4},
\end{align*}
where $(a)$ follows from $(a+b)^{4}\leq8a^{4}+8b^{4}$ (for $a,b\in\mathbb{R})$,
and $(b)$ follows from \eqref{eq: fourth moment of estimation error}.
Since $\sqrt{a+b}\leq\sqrt{a}+\sqrt{b}$ (for $a,b\geq0)$ we also
get
\begin{equation}
\sqrt{\E\left[\|\tilde{\theta}_{1}-\theta_{0}\|^{4}\right]}\leq\sqrt{48c^{2}\left(\frac{d\sigma_{1}^{2}}{\nicefrac{N}{2}}\right)^{2}+8Q_{1}^{2}}\leq7c\left[\frac{d\sigma_{1}^{2}}{\nicefrac{N}{2}}+Q_{1}^{2}\right].\label{eq: upper bound on the fourth moment square root of empirical source}
\end{equation}
Next we note that it holds with probability $1$ that $\hat{\theta}$
is either $\tilde{\theta}_{0}$ or $\tilde{\theta}_{1}$, and so 
\begin{align*}
\E\left[\|\hat{\theta}-\theta_{0}\|^{4}\right] & =\E\left[\|\tilde{\theta}_{0}-\theta_{0}\|^{4}\cdot\I\{\hat{\theta}=\tilde{\theta}_{0}\}\right]+\E\left[\|\tilde{\theta}_{1}-\theta_{0}\|^{4}\cdot\I\{\hat{\theta}=\tilde{\theta}_{1}\}\right]\\
 & \leq\E\left[\|\tilde{\theta}_{0}-\theta_{0}\|^{4}\right]+\E\left[\|\tilde{\theta}_{1}-\theta_{0}\|^{4}\right].
\end{align*}
Thus, 
\begin{align}
\sqrt{\E\left[\|\hat{\theta}-\theta_{0}\|^{4}\right]} & \leq\sqrt{\E\left[\|\tilde{\theta}_{0}-\theta_{0}\|^{4}\right]}+\sqrt{\E\left[\|\tilde{\theta}_{1}-\theta_{0}\|^{4}\right]}\nonumber \\
 & \trre[\leq,a]2c\left(\frac{d\sigma_{0}^{2}}{\nicefrac{N}{2}}\right)+7c\left[\frac{d\sigma_{1}^{2}}{\nicefrac{N}{2}}+Q_{1}^{2}\right]\nonumber \\
 & \leq14c\left(\frac{d\sigma_{0}^{2}}{N}+\frac{d\sigma_{1}^{2}}{N}+Q_{1}^{2}\right),\label{eq: upper bound on square root of fourth momen for alg estimator}
\end{align}
where $(a)$ follows from \eqref{eq: fourth moment of estimation error}
(with $t=0$) and \eqref{eq: upper bound on the fourth moment square root of empirical source}.
Using the above, we decompose

\begin{align}
\E\left[\|\hat{\theta}-\theta_{0}\|^{2}\right] & =\E\left[\|\hat{\theta}-\theta_{0}\|^{2}\cdot\I[{\cal E}_{\delta}]\right]+\E\left[\|\hat{\theta}-\theta_{0}\|^{2}\cdot\I[{\cal E}_{\delta}^{c}]\right]\nonumber \\
 & \trre[\leq,a]\E\left[\|\hat{\theta}-\theta_{0}\|^{2}\cdot\I[{\cal E}_{\delta}]\right]+\sqrt{\E\left[\|\hat{\theta}-\theta_{0}\|^{4}\right]\P[{\cal E}_{\delta}^{c}]}\nonumber \\
 & \trre[\leq,b]\frac{14c}{\nu}\left[\log(\frac{2}{\delta})\cdot\min\left\{ Q_{1}^{2}+\frac{d\sigma_{1}^{2}}{N},\frac{d\sigma_{0}^{2}}{N}\right\} +\left(\frac{d\sigma_{0}^{2}}{N}+\frac{d\sigma_{1}^{2}}{N}+Q_{1}^{2}\right)\cdot\sqrt{\delta}\right],\label{eq: decomposition of expected loss to high and low prob events}
\end{align}
where $(a)$ holds by Cauchy-Schwarz inequality, $(b)$ holds by the
high probability bound and \eqref{eq: upper bound on square root of fourth momen for alg estimator}.
Choosing $\delta\leq\delta_{0}$ where 
\[
\delta_{0}:=\left[\frac{\min\left\{ Q_{1}^{2}+\frac{d\sigma_{1}^{2}}{N},\frac{d\sigma_{0}^{2}}{N}\right\} }{\frac{d\sigma_{0}^{2}}{N}+\frac{d\sigma_{1}^{2}}{N}+Q_{1}^{2}}\right]^{2}
\]
assures that 
\begin{equation}
\E\left[\|\hat{\theta}-\theta_{0}\|^{2}\right]\leq\frac{14c}{\nu}\left(\log\left(\frac{2}{\delta}\right)+1\right)\cdot\min\left\{ Q_{1}^{2}+\frac{d\sigma_{1}^{2}}{N},\frac{d\sigma_{0}^{2}}{N}\right\} .\label{eq: bound in expectation single source final bound proof}
\end{equation}
It holds that 
\begin{align*}
\sqrt{\delta_{0}} & =\min\left\{ \frac{Q_{1}^{2}+\frac{d\sigma_{1}^{2}}{N}}{\frac{d\sigma_{0}^{2}}{N}+\frac{d\sigma_{1}^{2}}{N}+Q_{1}^{2}},\frac{\frac{d\sigma_{0}^{2}}{N}}{\frac{d\sigma_{0}^{2}}{N}+\frac{d\sigma_{1}^{2}}{N}+Q_{1}^{2}}\right\} \\
 & \geq\min\left\{ \frac{\sigma_{1}^{2}}{\sigma_{0}^{2}+\sigma_{1}^{2}},\frac{1}{4},\frac{\frac{d\sigma_{0}^{2}}{N}}{2Q_{1}^{2}}\right\} \\
 & \geq\min\left\{ \frac{\sigma_{1}^{2}}{\sigma_{0}^{2}+\sigma_{1}^{2}},\frac{1}{4},\frac{d\sigma_{0}^{2}}{8C_{\theta}^{2}N}\right\} \\
 & =\sqrt{\delta_{*}},
\end{align*}
where the first inequality holds since $\nicefrac{(a_{1}+a_{2})}{(a_{3}+a_{4})}\geq(\nicefrac{a_{1}}{a_{3}})\wedge(\nicefrac{a_{2}}{a_{4}})$
for $a_{1},a_{2},a_{3},a_{4}\geq0$ and since $\sigma_{1}^{2}\leq\sigma_{0}^{2}$.
Inserting $\delta_{*}$ to \eqref{eq: bound in expectation single source final bound proof}
and simplifying leads to the claim bound \eqref{eq: single sources bound in expectation}. 
\end{proof}

\section{Proof of Theorem \ref{thm: Multiple source CL algorithm} (multiple
source models $T>1$) \label{sec:Proofs-for-Section-algorithm}}
\begin{proof}[of Theorem \ref{thm: Multiple source CL algorithm}]
Algorithm \ref{alg: CL multiple sources} uses at most $N$ samples
since we set $\bar{N}=\nicefrac{N}{\bar{r}+2}$. Assume that the event
${\cal E}_{\delta}$ occurs, in which the empirical averages $\bar{\theta}_{0}(\bar{N})$
and $\bar{\theta}_{t}(\bar{N}_{t,r})$ are close to $\theta_{0}$
and $\theta_{t}$, respectively, as
\[
\|\bar{\theta}_{0}(\bar{N})-\theta_{0}\|^{2}\leq g(\bar{\delta})\frac{d\sigma_{0}^{2}}{\bar{N}}
\]
\[
\|\bar{\theta}_{t}(\bar{N}_{t,r})-\theta_{t}\|^{2}\leq g(\bar{\delta})\frac{d\sigma_{t}^{2}}{\bar{N}_{t,r}},
\]
for all $\bar{r}$ rounds, and all $t\in[T]$, and similarly, that
the last round estimate is close to $\theta_{t^{*}}$, that is
\[
\|\overline{\theta}_{t^{*}}(\bar{N})-\theta_{t^{*}}\|^{2}\leq g(\bar{\delta})\frac{d\sigma_{t^{*}}^{2}}{\bar{N}}.
\]
By the union bound, all these $T\bar{r}+2$ events hold with probability
at least $1-\delta$. We continue the analysis assuming that ${\cal E}_{\delta}$
occurs. Under this event, Corollary \ref{cor: source identification}
implies that after the first round only the source models $[T_{1}]$
with $T_{1}=T\beta_{\delta}(1)$ are retained, after the second round
only $[T_{2}]$ with $T_{2}=T\beta_{\delta}(\beta_{\delta}(1))$,
and so on. After $\bar{r}$ rounds, it holds that ${\cal T}_{\text{alg}}=[T_{\bar{r}}]$.
The algorithm then chooses an arbitrary $t^{*}\in[T_{\bar{r}}]$ and
the final high probability bound \eqref{eq: MSE multiple source high probability}
stems from
\begin{align*}
\|\hat{\theta}-\theta_{0}\|^{2} & =\|\overline{\theta}_{t^{*}}(\bar{N})-\theta_{0}\|^{2}\\
 & =\|\overline{\theta}_{t^{*}}(\bar{N})-\theta_{0}+\overline{\theta}_{t^{*}}(\bar{N})-\theta_{t^{*}}\|^{2}\\
 & \trre[\leq,a]2\|\theta_{t^{*}}-\theta_{0}\|^{2}+2\|\overline{\theta}_{t^{*}}(\bar{N})-\theta_{t^{*}}\|^{2}\\
 & \leq2Q_{t^{*}}^{2}+2g(\bar{\delta})\cdot\frac{d\sigma_{t^{*}}^{2}}{\bar{N}}.
\end{align*}
We next consider the bound in expectation, for which we set $g(\delta)=c\log(\nicefrac{e}{\delta})$.
The proof of follows the same lines as the proof in the single source
case (Theorem \ref{thm: Single source CL algorithm}), and so we briefly
outline it for the multiple source case. In general, we will be rather
loose with the upper bounds, since they only affect a logarithmic
factor. Note that the final empirical estimator uses $\bar{N}$ samples.
So, if we let $\tilde{\theta}_{t}=\bar{\theta}_{t}(\bar{N})$ be the
empirical estimator in the last round in which model $t$ is retained,
it holds that 
\[
\E\left[\|\tilde{\theta}_{t}-\theta_{t}\|^{4}\right]\leq2ec^{2}\left(\frac{d\sigma_{t}^{2}}{\bar{N}}\right)^{2}.
\]
Hence, similarly to the derivations leading to \eqref{eq: upper bound on the fourth moment square root of empirical source},
it holds for any $t\in[T]$ that
\[
\sqrt{\E\left[\|\tilde{\theta}_{t}-\theta_{0}\|^{4}\right]}\leq48c^{2}\left(\frac{d\sigma_{t}^{2}}{\bar{N}}\right)^{2}+8Q_{t}^{4}.
\]
Next we note that it holds with probability $1$ that $\hat{\theta}$
is from the set $\{\tilde{\theta}_{t}\}_{t\in\llbracket T\rrbracket}$,
and so 
\[
\E\left[\|\hat{\theta}-\theta_{0}\|^{4}\right]\leq\sum_{t\in\llbracket T\rrbracket}\E\left[\|\tilde{\theta}_{t}-\theta_{0}\|^{4}\right]
\]
Since $\sqrt{\sum_{i}a_{i}}\leq\sum_{i}\sqrt{a_{i}}$ we get
\[
\sqrt{\E\left[\|\hat{\theta}-\theta_{0}\|^{4}\right]}\leq\sum_{t\in\llbracket T\rrbracket}7c\left[\frac{d\sigma_{t}^{2}}{\bar{N}}+Q_{t}^{2}\right]
\]
(with $Q_{0}=0$). Using the above, we decompose as in \eqref{eq: decomposition of expected loss to high and low prob events}
and rearrange to obtain

\begin{align*}
\E\left[\|\hat{\theta}-\theta_{0}\|^{2}\right] & =\E\left[\|\hat{\theta}-\theta_{0}\|^{2}\cdot\I[{\cal E}_{\delta}]\right]+\E\left[\|\hat{\theta}-\theta_{0}\|^{2}\cdot\I[{\cal E}_{\delta}^{c}]\right]\\
 & \leq8c(\bar{r}+2)\cdot\log\left(\frac{(T\bar{r}+2)}{\delta}\right)\cdot\left[\frac{d\sigma_{\overline{t}({\cal T}_{\text{alg}})}^{2}}{N}+Q_{\overline{t}({\cal T}_{\text{alg}})}^{2}+\left(\sum_{t\in\llbracket T\rrbracket}\left[\frac{d\sigma_{t}^{2}}{\bar{N}}+Q_{t}^{2}\right]\right)\sqrt{\delta}\right],
\end{align*}
Choosing $\delta\leq\delta_{0}$ where now
\[
\delta_{0}:=\left[\frac{\frac{d\sigma_{\overline{t}({\cal T}_{\text{alg}})}^{2}}{N}+Q_{\overline{t}({\cal T}_{\text{alg}})}^{2}}{\sum_{t\in\llbracket T\rrbracket}\left[\frac{d\sigma_{t}^{2}}{\bar{N}}+Q_{t}^{2}\right]}\right]^{2}
\]
assures that 
\begin{equation}
\E\left[\|\hat{\theta}-\theta_{0}\|^{2}\right]\leq16c(\bar{r}+2)\cdot\log\left(\frac{(T\bar{r}+2)}{\delta}\right)\cdot\left[\frac{d\sigma_{\overline{t}({\cal T}_{\text{alg}})}^{2}}{N}+Q_{\overline{t}({\cal T}_{\text{alg}})}^{2}\right].\label{eq: bound in expectation multiple source final bound proof}
\end{equation}
It holds that 
\begin{align*}
\sqrt{\delta_{0}} & =\frac{\frac{d\sigma_{\overline{t}({\cal T}_{\text{alg}})}^{2}}{N}+Q_{\overline{t}({\cal T}_{\text{alg}})}^{2}}{\sum_{t\in\llbracket T\rrbracket}\left[\frac{d\sigma_{t}^{2}}{\bar{N}}+Q_{t}^{2}\right]}\\
 & \geq\frac{\min_{t\in\llbracket T\rrbracket}\frac{d\sigma_{t}^{2}}{N}+Q_{t}^{2}}{\sum_{t\in\llbracket T\rrbracket}\left[\frac{d\sigma_{t}^{2}}{\bar{N}}+Q_{t}^{2}\right]}\\
 & \geq\frac{1}{(T+1)}\cdot\frac{\min_{t\in\llbracket T\rrbracket}\frac{d\sigma_{t}^{2}}{N}+Q_{t}^{2}}{\max_{t\in\llbracket T\rrbracket}\frac{d\sigma_{t}^{2}}{\bar{N}}+Q_{t}^{2}}\\
 & \geq\frac{1}{(T+1)}\cdot\frac{\min_{t\in\llbracket T\rrbracket}\frac{d\sigma_{t}^{2}}{N}+Q_{t}^{2}}{\max_{t\in\llbracket T\rrbracket}\frac{d\sigma_{t}^{2}}{\bar{N}}+4C_{\theta}^{2}}\\
 & \geq\frac{1}{(T+1)}\cdot\frac{\min_{t\in\llbracket T\rrbracket}\frac{d\sigma_{t}^{2}}{N}}{\max_{t\in\llbracket T\rrbracket}\frac{d\sigma_{t}^{2}}{\bar{N}}+4C_{\theta}^{2}}\\
 & \geq\frac{1}{2(T+1)}\cdot\left[\min_{t,t'\in\llbracket T\rrbracket}\frac{\sigma_{t}^{2}}{\sigma_{t'}^{2}}\vee\min_{t\in\llbracket T\rrbracket}\frac{d\sigma_{t}^{2}}{8NC_{\theta}^{2}}\right]\\
 & =\sqrt{\delta_{*}}.
\end{align*}
Inserting $\delta_{*}$ to \eqref{eq: bound in expectation multiple source final bound proof}
and simplifying leads to the claim bound \eqref{eq: multiple sources bound in expectation}.
\end{proof}

\section{An extension of Algorithm \ref{alg: CL multiple sources} to unknown
noise covariance matrices \label{sec:unknown noise covariance}}

The covariance matrix of the noise in the mean estimation model ${\cal M}_{t}$
is $\Sigma_{t}=\sigma_{t}^{2}\cdot\overline{\Sigma}_{t}$ where $\overline{\Sigma}_{t}\in\overline{\mathbb{S}}_{d}^{++}$.
We have assumed thus far that $\{\Sigma_{t}\}_{t\in\llbracket T\rrbracket}$
are known to the learner. In this appendix, we consider the case in
which the learner knows that $\overline{\Sigma}_{t}=I_{d}$ for all
$t\in\llbracket T\rrbracket$, but not $\{\sigma_{t}^{2}\}_{t\in\llbracket T\rrbracket}$,
and then the case that $\{\Sigma_{t}\}_{t\in\llbracket T\rrbracket}$
are completely unknown. We show that under rather mild conditions,
Algorithm \ref{alg: CL multiple sources} can be further extended
to these settings too. We focus on the mean estimation setting, though
a similar extension can be made in the linear regression setting too
(Appendix \ref{sec:The-linear-regression}).

The idea is straightforward and is based on a preliminary step of
estimating either $\{\sigma_{t}^{2}\}_{t\in\llbracket T\rrbracket}$
(for the first setting) or $\{\Sigma_{t}\}_{t\in\llbracket T\rrbracket}$
(for the second setting), and then plugging in the estimated values
into Algorithm \ref{alg: CL multiple sources}, instead of the true
values. Additional minor modifications then guarantee an upper bound
on the loss which is of the same order as in the case the noise covariance
is known to the learner. The aforementioned conditions assure that
the additional error due to the use of inaccurate values of $\{\sigma_{t}^{2}\}_{t\in\llbracket T\rrbracket}$
or $\{\Sigma_{t}\}_{t\in\llbracket T\rrbracket}$ is negligible compared
to the loss of the algorithm. 

For the case of unknown $\{\sigma_{t}^{2}\}_{t\in\llbracket T\rrbracket}$,
the condition is provided in Proposition \ref{prop: Algorithm and guarantees for unknown variances},
and requires that $dN\gtrsim\bar{r}T$. In the regime discussed above
in Proposition \ref{prop: beta strictly bounded } it is required
that $\bar{r}=\Theta(\log T)$, and then the qualifying condition
is $dN=\Omega(T\log T)$. This is a fairly mild condition since $N\geq T$
is assumed, and typically one would like to have in practice $N\gg T$
to allow exploration of all source models. The proof of Proposition
\ref{prop: Algorithm and guarantees for unknown variances} is based
on Lemma \ref{lem:empirical variance estimation scaling}, which provides
a \emph{multiplicative} confidence interval for an estimator $\hat{\sigma^{2}}(K)$
that uses $K$ samples to estimate the variance $\sigma^{2}$. It
is shown that $\frac{1}{2}\sigma^{2}\leq\hat{\sigma^{2}}(K)\leq\frac{3}{2}\sigma^{2}$
with probability at least $1-2e^{-d(K-1)/24}$, and this result is
used with $K=\nicefrac{N}{T}$. 

For the case of unknown $\{\Sigma_{t}\}_{t\in\llbracket T\rrbracket}$,
the idea is similar, and compared to the case of unknown $\{\sigma_{t}^{2}\}_{t\in\llbracket T\rrbracket}$,
only Lemma \ref{lem:empirical variance estimation scaling} is adapted.
As shown in Lemma \ref{lem: high probability MSE for empirical average}
(Appendix \ref{sec:High-probability-bounds}), the noise covariance
matrix $\Sigma_{t}$ affects the estimation error of the empirical
mean estimators $\overline{\theta}_{t}(N)$ error via $\varsigma_{t}^{2}:=\frac{1}{d}\Tr[\Sigma_{t}]$.
Lemma \ref{lem:empirical trace estimation scaling} then replaces
the confidence interval on the estimator of $\sigma_{t}^{2}$ with
a similar confidence interval on $\varsigma_{t}^{2}$. It is shown
that $\frac{1}{2}\varsigma^{2}\leq\hat{\varsigma^{2}}(K)\leq\frac{3}{2}\varsigma^{2}$
with probability at least $1-2e^{-[(K-1)/24-\log d]}$, and this result
is used with $K=\nicefrac{N}{T}$. Given Lemma \ref{lem:empirical trace estimation scaling},
the required modifications to Proposition \ref{prop: Algorithm and guarantees for unknown variances}
are minor and thus not explicitly stated. The main modification is
that the condition \eqref{eq: condition for reliability of variance estimation}
is replaced by the more stringent condition $dN\gtrsim d\log d\cdot\bar{r}T.$
That is, $N$ in the case of unknown covariance matrix should be larger
by a factor of $d\log d$ compared to the case of unknown variance.

\subsection{Unknown noise variances}

We employ a preliminary step to Algorithm \ref{alg: CL multiple sources}
in which $\{\sigma_{t}^{2}\}_{t\in\llbracket T\rrbracket}$ are estimated.
Then, the estimated values are plugged into Algorithm \ref{alg: CL multiple sources}
instead of the true values. Since the performance guarantees of Algorithm
\ref{alg: CL multiple sources} are based on the order $\Theta(\nicefrac{d\sigma_{t}^{2}}{N})$,
it suffices that the estimated value of $\sigma_{t}^{2}$ should match
the true value up to multiplicative constants. As we next show, under
mild conditions, the additional error of Algorithm \ref{alg: CL multiple sources}
due to this preliminary estimation step is negligible. The necessary
conditions and modifications are stated in the following proposition: 
\begin{prop}
\label{prop: Algorithm and guarantees for unknown variances}Let $\delta\in(0,1)$
be given. Assume that $N\geq2(\bar{r}+3)(T+1)$ and that 
\begin{equation}
dN\geq48(\bar{r}+3)(T+1)\cdot\log\frac{2}{\bar{\delta}}.\label{eq: condition for reliability of variance estimation}
\end{equation}
Suppose that Algorithm \ref{alg: CL multiple sources} is run with
the following changes:
\begin{enumerate}
\item Line \ref{line: algorithm delta_bar}: Set 
\[
\bar{N}\leftarrow\frac{N}{\bar{r}+3}
\]
 and 
\[
\bar{\delta}\leftarrow\frac{\delta}{T(\bar{r}+1)+3}.
\]
\item The values $\{\sigma_{t}^{2}\}_{t\in\llbracket T\rrbracket}$ are
not input to the algorithm, and instead, there is an additional step
after line \ref{line: algorithm delta_bar}: The noise variances are
estimated as $\hat{\sigma_{t}^{2}}(\nicefrac{\bar{N}}{(T+1)})$ for
$t\in\llbracket T\rrbracket$.
\item \label{enu: changes to unknown variance - sample allocation}Line
\ref{line: algorithm sample allocation}: The number of samples allocated
in the $r$th round to the $t\in{\cal T}_{r-1}$ is 
\[
\hat{N}_{r,t}:=\frac{\bar{N}}{T_{r-1}}\cdot\frac{\hat{\sigma_{t}^{2}}\left(\frac{\bar{N}}{T+1}\right)}{\frac{1}{T_{r-1}}\sum_{t=1}^{T_{r-1}}\hat{\sigma_{t}^{2}}\left(\frac{\bar{N}}{T+1}\right)}.
\]
\item \label{enu: changes to unknown variance - elimination rule}Line \ref{line: algorithm elimination rule}:
Model $t\in{\cal T}_{r-1}$ is eliminated at the $r$th round if 
\[
\|\tilde{\theta}_{0}-\tilde{\theta}_{t}\|^{2}\geq10g(\bar{\delta})\left[\frac{d\cdot\frac{1}{T_{r-1}}\sum_{t=1}^{T_{r-1}}\hat{\sigma_{t}^{2}}\left(\frac{\bar{N}}{T+1}\right)}{\bar{N}/T_{r-1}}\vee\frac{d\hat{\sigma_{0}^{2}}}{\bar{N}}\right].
\]
\end{enumerate}
Then, the performance guarantees of Theorem \ref{thm: Multiple source CL algorithm}
hold with the pre-factor changed from $\bar{r}+2$ to $\bar{r}+3$
and in the definition of $\beta_{\delta}(\tau)$, the constant $\nu$
is replaced by $\nicefrac{\nu}{6}$. 
\end{prop}
It should be stressed that items \ref{enu: changes to unknown variance - sample allocation}
and \ref{enu: changes to unknown variance - elimination rule} can
indeed be computed by the learner, since they do not depend on the
unknown $\sigma_{t}^{2}$; only on their estimated values. As long
as the mild condition \eqref{eq: condition for reliability of variance estimation}
holds, the modified algorithm is efficient, and performs as well as
Algorithm \ref{alg: CL multiple sources}, up to numerical constants.

The proof of Proposition \ref{prop: Algorithm and guarantees for unknown variances}
is based on the following lemma:
\begin{lem}
\label{lem:empirical variance estimation scaling}Assume that $Y=\theta+\epsilon$
where $\theta\in\mathbb{\mathbb{R}}^{d}$ and $\epsilon\sim{\cal N}(0,\sigma^{2}\cdot I_{d})$.
Let $\delta\in(0,1)$ be given. Consider an estimator for $\sigma^{2}$
based on $K$ \IID samples $\{Y(i)\}_{i\in[K]}$ given by 
\begin{equation}
\hat{\sigma^{2}}(K):=\frac{1}{d(K-1)}\sum_{i=1}^{K}\left\Vert Y(i)-\frac{1}{K}\sum_{i'=1}^{K}Y(i')\right\Vert ^{2}.\label{eq: estimator for sigma square}
\end{equation}
Then, with probability at least $1-2e^{-d(K-1)/24}$
\[
\frac{1}{2}\sigma^{2}\leq\hat{\sigma^{2}}(K)\leq\frac{3}{2}\sigma^{2}.
\]
\end{lem}
\begin{proof}
Let $Y(i)=(Y(i;1),\ldots,Y(i;d))$. As is well known, (\eg, \citep[Theorem 8.3.1]{schervish2014probability})
it holds for any $j\in[d]$ that 
\[
\tilde{\sigma^{2}}(j)=\frac{1}{K-1}\sum_{i=1}^{K}\left(Y(i;j)-\frac{1}{K}\sum_{i'=1}^{K}Y(i';j)\right)^{2}
\]
is an unbiased estimator of $\sigma^{2}$ (that is, $\E[\tilde{\sigma^{2}}(j)]=\sigma^{2}$),
and moreover, that 
\[
\frac{(K-1)}{\sigma^{2}}\cdot\tilde{\sigma^{2}}(j)\sim\chi_{K-1}^{2}.
\]
Now, $\{\tilde{\sigma^{2}}(j)\}_{j\in[d]}$ are \IID, and so 
\[
\sum_{j=1}^{d}\frac{(K-1)}{\sigma^{2}}\cdot\tilde{\sigma^{2}}(j)\sim\chi_{d(K-1)}^{2}.
\]
Hence,
\[
\frac{d(K-1)}{\sigma^{2}}\cdot\hat{\sigma^{2}}(K)=\frac{1}{\sigma^{2}}\sum_{i=1}^{K}\sum_{j=1}^{d}\left(Y(i;j)-\frac{1}{K}\sum_{i'=1}^{K}Y(i';j)\right)^{2}\sim\chi_{d(K-1)}^{2}.
\]
By \citet[Lemma 1]{laurent2000adaptive} or \citep[Remark 2.11]{boucheron2013concentration}
(note that $\E[\chi_{\ell}^{2}]=\ell$) it holds that 
\[
\P\left[\left|\chi_{\ell}^{2}-\ell\right|\geq2\sqrt{\ell t}+2t\right]\leq2e^{-t}.
\]
and so 
\[
\P\left[\left|\hat{\sigma^{2}}(K)-\sigma^{2}\right|\geq\left[2\sqrt{\frac{t}{d(K-1)}}+\frac{2t}{d(K-1)}\right]\cdot\sigma^{2}\right]\leq2e^{-t}.
\]
Choosing $t=\nicefrac{d(K-1)}{24}$ we get that the deviation pre-factor
is 
\[
\frac{2\sqrt{\ell t}+2t}{d(K-1)}\leq\frac{1}{2}
\]
which, in turn, implies that 
\[
\frac{1}{2}\sigma^{2}\leq\hat{\sigma^{2}}(K)\leq\frac{3}{2}\sigma^{2}
\]
with probability at least $1-2\exp[-\nicefrac{d(K-1)}{24}]$.
\end{proof}
We may now prove Proposition \ref{prop: Algorithm and guarantees for unknown variances}. 
\begin{proof}[of Proposition \ref{prop: Algorithm and guarantees for unknown variances}]
In order to accommodate for the unknown variances, we need to consider
their influence on a few steps of the algorithm. First, their influence
on the additional samples required for variance estimation and the
additional possible events that one of these estimated values is inaccurate.
Second, their influence on the allocation of samples at each round,
that is, the computation of $\bar{N}_{t,r}$. This, in turn, affects
the condition used by the algorithm to eliminate models, and on the
models which satisfy this condition. 

\uline{The number of samples and the reliability:} Instead of $\bar{r}+2$
batches of equal size samples (combined from all the models) we now
have $\bar{r}+3$ including the one needed for variance estimation.
So we redefine $\bar{N}=\nicefrac{N}{(\bar{r}+3)}$. We add to the
union bound over $T\bar{r}+2$ events, additional $T+1$ events, which
assure that $\overline{\sigma_{t}^{2}}$ are on the same scale as
$\sigma_{t}^{2}$ for all $t\in\llbracket T\rrbracket$. To handle
this, we redefine $\bar{\delta}=\nicefrac{\delta}{T(\bar{r}+1)+3}$,
and require that 
\begin{equation}
\frac{1}{2}\sigma_{t}^{2}\leq\hat{\sigma_{t}^{2}}\left(\frac{\bar{N}}{T+1}\right)\leq\frac{3}{2}\sigma_{t}^{2}\label{eq: high probability event for noise variances}
\end{equation}
with probability at least $1-\bar{\delta}$. By Lemma \ref{lem:empirical variance estimation scaling},
this holds if 
\[
2e^{-d(\bar{N}/(T+1)-1)/24}\leq\bar{\delta},
\]
or, equivalently
\[
d\left(\frac{N}{(\bar{r}+3)(T+1)}-1\right)\geq24\cdot\log\frac{2}{\bar{\delta}}.
\]
Under the assumption of the lemma, a simple (looser) sufficient condition
for this to hold is \eqref{eq: condition for reliability of variance estimation}. 

\uline{The sample allocation at each round and the elimination
condition:} Algorithm \ref{alg: CL multiple sources} uses the allocation
\[
\bar{N}_{t,r}=\frac{\bar{N}}{T_{r-1}}\cdot\frac{\sigma_{t}^{2}}{\overline{\sigma^{2}}({\cal T}_{r-1})}
\]
and thus depends on the noise variances. One may replace the exact
variances with their estimates as 
\begin{equation}
\hat{N}_{r,t}:=\frac{\bar{N}}{T_{r-1}}\cdot\frac{\hat{\sigma_{t}^{2}}\left(\frac{\bar{N}}{T+1}\right)}{\frac{1}{T_{r-1}}\sum_{t=1}^{T_{r-1}}\hat{\sigma_{t}^{2}}\left(\frac{\bar{N}}{T+1}\right)}.\label{eq: proof unknown variances - empirical sample allocation}
\end{equation}
Corollary \ref{cor: source identification} guarantees that if 
\begin{equation}
Q_{t}^{2}\geq2\cdot\frac{10g(\bar{\delta})}{\nu}\left[\frac{d\cdot\sigma_{t}^{2}}{\hat{N}_{r,t}}\vee\frac{d\sigma_{0}^{2}}{\bar{N}}\right]\label{eq: proof unknown variances - elimination condition on squared distance first}
\end{equation}
then
\begin{equation}
\|\tilde{\theta}_{0}-\tilde{\theta}_{t}\|^{2}\geq20g(\bar{\delta})\left[\frac{d\cdot\sigma_{t}^{2}}{\hat{N}_{r,t}}\vee\frac{d\sigma_{0}^{2}}{\bar{N}}\right]\label{eq: proof unknown variances - elimination rule first}
\end{equation}
will occur \WHP, and then ${\cal M}_{t}$ will be eliminated. Note
that the above statement uses the actual number of samples $\hat{N}_{r,t}$
taken by the algorithm, yet the exact values of the noise variances
$\sigma_{0}^{2},\sigma_{t}^{2}$ affecting the empirical estimates
$\tilde{\theta}_{0}$ and $\tilde{\theta}_{t}$. Also note the additional
factor of $2$ to be used. Now, under the high probability event \eqref{eq: high probability event for noise variances},
which is assumed to hold, the event in \eqref{eq: proof unknown variances - elimination rule first}
also implies that 
\begin{align*}
\|\tilde{\theta}_{0}-\tilde{\theta}_{t}\|^{2} & \geq10g(\bar{\delta})\left[\frac{d\cdot\hat{\sigma_{t}^{2}}\left(\frac{\bar{N}}{T+1}\right)}{\hat{N}_{r,t}}\vee\frac{d\hat{\sigma_{0}^{2}}}{\bar{N}}\right]\\
 & =10g(\bar{\delta})\left[\frac{d\cdot\frac{1}{T_{r-1}}\sum_{t=1}^{T_{r-1}}\hat{\sigma_{t}^{2}}\left(\frac{\bar{N}}{T+1}\right)}{\bar{N}/T_{r-1}}\vee\frac{d\hat{\sigma_{0}^{2}}}{\bar{N}}\right],
\end{align*}
which is exactly the condition for elimination in the modified algorithm
(see item \ref{enu: changes to unknown variance - elimination rule}
in the statement of the proposition). Thus, if \eqref{eq: proof unknown variances - elimination condition on squared distance first}
holds then ${\cal M}_{t}$ is eliminated. Now, using the expression
for $\hat{N}_{r,t}$ in \eqref{eq: proof unknown variances - empirical sample allocation}
we note that \eqref{eq: proof unknown variances - elimination condition on squared distance first}
is equivalent to 
\begin{equation}
Q_{t}^{2}\geq2\cdot\frac{10g(\bar{\delta})}{\nu}\left[\frac{d\cdot\frac{1}{T_{r-1}}\sum_{t=1}^{T_{r-1}}\hat{\sigma_{t}^{2}}\left(\frac{\bar{N}}{T+1}\right)}{\bar{N}/T_{r-1}}\cdot\frac{\sigma_{t}^{2}}{\hat{\sigma_{t}^{2}}\left(\frac{\bar{N}}{T+1}\right)}\vee\frac{d\sigma_{0}^{2}}{\bar{N}}\right].\label{eq: proof unknown variances - elimination condition on squared distance second}
\end{equation}
In turn, the high probability event \eqref{eq: high probability event for noise variances}
also implies that if 
\[
Q_{t}^{2}\geq3\cdot2\cdot\frac{10g(\bar{\delta})}{\nu}\left[\frac{d\cdot\frac{1}{T_{r-1}}\sum_{t=1}^{T_{r-1}}\sigma_{t}^{2}}{\bar{N}/T_{r-1}}\vee\frac{d\sigma_{0}^{2}}{\bar{N}}\right]
\]
then \eqref{eq: proof unknown variances - elimination condition on squared distance second},
and so also \eqref{eq: proof unknown variances - elimination condition on squared distance first},
hold. Evidently, this is the same condition as for Algorithm \ref{alg: CL multiple sources},
where the numerical constant $\nu$ is replaced by $\frac{\nu}{6}$. 
\end{proof}

\subsection{Unknown noise covariance matrices }

Next assume that the covariance matrices $\{\Sigma_{t}\}_{t\in\llbracket T\rrbracket}$
are unknown to the learner, where $\Sigma_{t}:=\sigma_{t}^{2}\cdot\overline{\Sigma}_{t}$.
An efficient estimator is still the empirical mean $\bar{\theta}_{t}(N)$,
and this estimator does not depend on the unknown covariance matrix
$\Sigma_{t}$ (Lemma \ref{lem: high probability MSE for empirical average}).
For Algorithm \ref{alg: CL multiple sources}, the required unknown
parameter is then $\varsigma_{t}^{2}=\frac{1}{d}\Tr[\Sigma_{t}]$
(replacing $\sigma_{t}^{2})$. Since $\Tr[\Sigma_{t}]$ equals to
the sum of its diagonal elements, the estimator \eqref{eq: estimator for sigma square}
proposed for the case $\Sigma_{t}=\sigma_{t}^{2}\cdot I_{d}$ is still
intact. Nonetheless, in that case, the $d$ coordinates are independent
and reduce the variance in estimating $\sigma^{2}$, whereas here
they are not independent. The following lemma modifies Lemma \ref{lem:empirical variance estimation scaling}
to an estimator of $\varsigma^{2}:=\frac{1}{d}\Tr[\Sigma]$. 
\begin{lem}
\label{lem:empirical trace estimation scaling}Assume that $Y=\theta+\epsilon$
where $\theta\in\mathbb{\mathbb{R}}^{d}$ and $\epsilon\sim{\cal N}(0,\Sigma)$
with $\Sigma\in\mathbb{S}_{d}^{++}$. Let $\delta\in(0,1)$ be given.
Consider an estimator for $\varsigma^{2}:=\frac{1}{d}\Tr[\Sigma]$
based on $K$ \IID samples $\{Y(i)\}_{i\in[K]}$ given by 
\begin{equation}
\hat{\varsigma^{2}}(K):=\frac{1}{d(K-1)}\sum_{i=1}^{K}\left\Vert Y(i)-\frac{1}{K}\sum_{i'=1}^{K}Y(i')\right\Vert ^{2}.\label{eq: estimator for sigma square-1}
\end{equation}
Then, with probability at least $1-2e^{-[(K-1)/24-\log d]}$
\begin{equation}
\frac{1}{2}\varsigma^{2}\leq\hat{\varsigma^{2}}(K)\leq\frac{3}{2}\varsigma^{2}.\label{eq: good scaling for the trace}
\end{equation}
\end{lem}
If we make the mild assumption $(K-1)\geq48\log d$ then the probability
guaranteed by Lemma \ref{lem:empirical trace estimation scaling}
is $1-e^{-\Theta(K)}$ compared to the larger probability $1-e^{-\Theta(dK)}$
guaranteed by Lemma \ref{lem:empirical trace estimation scaling}.
\begin{proof}
We follow the proof of Lemma \ref{lem:empirical variance estimation scaling}.
Now, $\sigma^{2}(j)=\Sigma(j,j)$ is the $j$th diagonal coordinate
of $\Sigma$, and $\tilde{\sigma^{2}}(j)$ defined therein (the sample
variance of $\{Y(i;j)\}_{i\in[K]}$) is an unbiased estimator of $\sigma^{2}(j)$
and moreover, it holds that
\[
\frac{(K-1)}{\sigma^{2}(j)}\cdot\tilde{\sigma^{2}(j)}\sim\chi_{K-1}^{2}.
\]
However, unlike the proof of Lemma \ref{lem:empirical variance estimation scaling},
now $\{\tilde{\sigma^{2}}\}_{j\in[d]}$ are not necessarily independent,
and thus their sum is not necessarily distributed as a chi-squared.
To overcome this, we require the estimate of $\tilde{\sigma^{2}(j)}$
will be accurate for all $d$ coordinates. Specifically, using the
same arguments used in the proof of Lemma \ref{lem:empirical variance estimation scaling}
(setting $d=1$ therein), it holds that 
\begin{equation}
\frac{1}{2}\sigma^{2}(j)\leq\tilde{\sigma^{2}(j)}\leq\frac{3}{2}\sigma^{2}(j)\label{eq: good scaling separately for each coordinate}
\end{equation}
with probability at least  $1-2e^{-(K-1)/24}$. By the union bound,
this holds for all $j\in[d]$ with probability at least
\[
1-2de^{-(K-1)/24}=1-2e^{-[(K-1)/24-\log d]}.
\]
The event that \eqref{eq: good scaling separately for each coordinate}
holds for all $j\in[d]$ implies that \eqref{eq: good scaling for the trace}
holds too. 
\end{proof}

\section{Empirical experiments \label{sec:Simulations}}

In this appendix, we run an implementation of Algorithm \ref{alg: CL multiple sources}
on simple simulated estimation settings, and empirically demonstrate
the various properties of this elimination algorithm we have discussed
during the theoretical analysis. Python code for these simulations
can be found in \url{https://anonymous.4open.science/r/CurriculumLearning-CAD5/}.
In what follows we will use $\tilde{Q}_{t}^{2}:=\nicefrac{Q_{t}^{2}}{(\nicefrac{d\sigma_{0}^{2}}{N})}$
to denote normalized square distances. 

In the first experiment, a two task setting $T=2$ is considered.
We fix the parameter of the first source model $t=1$ at distance
$Q_{1}=0$, and sweep over $Q_{2}$, the distance of parameter of
the second source to the target parameter, where we normalize $Q_{2}^{2}$
by the typical MSE of the target task $\nicefrac{d\sigma_{0}^{2}}{N}$.
Figure \ref{fig:simulation 1} shows the loss of the algorithm on
a logarithmic scale, and the empirical probability that the algorithm
chooses task $2$. The outcomes, as depicted in Figure \ref{fig:simulation 1},
reveal two discernible regimes. In the initial regime, approximately
when $\tilde{Q_{2}^{2}}\in(0,100)$, the set ${\cal T}_{\text{alg}}$
is consistently non-empty and the probability to choose task $t=2$
is roughly $0.5$. This implies that, within this regime, Algorithm
\ref{alg: CL multiple sources} correctly identifies that both source
models are better than the target model for estimating the target
parameter, however, the algorithm is unable to consistently eliminate
task $2$. The observed increase in error aligns with the characteristics
of a weak oracle error, as $Q_{2}^{2}\lesssim\nicefrac{d\sigma_{0}^{2}}{N}$.
In the subsequent regime, when $\tilde{Q}_{2}^{2}\in$$(100,300)$,
the probability of selecting task $t=2$ gradually diminishes until
reaching $0$, concurrently with a reduction in the error as $\tilde{Q}_{2}^{2}$
increases. This stems from the fact that $Q_{2}^{2}\gtrsim\nicefrac{d\sigma_{0}^{2}}{N}$
and is the result of the improved ability of Algorithm \ref{alg: CL multiple sources}
to eliminate the source model $t=2$, and then use the better source
model $t=1$ for estimation of the target parameter. 

\begin{figure}
\begin{centering}
\includegraphics[scale=0.8]{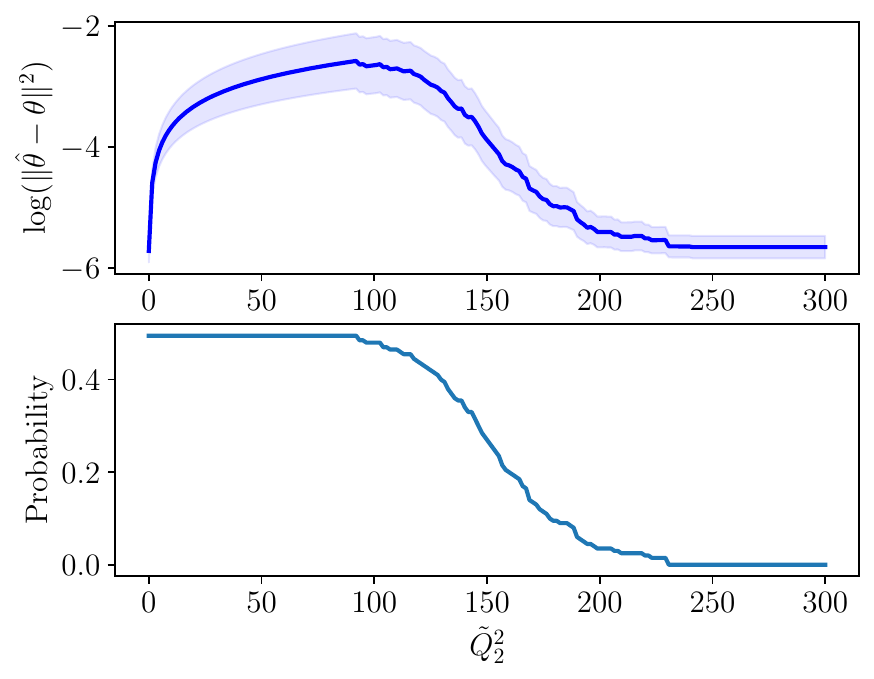}
\par\end{centering}
\caption{\label{fig:simulation 1}The first experiment: Runs of Algorithm \ref{alg: CL multiple sources}
over $200$ repetitions. Parameters are $T=2,\;N=1000,\;d=2,\;\sigma^{2}=1,\;\sigma_{0}^{2}=10$.}
\end{figure}

In the second experiment, we consider settings with large number of
models. We define three types of parameter locations, relative to
the target parameter, designated as ``close'', ``medium'', and
``far''. We consider a sequence of problem settings, in which we
first set one source parameter at each of the location types. Then,
we add more models from one of the types. As can be seen in Figure
\ref{fig:simulation 2}, Algorithm \ref{alg: CL multiple sources}
behaves differently in response to the addition of models in different
location types. The inclusion of models in proximity to the target
task facilitates better estimation error, a phenomenon that aligns
with intuitive expectations. Conversely, in scenarios with a substantial
number of models categorized as ``medium,'' the algorithm struggles
to eliminate those models and thus achieve high MSE. Notably, when
models are introduced in the ``far'' location the algorithm eliminates
them effectively and the addition of models from this category does
not affect the algorithm's loss. The experiment appearing in Section
\ref{sec:A-CL-elimination-algorithm} generalizes this experiment
to other mixtures of types of source models.

\begin{figure}
\begin{centering}
\includegraphics[scale=0.8]{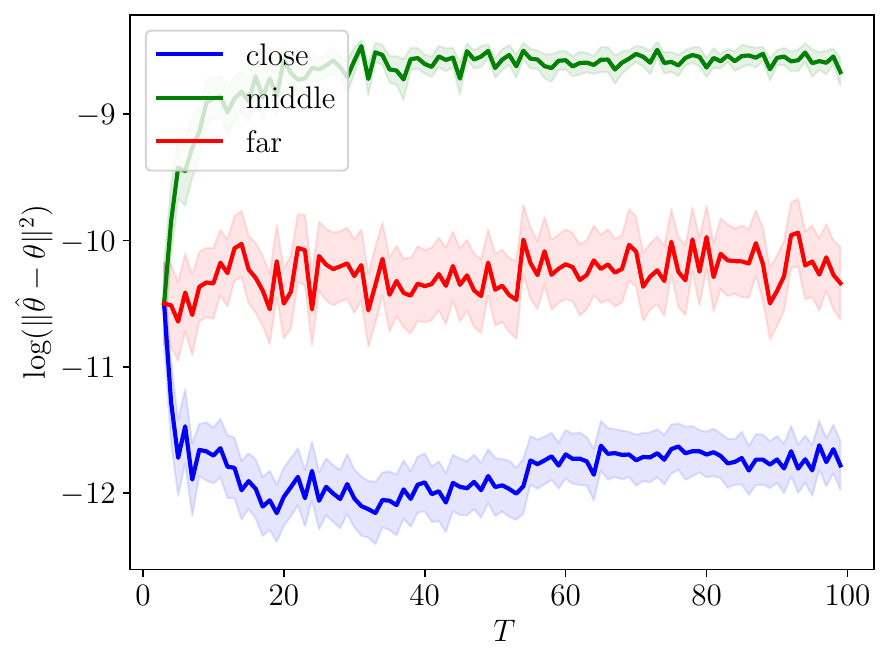}
\par\end{centering}
\caption{\label{fig:simulation 2}The second experiment: Runs of Algorithm
\ref{alg: CL multiple sources} over $200$ repetitions. Parameters
are $N=10^{5},\;d=2,\;\sigma^{2}=0.1,\;\sigma_{0}^{2}=1,\;\tilde{Q}_{\text{close}}^{2}=0,\;\tilde{Q}_{\text{medium}}^{2}=10,\;\tilde{Q}_{\text{far}}^{2}=2\cdot10^{4}$.}
\end{figure}

In the third experiment, we set $T=100$ source models, and choose
\[
\tilde{Q}_{t}^{2}=\begin{cases}
0, & t\in[10]\\
(t-10)^{\gamma}, & t\in\{11,12,\ldots,100\}
\end{cases},
\]
for which various values of $\gamma$ were examined. For each value
we generated the corresponding \textbf{$\beta_{\delta}(\tau)$}, in
order to demonstrate the connection between $\beta_{\delta}(\tau)$
and the algorithm's efficacy in model elimination. To this end, we
record for each run of Algorithm \ref{alg: CL multiple sources} which
models were eliminated and which were not. We evaluate the algorithm
performance using the \emph{precision} score, defined as

\[
\precision:=\frac{\left|{\cal T}_{\text{alg}}\cap{\cal T}_{\text{w.o.}}\right|}{\left|{\cal T}_{\text{alg}}\right|},
\]
and the \emph{recall} score, defined as
\[
\recall:=\frac{\left|{\cal T}_{\text{alg}}\cap{\cal T}_{\text{w.o.}}\right|}{\left|{\cal T}_{\text{w.o.}}\right|}.
\]
 As seen in Figure \ref{fig:simulation 4}, as $\gamma$ is increases,
the location of the fixed point $\beta_{\delta}(\tau)=\tau$ decreases,
which should result a better precision in eliminating models. This
is simply because increasing $\gamma$ also increases the number of
tasks at distant locations, which are easier for the algorithm to
eliminate in the first round, thus allowing for further elimination
in second round, and so on. In all our experiments, Algorithm \ref{alg: CL multiple sources}
has achieved perfect recall, that is, all models in the weak oracle
set are identified as such by the algorithm, and as seen in Figure
\ref{fig:simulation 4}, the precision improves with increasing $\gamma$. 

\begin{figure}
\begin{centering}
\includegraphics[scale=0.6]{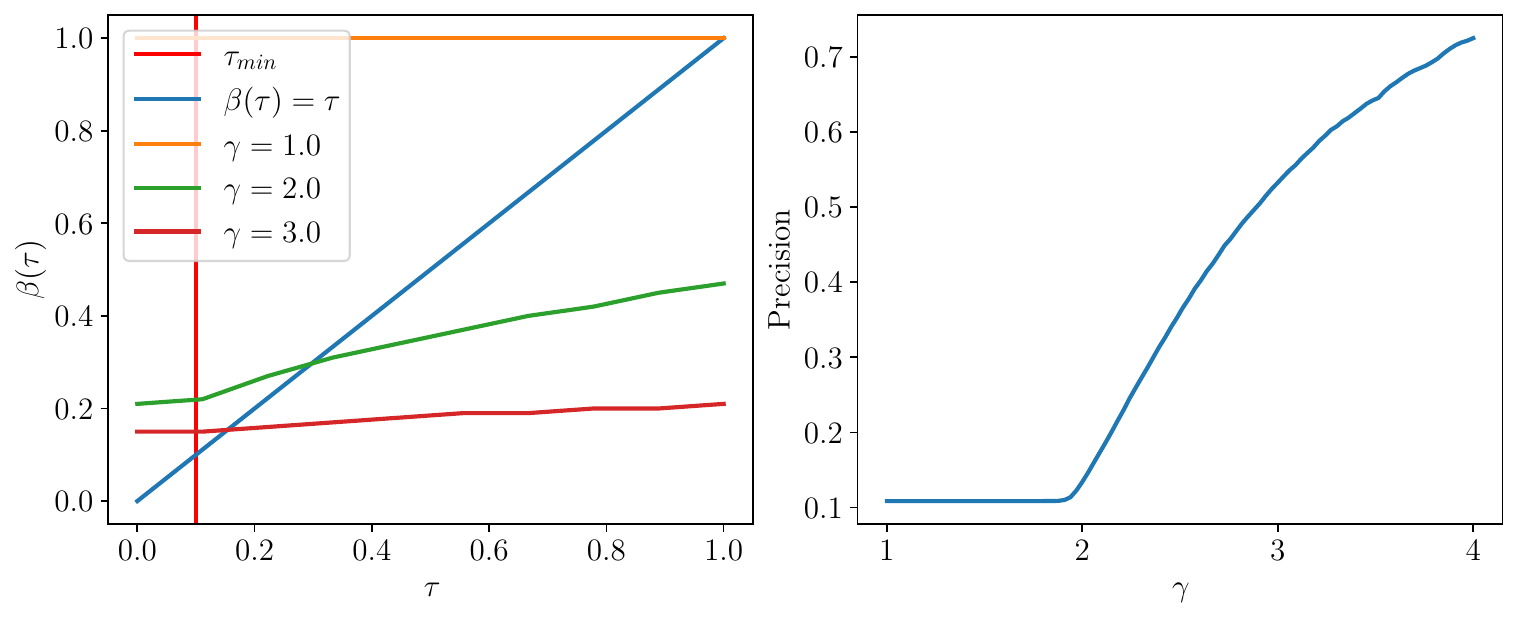}
\par\end{centering}
\caption{\label{fig:simulation 4}The third experiment: Parameters are $N=10^{5},\;d=2,\;\sigma^{2}=0.1,\;\sigma_{0}^{2}=1,\;\delta=0.05$.}
\end{figure}

\section{Minimax lower bounds: Extended discussion and proofs for Section
\ref{sec:Minimax-risk-lower} \label{sec:Proofs-for-Section-Minimax}}

\subsection{A general reduction to hypothesis testing \label{subsec:A-general-reduction}}

For a CL algorithm, the $N$ observed samples are not \IID, because
the learner may select the model from which the next sample is collected
based on previous samples. In accordance, the reduction to hypothesis
testing should be made to a similarly adaptive tester. For completeness,
we next formally define the hypothesis testing problem and then the
reduction (Proposition \ref{prop: variation of reduction from estimation to HT}).
The problem formulation is similar to the general learning CL setting
in Section \ref{sec:setting}, except that the set of possible models
is of finite cardinality $K$, and one of the models is chosen by
Nature with a uniform probability. The learner should then collect
samples in order to identify the model, and its performance is judged
by the error probability. Let $K$ be a positive integer. Let $\boldsymbol{\Theta}_{\text{test}}:=\{\theta^{(j)}\}_{j\in[K]}\subset\boldsymbol{\Psi}\subset(\mathbb{R}^{d})^{T+1}$
be a \emph{testing set} of parameters, each one of them $\theta^{(j)}=(\theta_{0}^{(j)},\theta_{1}^{(j)},\ldots,\theta_{T}^{(j)})$
is comprised of parameters for the $T+1$ models, where the index
$j$ represents the hypothesis. Let $J\sim\text{Uniform}[K]$ be an
\rv that chooses the hypothesis. Given that $J=j$, the samples of
the $T+1$ models are obtained from the model ${\cal M}_{t}$ with
parameter $\theta_{t}^{(j)}$, for all $t\in\llbracket T\rrbracket$.
A CL hypothesis tester collects samples as a general CL learner, described
in Section \ref{sec:setting}. Given the collected samples $(A_{i},S_{i})_{i\in[N]}$,
the tester produces a guess of $J$, given by a testing function $\hat{J}\colon((\llbracket T\rrbracket)\times{\cal Z})^{N}\to[K]$.
A CL testing algorithm is the pair ${\cal A}:=(\pi,\hat{J})$, comprised
of a sampling policy and a testing function. The expected error probability
of the tester when the testing set is $\boldsymbol{\Theta}_{\text{test}}$
and it uses an algorithm ${\cal A}$ is
\[
e_{N}(\boldsymbol{\Phi}_{\text{test}},{\cal A}):=\P_{J,{\cal A}}\left[\hat{J}\left((A_{i},S_{i})_{i\in[N]}\right)\neq J\right],
\]
where $\boldsymbol{\Phi}_{\text{test}}:=\{\phi^{(j)}\}_{j\in[K]}$,
$\phi^{(j)}=(\phi_{0}^{(j)},\phi_{1}^{(j)},\ldots,\phi_{T}^{(j)})$
and $\phi_{t}^{(j)}:=(\theta_{t}^{(j)},\sigma_{t}^{2},I_{d})$. It
is well known that the optimal testing function in this setting is
the \emph{maximum likelihood} test, which satisfies $\hat{J}=j$ only
if 
\[
\P\left[(A_{i},S_{i})_{i\in[N]}\mid j\right]\geq\max_{j'\in[K]}\P\left[(A_{i},S_{i})_{i\in[N]}\mid j'\right].
\]
As is also well known, the estimation problem is more difficult than
hypothesis testing, and this reduction is formulated below:
\begin{prop}[{A variation of \citet[Proposition 15.1]{wainwright2019high}}]
\label{prop: variation of reduction from estimation to HT}Let $\boldsymbol{\Phi}_{\text{test}}$
be given such that $\boldsymbol{\Theta}_{\text{test}}\subset\boldsymbol{\Psi}$,
and let 
\begin{equation}
\eta\equiv\eta(\boldsymbol{\Theta}_{\text{test}}):=\min_{j_{1},j_{2}\in[K]\colon j_{1}\neq j_{2}}\frac{1}{2}\|\theta_{0}^{(j_{1})}-\theta_{0}^{(j_{2})}\|.\label{eq: packing radius}
\end{equation}
Then, for any CL testing algorithm ${\cal A}=(\pi,\hat{\phi})$
\begin{equation}
L_{N,d}(\boldsymbol{\Psi},\boldsymbol{\Gamma})\geq\eta^{2}\cdot\min_{{\cal A}}e_{N}(\boldsymbol{\Phi}_{\text{test}},{\cal A}).\label{eq: lower bounding risk with error probability}
\end{equation}
\end{prop}
\begin{proof}
Let ${\cal A}_{*}=(\pi_{*},\hat{\theta}_{*})$ be the optimal policy
and estimator in the CL setting, and let $\hat{J}$ be an arbitrary
testing function. Then, 
\begin{align*}
L_{N,d}(\boldsymbol{\Psi},\boldsymbol{\Gamma}) & =\sup_{\boldsymbol{\theta}\in\boldsymbol{\Psi}}\E_{\boldsymbol{\phi},{\cal A}_{*}}\left[\left\Vert \hat{\theta}_{*}\left((A_{i},S_{i})_{i\in[N]}\right)-\theta_{0}\right\Vert ^{2}\right]\\
 & \trre[\geq,*]\eta^{2}\cdot e_{N}(\boldsymbol{\Phi}_{\text{test}},(\pi_{*},\hat{J}))\\
 & \geq\eta^{2}\cdot\min_{{\cal A}}e_{N}(\boldsymbol{\Phi}_{\text{test}},{\cal A}),
\end{align*}
where $(*)$ holds since when the policy is fixed to $\pi_{*}$, the
CL learner estimates the parameter, or tests the hypothesis from a
standard statistical model, for which \eqref{eq: lower bounding risk with error probability}
was originally derived in \citep[Proposition 15.1]{wainwright2019high}.
\end{proof}
Utilizing Proposition \ref{prop: variation of reduction from estimation to HT},
the goal is now to derive a lower bound on the error probability of
CL tests. This is, however, challenging, since lower bounds for adaptive
tests are not as widely available, compared to Le-Cam's and Fano's
bounds for non-adaptive testers. We next propose two simple ways to
circumvent the need for bounds on the error probability of adaptive
testers. One idea was used in the proof of \citet[Theorem 1]{xu2022statistical},
where in Appendix \ref{subsec: Xu Theorem 1} we have discussed in
detail its proof, and highlighted a few subtleties related to it. 

First, the CL problem is easier if the learner can collect $N$ samples
from the target task and from each of the source models, and thus
a total of $N\cdot(T+1)$ samples. We state this simple idea as a
general claim. 
\begin{claim}
\label{claim: From adaptive sampling to N samples each}Let $\boldsymbol{Y}:=\{Y_{i,t}\}_{i\in[N],t\in\llbracket T\rrbracket}$
be a collection of $N\cdot(T+1)$ independent \rv's so that $Y_{i,t}$
is drawn from the model ${\cal M}_{t}$ in \eqref{eq: mean estimation model}.
Let $\hat{\vartheta}$ be an estimator for $\theta_{0}$ based on
$\boldsymbol{Y}$. Then
\[
L_{N,d}(\boldsymbol{\Psi},\boldsymbol{\Gamma})\geq\inf_{\hat{\vartheta}}\sup_{\boldsymbol{\theta}\in\boldsymbol{\Psi}}\E_{\boldsymbol{Y}\sim\boldsymbol{\theta}}\left[\left\Vert \hat{\vartheta}\left(\boldsymbol{Y}\right)-\theta_{0}\right\Vert ^{2}\right]:=\tilde{L}_{N,d}(\boldsymbol{\Psi},\boldsymbol{\Gamma}).
\]
\end{claim}
\begin{proof}
The proof is trivial data-processing argument since any adaptive CL
algorithm ${\cal A}=(\pi,\hat{\theta})$ that collects a total of
$N$ samples may first collect the $N(T+1)$ samples of $\boldsymbol{Y}$,
run the sampling policy $\pi$, and only use the $N$ samples collected
by the policy to compute the estimator estimate $\hat{\theta}((A_{i},S_{i})_{i\in[N]})$. 
\end{proof}
This method does not capture the correct dependence of the risk on
$T$, and so is effective whenever the number of source models is
fixed. 

Second, under proper assumptions, it can be argued that an optimal
learner does not sample from some of the models. For example, source
models whose parameter is known to be at large distance from the target
parameter, or that are known to have large noise, may be ignored by
the learner. This is the origin of the claim by \citet[Theorem 1]{xu2022statistical},
that a strong-oracle learner, which is fully informed of $\{Q_{t}^{2}\}_{t\in[T]}$,
may only sample from the model with minimal $Q_{t}^{2}$ (assuming,
for simplicity, that the noise variance of all source models is the
same). In general, it appears to be difficult to rigorously prove
this claim in an estimation context, however, it can be proved for
a test, that is, after the reduction to a finite set of hypotheses
in $\boldsymbol{\Theta}_{\text{test}}$. Concretely, will use the
following two arguments:
\begin{claim}
\label{claim: Optimal tester ignore models}Let $\boldsymbol{\Theta}_{\text{test}}$
be a given testing set. 
\begin{itemize}
\item If $\theta_{t}^{(j)}$ is the same for all $j\in[K]$ then an optimal
policy in the CL hypothesis testing problem does not sample from ${\cal M}_{t}$.
\item If there exists $\alpha\in\mathbb{R}$ so that $\theta_{t_{1}}^{(j)}=\alpha\theta_{t_{2}}^{(j)}$
for some $t_{1},t_{2}\in[T]$ and all $j\in[K]$ and $\nicefrac{\sigma_{t_{2}}^{2}}{\alpha^{2}}\geq\sigma_{t_{1}}^{2}$
then an optimal tester does not sample from ${\cal M}_{t_{2}}$. 
\end{itemize}
\end{claim}
\begin{proof}
In the first case, the observations of ${\cal M}_{t}$ are independent
of $J$. In the second case, the learner can simulate samples from
${\cal M}_{t_{2}}$ by sampling from ${\cal M}_{t_{1}}$, scaling,
and adding independent Gaussian noise with proper variance.
\end{proof}

\subsection{The choice of the localization set and associated challenges \label{subsec:The-choice-of}}

We have discussed in Section \ref{sec:Minimax-risk-lower} the choice
of the set $\boldsymbol{\Psi}$, and showed that choosing it to be
too small may allow the learner to bypass the actual difficulty of
the CL problem, because given that $\boldsymbol{\theta}\in\boldsymbol{\Psi}$
may already reveal too much information to the learner, and allow
it to obtain unrealistically low risk. Thus, the choice of the set
$\boldsymbol{\Psi}$ should be judicious in order to obtain a non-trivial
localized bound. We have provided one example, in which the choice
of the set $\boldsymbol{\Psi}_{=}(\boldsymbol{q})$ with a specific
choice of $\boldsymbol{q}$ allowed the learner to use a majority
rule, which obtains an unrealistic zero risk. This shows that no trivial
lower bound can be derived. Here we continue presenting such examples,
and show that in some cases the correct dependency of the lower bound
on $d$ cannot be obtained, simply because the learner can effectively
reduce the dimension of the problem to a lower effective dimension. 
\begin{example}
\label{exa:learning by dimension reduction}Assume noiseless source
models similarly to the example in Section \ref{sec:Minimax-risk-lower},
but that $T=2$ and that $q_{0}^{2}=q_{1}^{2}=0<q_{2}^{2}\lesssim\nicefrac{d\sigma_{0}^{2}}{N}$.
The loss of the weak oracle in \eqref{eq: weak oracle} is given by
$\|\hat{\theta}_{\text{w.o.}}-\theta_{0}\|^{2}\lesssim q_{2}^{2}$.
Contrary to example in Section \ref{sec:Minimax-risk-lower}, a learner
which collects one sample from each source model cannot decide which
one of them is located at the target parameter and which one of them
is at distance $q_{2}$ from the target parameter. However, let $\tilde{\theta}$
be such that $\|\tilde{\theta}-\theta_{0}\|=q_{2}$. Then, given the
two samples, say $Y_{1},Y_{2}$, the learner can sample directly from
the target model, and project the samples in the direction of $Y_{1}-Y_{2}=\pm(\theta_{0}-\tilde{\theta})$.
This reduces the noise variance to $\frac{\sigma_{0}^{2}}{N}$. So,
if $q_{2}^{2}\gtrsim\nicefrac{\sigma_{0}^{2}}{N}$ then the learner
can identify whether $Y_{1}=\theta_{0}$ or $Y_{2}=\theta_{0}$. The
resulting lower bound is smaller by a factor of $d$ compared to the
desired lower bound. 
\end{example}
The example in Section \ref{sec:Minimax-risk-lower} and Example \ref{exa:learning by dimension reduction}
appear to be specialized to cases in which there exists a source parameter
that exactly equals the target parameter, and the noise variance of
the source models is zero. However, this is not the case, as we next
explain. To this end, recall that the testing set is chosen as a packing
(or separated) set of target parameters, for which the distance between
each pair of parameters is larger than some $q$. Then, if it can
be shown that the error probability in this hypothesis testing problem
is bounded away from zero, this implies a lower bound of $\Theta(q^{2})$
on the minimax risk. If we let $t_{\text{w.o.}}$ be the maximal source
model index in the weak-oracle set, that is, the largest $t$ so that
$Q_{t}^{2}\lesssim\nicefrac{d\sigma_{0}^{2}}{N}$, then ideally we
would like to set $q=q_{\text{w.o.}}$ in order to establish the weak-oracle
learner risk as the lower bound on the minimax risk. In this case,
due to the triangle inequality, source parameters which are at distance
less than $\nicefrac{q}{2}$ from the target parameter are effectively
the same as ones at zero distance, in the sense that if the learner
knows that a source parameter is at distance $Q_{t}\leq\nicefrac{q}{2}$,
it can identify the target parameter in the packing set, and thus
identify the hypothesis with zero error probability. This leads to
a zero lower bound on the estimation loss. This will continue to hold
if the noise variance is non-zero, yet sufficiently small. Thus, the
aspects illuminated in Section \ref{sec:Minimax-risk-lower} and Example
\ref{exa:learning by dimension reduction} depict the actual challenges
of deriving minimax lower bounds by a reduction to a finite set of
hypotheses for the CL problem.  

For the simple case of a one-dimensional parameter $d=1$, the problem
is even more severe, since if the learner knows the distance between
$\theta_{t}$ and $\theta_{0}$ this already reduces the ambiguity
of the learner to just two possibilities. This is demonstrated in
the following example. 
\begin{example}
\label{exa:learning by identification}Assume that $d=1$, and $T=1$.
In this case, the source parameter may be just one of two possibilities
$\theta_{1}=\theta_{0}+sq_{1}$, where $s\in\{\pm1\}$ is an unknown
sign. The learner can actually achieve an error on the order of $\nicefrac{\sigma_{1}^{2}}{N}$
(which is better than the weak oracle) even if $q_{1}^{2}\gtrsim\nicefrac{\sigma_{0}^{2}}{N}\gtrsim\nicefrac{\sigma_{1}^{2}}{N}$.
To achieve this, the CL learner will estimate $\theta_{0}$ (resp.
$\theta_{1}$) via an empirical mean estimator $\overline{\theta}_{0}(N)$
(resp. $\overline{\theta}_{1}(N)$), to obtain an error of $\Theta(\nicefrac{\sigma_{0}}{\sqrt{N}})$
(resp. $\Theta(\nicefrac{\sigma_{1}}{\sqrt{N}})$); both error bounds
hold \WHP. Under the assumption on $q_{1}$, and under the high probability
event, the learner can guess the sign $\hat{s}:=\sgn(\overline{\theta}_{1}(N)-\overline{\theta}_{0}(N))$,
and succeed \WHP. Then, it will estimate the target parameter as
$\overline{\theta}_{1}(N)-\hat{s}\cdot q_{1}$. Such a learner will
achieve a squared error loss of $\Theta(\nicefrac{\sigma_{1}^{2}}{N})$,
which is better than the weak oracle rate given by $\Theta(\nicefrac{\sigma_{0}^{2}}{N})$. 
\end{example}
Evidently, the initial uncertainty of a CL algorithm goes beyond an
unknown sign. One may claim that this is a result of the learner knowing
exactly that $|\theta_{1}-\theta_{0}|=q_{1}$, but the same method
can be used even if $q_{1}$ is only known approximately. 

\subsection{A comparison of Theorem \ref{thm: lower bound global risk constant dimension}
with the lower bound of \citet{xu2022statistical} and its proof \label{subsec:Proof-of-Theorem- first-minimax}}

\paragraph{Comparison to \citet[Theorem 2]{xu2022statistical}}

The lower bound of \citet[Theorem 2]{xu2022statistical} is developed
under the simplifying assumption that $q_{1}=0$ and $q_{t}=\overline{q}$
for $t\in[T]\backslash\{1\}$ and is given by
\begin{equation}
L_{N,d}\left(\boldsymbol{\Psi}_{\leq}(\boldsymbol{q}),\boldsymbol{\Gamma}\right)\gtrsim\frac{\sigma_{0}^{2}\log T}{N}\wedge\overline{q}^{2}+\min_{t\in[T]}\frac{d\sigma_{t}^{2}}{N}.\label{eq: lower bound of Xu Theorem 2}
\end{equation}
The term $\min_{t\in[T]}\nicefrac{d\sigma_{t}^{2}}{N}$ comes from
the strong-oracle lower bound of \citet[Theorem 1]{xu2022statistical},
and so we next focus on the case in which the first term dominates.
Thus we focus the comparison on $\sigma_{t}^{2}=0$ for all $t\in[T]$,
for which the lower bound is $\Omega(\nicefrac{\sigma_{0}^{2}\log T}{N}\wedge\overline{q}^{2})$.
For the case at hand, and ignoring logarithmic terms, if $\nicefrac{\sigma_{0}^{2}}{N}\lesssim\overline{q}^{2}\lesssim\nicefrac{d\sigma_{0}^{2}}{N}$
then the lower bound of Theorem \ref{thm: lower bound global risk constant dimension}
is $\Omega(\overline{q}^{2})$ which matches the weak-oracle learner
risk, and is better than the risk lower bound $\tilde{\Omega}(\nicefrac{\sigma_{0}^{2}}{N})$
of \citet[Theorem 2]{xu2022statistical}. If $\overline{q}^{2}\lesssim\nicefrac{\sigma_{0}^{2}}{N}$
then both bounds are of the same order, and match the weak-oracle
learner risk. Finally, if $\overline{q}^{2}\gtrsim\nicefrac{d\sigma_{0}^{2}}{N}$
then the construction of \citet[Appendix B]{xu2022statistical} is
such that the distance between the distant models parameter and the
target parameter is $q_{t}\asymp\nicefrac{\sigma_{0}}{\sqrt{N}}\lesssim\overline{q}$
for $t\in[T]\backslash\{1\}$. Thus, all models are in ${\cal T}_{\text{w.o.}}$,
and in turn, all bounds are on the order $\tilde{\Omega}(\nicefrac{\sigma_{0}^{2}}{N})$.\footnote{The bound \eqref{eq: lower bound of Xu Theorem 2} of \citet{xu2022statistical}
in case the first term dominates is interpreted as ``\emph{a potential
improvement of a factor of $\nicefrac{d}{\log T}$ when $\overline{q}$
and $\sigma_{t}^{2}$ are small}'' \citep[p. 4, after Theorem 2]{xu2022statistical};
this is due to the optimal risk of $\Theta(\nicefrac{d\sigma_{0}^{2}}{N})$
obtained by sampling only from the target model (when no source models
are available), and then taking the ratio $\nicefrac{(\nicefrac{\sigma_{0}^{2}\log T}{N})}{(\nicefrac{d\sigma_{0}^{2}}{N})}$.
However, more carefully stated, this is the gain only when $\overline{q}^{2}\gtrsim\nicefrac{\sigma_{0}^{2}\log T}{N}\gtrsim\min_{t\in[T]}\nicefrac{d\sigma_{t}^{2}}{N}$.
Otherwise, the ratio between the \RHS of \eqref{eq: lower bound of Xu Theorem 2}
and $\nicefrac{d\sigma_{0}^{2}}{N}$ is different from\emph{ $\nicefrac{d}{\log T}$}.}

\paragraph{Proof outline}

Following Claim \ref{claim: From adaptive sampling to N samples each},
we focus on bounding $\tilde{L}_{N,d}(\boldsymbol{\Psi},\boldsymbol{\Gamma})$
rather than its adaptive-sampling counterpart. We then prove Theorem
\ref{thm: lower bound global risk constant dimension} in a few gradual
steps. For a fixed dimension, there is no need to invoke Fano's method
(as \citet{mousavi2020minimax} did), and the simpler Le-Cam's two-point
method suffices to obtain lower bounds. As a warm-up, we recall this
method for the standard mean estimation case. We then consider the
single source task case $T=1$ (considered by \citet{mousavi2020minimax}),
and provide a self-contained proof, in a form that will allow us to
generalize it to $T>1$. We then extend this proof to $T=2$ source
models, which provides the clearest intuition on why it is the performance
of the weak-oracle learner that is achieved rather than that of the
strong-oracle learner. Then, in order to illuminate the challenge
of extending such an argument to a general $T>2$, we explicitly consider
the case $T=3$. We aim to provide a construction of an hypothesis
testing problem which naturally generalizes the $T=2$ case, but then
explicitly point out the problematic aspect of such an extension.
We then prove Theorem \ref{thm: lower bound global risk constant dimension}
by reducing $T>2$ to the $T=2$ case, which is already solved.

We will repeatedly use the following expression for the KL divergence
between Gaussian \rv's: 
\begin{fact}[{\citep[Exercise 15.13]{wainwright2019high}}]
\label{fact: KL for Gaussians}For $\mu_{1},\mu_{2}\in\mathbb{R}^{d}$
and $\Sigma_{1},\Sigma_{2}\in\mathbb{S}_{++}^{d}$ it holds that
\[
D_{\text{KL}}\left({\cal N}(\mu_{1},\Sigma_{1})\mid\mid{\cal N}(\mu_{2},\Sigma_{2})\right)=\frac{1}{2}\log\frac{\left|\Sigma_{2}\right|}{\left|\Sigma_{1}\right|}-\frac{d}{2}+\frac{1}{2}\Tr\left[\Sigma_{1}\Sigma_{2}^{-1}\right]+\frac{1}{2}\left(\mu_{2}-\mu_{1}\right)^{\top}\Sigma_{2}^{-1}\left(\mu_{2}-\mu_{1}\right).
\]
\end{fact}

\paragraph{Warm-up: Le-Cam's method for the Gaussian location model}

Recall Le-Cam's two-point method: A family of probability measures
$\{P_{\theta}\}_{\theta\in\Theta}$ is given, where $\theta$ is a
parameter in the class $\Theta$. Assume for simplicity that $\Theta\subset\mathbb{R}$
and the squared loss function. Le-Cam's method is based on reducing
the estimation problem of $\theta$ based on a sample from $P_{\theta}$,
to a binary hypothesis problem between a pair $\theta^{(0)},\theta^{(1)}\in\Theta$.
If the optimal test between $P_{\theta^{(0)}}$ and $P_{\theta^{(1)}}$
with equal prior has large average error probability (bounded away
from zero), then the estimation error must be at least on the order
of $(\theta^{(0)}-\theta^{(1)})^{2}$. In turn, the error probability
is characterized using the total variation between $P_{\theta^{(0)}}$
and $P_{\theta^{(1)}}$ (which can later be bounded using the Hellinger
or the KL divergence, which are simpler to analyze). The resulting
bound is given, \eg, in \citep[Eq. (15.14)]{wainwright2019high},
\[
\inf_{\hat{\theta}}\sup_{\theta\in\Theta}\E\left[\|\theta-\hat{\theta}\|^{2}\right]\geq\frac{(\theta^{(0)}-\theta^{(1)})^{2}}{2}\cdot\left[1-\dtv(P_{\theta^{(0)}},P_{\theta^{(1)}})\right],
\]
which can be optimized over the choice of the pair $\theta^{(0)}$
and $\theta^{(1)}$. Using Pinsker's inequality \citep[Eq. (3.59)]{wainwright2019high},
it holds for a Gaussian location model ${\cal N}(\theta,\nicefrac{\sigma^{2}}{N})$
(say $P_{\theta}$ represents $N$ independent measurements from the
same model) that (Fact \ref{fact: KL for Gaussians})
\[
\dtv(P_{\theta^{(0)}},P_{\theta^{(1)}})\leq\sqrt{\frac{1}{2}\Dkl(P_{\theta^{(0)}},P_{\theta^{(1)}})}=\frac{|\theta^{(0)}-\theta^{(1)}|}{2\sigma/\sqrt{N}}.
\]
Thus, choosing $\theta^{(0)}$ and $\theta^{(1)}$ to be any pair
satisfying $|\theta^{(0)}-\theta^{(1)}|=\nicefrac{\sigma}{N}$, implies
that 
\[
\inf_{\hat{\theta}}\sup_{\theta_{0}\in\Theta}\E\left[\|\theta-\hat{\theta}\|^{2}\right]\geq\frac{1}{8}\frac{\sigma^{2}}{N}.
\]

\paragraph{A single source model ($T=1$)}

We next consider the single source model setting, that is $T=1$.
While this model was considered by \citet{mousavi2020minimax}, we
provide here a self-contained proof, which is tailored to the CL setting,
and which will be generalized later on for $T>1$. In the case of
$T=1$, the learner knows that the parameter of the model ${\cal M}_{1}$
is at squared distance at most $q_{1}^{2}$ from the true parameter. 
\begin{prop}
If $\sigma_{0}^{2}\geq\sigma_{1}^{2}$ then
\[
L_{N,1}\left(\boldsymbol{\Psi}_{\leq}(q_{1}),\boldsymbol{\Gamma}\right)\geq\frac{1}{20}\cdot\min\left\{ \frac{\sigma_{0}^{2}}{N},q_{1}^{2}+\frac{\sigma_{1}^{2}}{N}\right\} .
\]
\end{prop}
The proof of the lower bound in the most interesting regime $q_{1}^{2}\lesssim\nicefrac{\sigma_{0}^{2}}{N}$
(the source is useful) is based on the construction of a binary hypothesis
testing between the parameters shown in Figure \ref{fig:binary hypothesis single source}.
In this case, the distances $|\theta_{t}^{(0)}-\theta_{t}^{(1)}|$
are below the noise standard error $\nicefrac{\sigma_{t}}{\sqrt{N}}$
for both $t=0$ (target) and $t=1$ (source). As a result, the error
in the binary hypothesis testing is large, which leads to squared
loss of $|\theta_{0}^{(0)}-\theta_{0}^{(1)}|\asymp q_{1}^{2}+\nicefrac{\sigma_{1}^{2}}{N}$.
The proof in the regime in which the source is useless since $q_{1}^{2}\gtrsim\nicefrac{\sigma_{0}^{2}}{N}$
is based on the monotonicity of $\tilde{L}_{N,d}(\boldsymbol{\Psi}_{\leq}(q_{1}),\boldsymbol{\Gamma})$
in $\boldsymbol{Q}$. 

\begin{figure}
\begin{centering}
\includegraphics{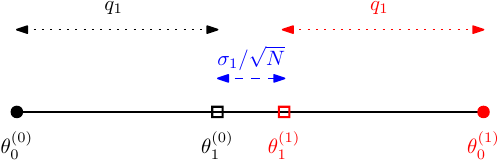}
\par\end{centering}
\caption{The parameter hypotheses for $T=1$. One of the hypotheses is in black
and the other in red. The target parameter is designated by a disc
and the source parameter by a square. \label{fig:binary hypothesis single source}}
\end{figure}

\begin{proof}
Let $V:=q_{1}+\nicefrac{\sigma_{1}}{(4\sqrt{N})}$ for brevity, and
consider the case $V^{2}\leq\nicefrac{\sigma_{0}^{2}}{(4N)}$. We
construct the following binary hypothesis testing problem between
\[
\boldsymbol{\theta}^{(0)}:=(\theta_{0}^{(0)},\theta_{1}^{(0)})=(-V,-V+q_{1})
\]
 and 
\[
\boldsymbol{\theta}^{(1)}:=(\theta_{0}^{(1)},\theta_{1}^{(1)})=(V,V-q_{1})
\]
using $N$ samples from each of the models. Using Fact \ref{fact: KL for Gaussians},
it holds that 
\begin{align*}
\Dkl(P_{\boldsymbol{\theta}^{(0)}},P_{\boldsymbol{\theta}^{(1)}}) & =N\cdot\Dkl(P_{\theta_{0}^{(0)}},P_{\theta_{0}^{(1)}})+N\cdot\Dkl(P_{\theta_{1}^{(0)}},P_{\theta_{1}^{(1)}})\\
 & =N\cdot\frac{2V^{2}}{\sigma_{0}^{2}}+N\frac{2\left(\frac{\sigma_{1}}{4\sqrt{N}}\right)^{2}}{\sigma_{1}^{2}}\\
 & \leq\frac{1}{2}+\frac{1}{8}=\frac{5}{8},
\end{align*}
where the inequality follows from the assumptions. By Pinsker's inequality
\citep[Eq. (3.59)]{wainwright2019high}
\[
\dtv(P_{\boldsymbol{\theta}^{(0)}},P_{\boldsymbol{\theta}^{(1)}})\leq\sqrt{\frac{1}{2}\Dkl(P_{\boldsymbol{\theta}^{(0)}},P_{\boldsymbol{\theta}^{(1)}})}\leq\frac{\sqrt{5}}{4}.
\]
Then, Le-Cam's method implies that
\begin{align}
\inf_{\hat{\vartheta}}\sup_{\boldsymbol{\theta}\in\boldsymbol{\Psi}_{\leq}(q_{1},\boldsymbol{\Gamma})}\E_{\boldsymbol{Y}\sim\boldsymbol{\phi}}\left[\left|\hat{\vartheta}\left(\boldsymbol{Y}\right)-\theta_{0}\right|^{2}\right] & \geq\frac{(\theta_{0}^{(0)}-\theta_{0}^{(1)})^{2}}{2}\cdot\left[1-\dtv(P_{\boldsymbol{\theta}^{(0)}},P_{\boldsymbol{\theta}^{(1)}})\right]\nonumber \\
 & \geq\frac{(\theta^{(0)}-\theta^{(1)})^{2}}{5}\nonumber \\
 & =\frac{4\left(q_{1}+\frac{\sigma_{1}}{4\sqrt{N}}\right)^{2}}{5}.\label{eq: proof lower bound one source one dimension first bound}
\end{align}
The bound \eqref{eq: proof lower bound one source one dimension first bound}
implies that for $V^{2}=\nicefrac{\sigma_{0}^{2}}{(4N)}$ it holds
that 
\[
\inf_{\hat{\vartheta}}\sup_{\boldsymbol{\theta}\in\boldsymbol{\Psi}_{\leq}(q_{1},\boldsymbol{\Gamma})}\E_{\boldsymbol{Y}\sim\boldsymbol{\phi}}\left[\left|\hat{\vartheta}\left(\boldsymbol{Y}\right)-\theta_{0}\right|^{2}\right]\geq\frac{1}{5}\cdot\frac{\sigma_{0}^{2}}{N},
\]
which together with the monotonicity of $\tilde{L}_{N,d}(\boldsymbol{\Psi}_{\leq}(q_{1}),\boldsymbol{\Gamma})$
in $q_{1}$ and \eqref{eq: proof lower bound one source one dimension first bound}
imply the lower bound 
\[
\tilde{L}_{N,d}(\boldsymbol{\Psi}_{\leq}(q_{1}),\boldsymbol{\Gamma})\geq\min\left\{ \frac{4\left(q_{1}+\frac{\sigma_{1}}{4\sqrt{N}}\right)^{2}}{5},\frac{1}{5}\cdot\frac{\sigma_{0}^{2}}{N}\right\} .
\]
This is then further weakened to obtain the stated lower bound.
\end{proof}

\paragraph{Two source models ($T=2$)}

We next consider the setting of two source models, $T=2$. In this
setting, the learner has to decide if either of the two source models
is useful for the estimation of the target parameter. We show that
in this case, the lower bound on the risk matches the risk of the
weak-oracle learner.
\begin{prop}
\label{prop: Lower bound d=00003D1 T=00003D2}If $q_{0}^{2}=0\leq q_{1}^{2}\leq q_{2}^{2}$
and $\sigma_{0}^{2}\geq\sigma_{1}^{2}\geq\sigma_{2}^{2}$ then
\begin{align}
L_{N,1}\left(\boldsymbol{\Psi}_{\leq}(\boldsymbol{q}),\boldsymbol{\Gamma}\right) & \geq\begin{cases}
\frac{1}{24}\cdot\left(q_{2}^{2}+\frac{\sigma_{2}^{2}}{N}\right), & q_{1}^{2}\leq q_{2}^{2}\leq\frac{1}{16}\frac{\sigma_{0}^{2}}{N}\\
\frac{1}{720}\cdot\left(q_{1}^{2}+\frac{\sigma_{1}^{2}}{N}\right), & q_{1}^{2}\leq\frac{1}{16}\frac{\sigma_{0}^{2}}{N}\leq q_{2}^{2}\\
\frac{1}{45}\cdot\frac{\sigma_{0}^{2}}{N}, & q_{2}^{2}\geq q_{1}^{2}\geq\frac{1}{16}\frac{\sigma_{0}^{2}}{N}
\end{cases}.\label{eq: lower bound d=00003D1 T=00003D2}\\
 & \geq\frac{1}{720}\cdot\left\{ \frac{d\sigma_{\text{w.o.}}^{2}}{N}+q_{\text{w.o.}}^{2}\right\} \label{eq: lower bound d=00003D1 T=00003D2 alternative}
\end{align}
The final lower bound matches the risk of the weak-oracle learner.
\end{prop}
The proof of the lower bound in the case $T=2$ is more complicated
than the proof in the single source case $(T=1$), yet it is conceptually
similar. In the most insightful regime, in which $q_{2}^{2}\lesssim\nicefrac{\sigma_{0}^{2}}{N}$,
the proof is based on the construction of a binary hypothesis testing
between the parameters shown in Figure \ref{fig:binary hypothesis two sources}.
Again, the distances $|\theta_{t}^{(0)}-\theta_{t}^{(1)}|$ are below
the noise standard error $\nicefrac{\sigma_{t}}{\sqrt{N}}$ for $t\in\{0,1,2\}$
(target and two sources). As a result, the error in the binary hypothesis
testing is large, which leads to squared loss of $|\theta_{0}^{(0)}-\theta_{0}^{(1)}|\asymp q_{1}^{2}+q_{2}^{2}+\nicefrac{\sigma_{2}^{2}}{N}\asymp q_{2}^{2}+\nicefrac{\sigma_{2}^{2}}{N}$.
The proof in the regime in which the second source is useless but
the first one is useful since $q_{1}^{2}\lesssim\nicefrac{\sigma_{0}^{2}}{N}\lesssim q_{2}^{2}$
is based on a construction similar to the $T=1$ case (\ie, the second
source can effectively be ignored), and the proof in the last regime,
in which no source is useful because $q_{2}^{2}\gtrsim q_{1}^{2}\gtrsim\nicefrac{\sigma_{0}^{2}}{N}$
is based again on the monotonicity of $L_{N,1}(\boldsymbol{\Psi}_{\leq}(\boldsymbol{q}),\boldsymbol{\Gamma})$
in $\boldsymbol{q}$. 

\begin{figure}
\begin{centering}
\includegraphics{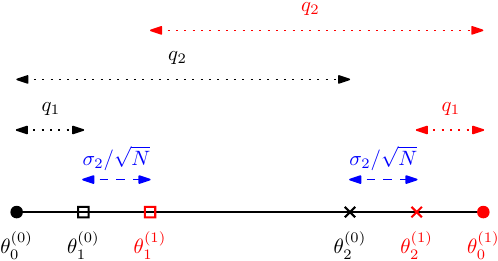}
\par\end{centering}
\caption{The parameter hypotheses for $T=2$. One of the hypotheses is in black
and the other in red. The target parameter is designated by a disc,
the first source parameter by a square, and the second by a cross.
\label{fig:binary hypothesis two sources}}
\end{figure}

\begin{proof}
Denote for brevity $U:=\frac{q_{1}}{2}+\frac{q_{2}}{2}+\nicefrac{\sigma_{2}}{4\sqrt{N}}$.
First consider the case $U^{2}\leq\nicefrac{\sigma_{0}^{2}}{(4N)}$.
We construct the following binary hypothesis testing problem between
\[
\boldsymbol{\theta}^{(0)}:=(\theta_{0}^{(0)},\theta_{1}^{(0)},\theta_{2}^{(0)})=(-U,-U+q_{1},-U+q_{2})
\]
 and 
\[
\boldsymbol{\theta}^{(1)}:=(\theta_{0}^{(1)},\theta_{1}^{(1)},\theta_{2}^{(1)})=(U,U-q_{2},U-q_{1})
\]
using $N$ samples from each of the models. This pair of hypotheses
is constructed so that: (1) Under each of the hypotheses, there is
one source parameter at distance $q_{1}$ from the target parameter
$\theta_{0}$ and one at distance $q_{2}$ (though the source model
with the closer parameter is different for each hypotheses). (2) The
parameters $\theta_{t}^{(0)}$ and $\theta_{t}^{(1)}$ are indistinguishable
using observations from the model ${\cal M}_{t}$. Indeed, using Fact
\ref{fact: KL for Gaussians}, it holds that 
\begin{align*}
\Dkl(P_{\boldsymbol{\theta}^{(0)}},P_{\boldsymbol{\theta}^{(1)}}) & =N\cdot\Dkl(P_{\theta_{0}^{(0)}},P_{\theta_{0}^{(1)}})+N\cdot\Dkl(P_{\theta_{1}^{(0)}},P_{\theta_{1}^{(1)}})+N\cdot\Dkl(P_{\theta_{2}^{(0)}},P_{\theta_{2}^{(1)}})\\
 & =N\cdot\frac{2U^{2}}{\sigma_{0}^{2}}+N\frac{\left(2U-q_{2}-q_{1}\right)^{2}}{2\sigma_{1}^{2}}+N\frac{\left(2U-q_{2}-q_{1}\right)^{2}}{2\sigma_{2}^{2}}\\
 & \leq N\cdot\frac{2U^{2}}{\sigma_{0}^{2}}+N\frac{\left(2U-q_{2}-q_{1}\right)^{2}}{2\sigma_{2}^{2}}+N\frac{\left(2U-q_{2}-q_{1}\right)^{2}}{2\sigma_{2}^{2}}\\
 & \leq\frac{1}{2}+\frac{1}{8}+\frac{1}{8}=\frac{3}{4},
\end{align*}
where the inequality holds under the assumptions $U^{2}\leq\nicefrac{\sigma_{0}^{2}}{(4N)}$
and $\sigma_{1}^{2}\geq\sigma_{2}^{2}$. By Pinsker's inequality \citep[Eq. (3.59)]{wainwright2019high}
\[
\dtv(P_{\boldsymbol{\theta}^{(0)}},P_{\boldsymbol{\theta}^{(1)}})\leq\sqrt{\frac{1}{2}\Dkl(P_{\boldsymbol{\theta}^{(0)}},P_{\boldsymbol{\theta}^{(1)}})}\leq\sqrt{\frac{3}{8}}.
\]
Then, Le-Cam's method implies that
\begin{align}
\tilde{L}_{N,1}\left(\boldsymbol{\Psi}_{\leq}(\boldsymbol{q}),\boldsymbol{\Gamma}\right) & \geq\frac{(\theta_{0}^{(0)}-\theta_{0}^{(1)})^{2}}{2}\cdot\left[1-\dtv(P_{\boldsymbol{\theta}^{(0)}},P_{\boldsymbol{\theta}^{(1)}})\right]\nonumber \\
 & \geq\frac{4U^{2}}{2}\cdot\left[1-\sqrt{\frac{3}{8}}\right]\nonumber \\
 & \geq\frac{30}{45}U^{2}.\label{eq: proof lower bound two sources one dimension first bound}
\end{align}
Next, denote for brevity $V:=q_{1}+\nicefrac{\sigma_{1}}{(8\sqrt{N})}$
and consider the case $U^{2}\geq\nicefrac{\sigma_{0}^{2}}{(4N)}$
yet $V^{2}\leq\nicefrac{\sigma_{0}^{2}}{(4N)}$. Note that these conditions
imply that $V\leq U$ and consequently 
\begin{equation}
\frac{\sigma_{1}}{8\sqrt{N}}\leq q_{2}.\label{eq: lower bound two source intermediate regime condition}
\end{equation}
We construct a pair of hypotheses which are similar to the $T=1$
setting for the target model and the first source model, and have
\emph{identical }parameter for the second source model, to wit
\[
\boldsymbol{\theta}^{(0)}:=(\theta_{0}^{(0)},\theta_{1}^{(0)},\theta_{2}^{(0)})=\frac{1}{3}\cdot(-V,-V+q_{1},0)
\]
 and 
\[
\boldsymbol{\theta}^{(1)}:=(\theta_{0}^{(1)},\theta_{1}^{(1)},\theta_{2}^{(1)})=\frac{1}{3}\cdot(V,V-q_{1},0).
\]
This set of hypotheses is in $\boldsymbol{\Psi}_{\leq}(\boldsymbol{q},\boldsymbol{\Gamma})$
since 
\[
\left|\theta_{0}^{(0)}-\theta_{1}^{(0)}\right|^{2}=\left|\theta_{0}^{(1)}-\theta_{1}^{(1)}\right|^{2}=\frac{q_{1}^{2}}{9}
\]
and
\[
\left|\theta_{0}^{(0)}-\theta_{2}^{(0)}\right|^{2}=\left|\theta_{0}^{(1)}-\theta_{2}^{(1)}\right|^{2}=\frac{1}{9}V^{2}=\frac{1}{9}\left(q_{1}+\frac{\sigma_{1}}{8\sqrt{N}}\right)^{2}\leq q_{2}^{2},
\]
where the inequality holds from the global assumptions $q_{1}\leq q_{2}$
and \eqref{eq: lower bound two source intermediate regime condition}.
Since the second source model $t=2$ clearly does not contribute to
the KL divergence between the two hypotheses, this reduces the problem
to the single source model. By scaling the bound in \eqref{eq: proof lower bound one source one dimension first bound}
by $\nicefrac{1}{3}$, it holds that 
\begin{equation}
\tilde{L}_{N,1}\left(\boldsymbol{\Psi}_{\leq}(\boldsymbol{q}),\boldsymbol{\Gamma}\right)\geq\frac{1}{9}\cdot\frac{4\left(q_{1}+\frac{\sigma_{1}}{8\sqrt{N}}\right)^{2}}{5}.\label{eq: proof lower bound two sources one dimension second bound}
\end{equation}
Next, similarly to the $T=1$ setting, if $V^{2}\geq\nicefrac{\sigma_{0}^{2}}{(4N)}$
then the monotonicity of $\tilde{L}_{N,1}(\boldsymbol{\Psi}_{\leq}(\boldsymbol{q}),\boldsymbol{\Gamma})$
in $\boldsymbol{q}$ (which holds element-wise) and \eqref{eq: proof lower bound two sources one dimension second bound}
imply the lower bound 
\[
\tilde{L}_{N,1}\left(\boldsymbol{\Psi}_{\leq}(\boldsymbol{q}),\boldsymbol{\Gamma}\right)\geq\frac{1}{45}\cdot\frac{\sigma_{0}^{2}}{N}.
\]
From the above three cases we deduce that 
\begin{align}
\tilde{L}_{N,1}\left(\boldsymbol{\Psi}_{\leq}(\boldsymbol{q}),\boldsymbol{\Gamma}\right) & \geq\begin{cases}
\frac{30}{45}U^{2}, & U^{2}\leq\frac{1}{4}\frac{\sigma_{0}^{2}}{N}\\
\frac{4}{45}\cdot V^{2}, & U^{2}\geq\frac{1}{4}\frac{\sigma_{0}^{2}}{N}\geq V^{2}.\\
\frac{1}{45}\cdot\frac{\sigma_{0}^{2}}{N}, & V^{2}\geq\frac{1}{4}\frac{\sigma_{0}^{2}}{N}
\end{cases}\label{eq: proof first full lower bound on minimax risk}
\end{align}
We next show that this lower bound matches the weak oracle rate. Consider
the following cases:
\begin{enumerate}
\item $q_{1}^{2}\leq q_{2}^{2}\leq\nicefrac{\sigma_{0}^{2}}{(16N)}$. It
then holds that 
\[
U^{2}=\left(\frac{q_{1}}{2}+\frac{q_{2}}{2}+\frac{\sigma_{2}}{4\sqrt{N}}\right)^{2}\leq\left(q_{2}+\frac{\sigma_{0}}{4\sqrt{N}}\right)^{2}\leq2q_{2}^{2}+\frac{1}{8}\frac{\sigma_{0}^{2}}{N}\leq\frac{1}{4}\frac{\sigma_{0}^{2}}{N},
\]
and so \eqref{eq: proof first full lower bound on minimax risk} implies
that 
\begin{align*}
\tilde{L}_{N,1}\left(\boldsymbol{\Psi}_{\leq}(\boldsymbol{q}),\boldsymbol{\Gamma}\right) & \geq\frac{30}{45}U^{2}=\frac{30}{45}\left(\frac{q_{1}}{2}+\frac{q_{2}}{2}+\frac{\sigma_{2}}{4\sqrt{N}}\right)^{2}\geq\frac{30}{45}\left(\frac{q_{2}^{2}}{4}+\frac{1}{16}\frac{\sigma_{2}^{2}}{N}\right)\\
 & \geq\frac{1}{24}\cdot\left(q_{2}^{2}+\frac{\sigma_{2}^{2}}{N}\right).
\end{align*}
\item $q_{1}^{2}\leq\nicefrac{\sigma_{0}^{2}}{(16N)}$ and $\nicefrac{\sigma_{0}^{2}}{N}\leq q_{2}^{2}$.
It then holds that 
\[
U^{2}=\left(\frac{q_{1}}{2}+\frac{q_{2}}{2}+\frac{\sigma_{2}}{4\sqrt{N}}\right)^{2}\geq\frac{q_{2}^{2}}{4}\geq\frac{1}{4}\frac{\sigma_{0}^{2}}{N},
\]
and 
\[
V^{2}=\left(q_{1}+\frac{\sigma_{1}}{8\sqrt{N}}\right)^{2}\leq2q_{1}^{2}+\frac{1}{32}\frac{\sigma_{1}^{2}}{N}\leq\frac{1}{4}\frac{\sigma_{0}^{2}}{N}
\]
and so \eqref{eq: proof first full lower bound on minimax risk} implies
that
\[
\tilde{L}_{N,1}\left(\boldsymbol{\Psi}_{\leq}(\boldsymbol{q}),\boldsymbol{\Gamma}\right)\geq\frac{4}{45}\cdot V^{2}=\frac{4}{45}\cdot\left(q_{1}+\frac{\sigma_{1}}{8\sqrt{N}}\right)^{2}\geq\frac{1}{720}\cdot\left(q_{1}^{2}+\frac{\sigma_{1}^{2}}{N}\right).
\]
\item $q_{2}^{2}\geq q_{1}^{2}\geq\nicefrac{\sigma_{0}^{2}}{(4N)}$ . It
then holds that 
\[
V^{2}=\left(q_{1}+\frac{\sigma_{1}}{8\sqrt{N}}\right)^{2}\geq q_{1}^{2}+\frac{1}{64}\frac{\sigma_{1}^{2}}{N}\geq\frac{1}{4}\frac{\sigma_{0}^{2}}{N}
\]
and so \eqref{eq: proof first full lower bound on minimax risk} implies
that 
\[
\tilde{L}_{N,1}\left(\boldsymbol{\Psi}_{\leq}(\boldsymbol{q}),\boldsymbol{\Gamma}\right)\geq\frac{1}{45}\cdot\frac{\sigma_{0}^{2}}{N}.
\]
\end{enumerate}
Combining all the above, we obtain 
\[
\tilde{L}_{N,1}\left(\boldsymbol{\Psi}_{\leq}(\boldsymbol{q}),\boldsymbol{\Gamma}\right)\geq\begin{cases}
\frac{1}{24}\cdot\left(q_{2}^{2}+\frac{\sigma_{2}^{2}}{N}\right), & q_{1}^{2}\leq q_{2}^{2}\leq\frac{1}{16}\frac{\sigma_{0}^{2}}{N}\\
\frac{1}{720}\cdot\left(q_{1}^{2}+\frac{\sigma_{1}^{2}}{N}\right), & q_{1}^{2}\leq\frac{1}{16}\frac{\sigma_{0}^{2}}{N}\leq\frac{\sigma_{0}^{2}}{N}\leq q_{2}^{2}\\
\frac{1}{45}\cdot\frac{\sigma_{0}^{2}}{N}, & q_{2}^{2}\geq q_{1}^{2}\geq\frac{1}{4}\frac{\sigma_{0}^{2}}{N}
\end{cases}.
\]
By the monotonicity property of $\tilde{L}_{N,1}(\boldsymbol{\Psi}_{\leq}(\boldsymbol{q}),\boldsymbol{\Gamma})$
in $\boldsymbol{q}$, the same bound holds for smaller values of $q_{1}$and
$q_{2}$, that is, the bound in \eqref{eq: lower bound d=00003D1 T=00003D2}
holds. The bound in \eqref{eq: lower bound d=00003D1 T=00003D2 alternative}
then directly follows. 
\end{proof}

\paragraph{Multiple source models ($T\protect\geq3$)}

We next move to the proof of Theorem \ref{thm: lower bound global risk constant dimension}
for $T\geq3$. The proof will exploit the monotonicity property of
the semi-local minimax risk, but before providing it, we explain why
the construction used in the $T=2$ fails for $T=3$ and above. To
this end, let us focus on the interesting regime of $q_{3}^{2}\lesssim\nicefrac{\sigma_{0}^{2}}{N}$,
which implies that all $3$ source models are better than the target
model. In the $T=2$ case, each of the two constructed hypotheses
satisfied that one of the source parameters at exactly $q_{1}$ distance
from $\theta_{0}$ and the other at exactly $q_{2}$ distance from
$\theta_{0}$. In Figure \ref{fig:binary hypothesis three sources},
we show a failed attempt for such a construction. In this construction,
sufficiently large distances are first assured for the extremal source
models ${\cal M}_{1}$ and ${\cal M}_{3}$, and any tester of the
two hypotheses based on $N$ samples from each model cannot decide
between $\theta_{1}^{(0)}$ \vs $\theta_{1}^{(1)}$ or between $\theta_{3}^{(0)}$
\vs $\theta_{3}^{(1)}$ \WHP. As in the $T=1$ or $T=2$ settings,
this already assures suitable distance between the parameters of the
target model ${\cal M}_{0}$. The problematic part of this construction
is that after fixing these distances, there are not enough ``degrees-of-freedom''
to adjust the distances between the parameters of the ``middle''
source model ${\cal M}_{2}$. Indeed, referring to Figure \ref{fig:binary hypothesis three sources},
$\xi:=|\theta_{2}^{(0)}-\theta_{2}^{(1)}|\lesssim\nicefrac{\sigma_{2}}{\sqrt{N}}$
will hold only if 
\[
q_{2}+\frac{\sigma_{2}}{\sqrt{N}}\gtrsim q_{3}+\frac{\sigma_{3}}{\sqrt{N}}.
\]
Since $q_{2}\leq q_{3}$ and $\sigma_{2}^{2}\geq\sigma_{3}^{2}$,
this holds only if and only if $\nicefrac{\sigma_{2}}{\sqrt{N}}\gtrsim q_{3}$,
which does not necessarily hold. If $\theta_{2}^{(0)}$ and $\theta_{2}^{(1)}$
are too close to one another, then using the $N$ samples from the
models ${\cal M}_{2}$ may allow the tester to reliably decide on
the hypothesis, and no non-trivial lower bound can be established
by the reduction to this binary hypothesis test. We remark that considering
multidimensional parameters ($d>1$) will not allow to ameliorate
this problem as it stems from constraint on Euclidean distances, which
must hold for any dimension.
\begin{figure}
\begin{centering}
\includegraphics{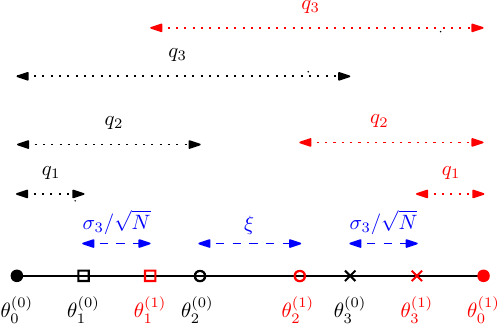}
\par\end{centering}
\caption{A failed attempt for a proper construction of parameter hypotheses
for $T=3$. One of the hypotheses is in black and the other in red.
The target parameter is designated by a disc, the first source parameter
by a square, the second by a circle, and the third by a cross. \label{fig:binary hypothesis three sources}}
\end{figure}

\begin{proof}[of Theorem \ref{thm: lower bound global risk constant dimension}
for $T\geq2$]
The proof is based on a reduction to the $T=2$ case. We first assume
that $d=1$. Consider the following subset of $\boldsymbol{\Psi}_{\leq}(\boldsymbol{q})$
with the additional constraints that

\[
\theta_{1}=\theta_{2}\cdots=\theta_{t_{\text{med}}-1}=:\tilde{\theta}_{1}
\]
and 
\[
\theta_{t_{\text{med}}}=\theta_{t_{\text{med}}+1}=\cdots=\theta_{T}=:\tilde{\theta}_{2},
\]
where $\|\tilde{\theta}_{1}-\theta_{0}\|^{2}\leq q_{1}^{2}:=\tilde{q}_{\zeta(1)}^{2}$
and $\|\tilde{\theta}_{2}-\theta_{0}\|^{2}\leq q_{\text{med}}^{2}:=\tilde{q}_{\zeta(2)}^{2}$.
Given this restricted subset, and since $\sigma_{0}^{2}\geq\sigma_{1}^{2}\geq\cdots\geq\sigma_{T}^{2}$,
a simple data-processing argument (Claim \ref{claim: Optimal tester ignore models})
implies that an optimal CL tester may sample just from the models
${\cal M}_{t_{\text{med}}-1}$ and ${\cal M}_{T}$ . Thus, the problem
is reduced to a $\tilde{T}=2$ source model problem with distances
$\tilde{\boldsymbol{q}}=(\tilde{q}_{1},\tilde{q}_{2})$, where $\tilde{q}_{1}=q_{1}$,
and $\tilde{q}_{2}=q_{\text{med}}$, for which it holds that $\tilde{q}_{1}^{2}\leq\tilde{q}_{2}^{2}\leq\nicefrac{\sigma_{0}^{2}}{(4N)}$,
and noise variances $\tilde{\boldsymbol{\Gamma}}:=(\sigma_{0}^{2},\tilde{\sigma}_{1}^{2},\tilde{\sigma}_{2}^{2})$,
for which it holds that $\tilde{\sigma}_{1}^{2}:=\sigma_{t_{\text{med}}-1}^{2}$
and $\tilde{\sigma}_{2}^{2}:=\sigma_{T}^{2}$. Restricting the set
$\boldsymbol{\Psi}_{\leq}(\boldsymbol{q})$ to such parameters only
reduces the risk, and so 
\[
\tilde{L}_{N,1}(\boldsymbol{\Psi}_{\leq}(\boldsymbol{q}),\boldsymbol{\Gamma})\geq\tilde{L}_{N,1}(\boldsymbol{\Psi}_{\leq}(\tilde{\boldsymbol{q}}),\tilde{\boldsymbol{\Gamma}})\trre[\geq,*]\frac{1}{720}\cdot\left(\tilde{q}_{2}^{2}+\frac{\tilde{\sigma}_{2}^{2}}{N}\right)=\frac{1}{720}\cdot\left(q_{\text{med}}^{2}+\frac{\sigma_{T}^{2}}{N}\right),
\]
where $(*)$ follows from Proposition \ref{prop: Lower bound d=00003D1 T=00003D2}.
This implies the statement of the theorem for $d=1$. For any constant
$d\geq1$, we may assume that learner knows that all parameters belongs
to a subset of $\boldsymbol{\Psi}_{\leq}(\boldsymbol{q})$, for which
the last $d-1$ coordinates are identically zero. An optimal learner
may ignore these coordinates. The only modification required compared
to the $d=1$ case is since ${\cal T}_{\text{w.o.}}(\kappa)$ is defined
as the set of source parameters for which $Q_{t}^{2}\leq\kappa\cdot\nicefrac{d\sigma_{0}^{2}}{N}$,
the value of $\kappa$ should be set so that $\kappa\cdot d=\nicefrac{1}{4}$
in order for the previous derivation to hold. 
\end{proof}
\begin{rem}
\label{rem: improving noise variance dependecny in lower bound}As
discussed, if $q_{t_{\text{w.o.}}+1}\geq\frac{1}{2}(q_{1}+q_{\text{med}}+\nicefrac{\sigma_{\text{w.o.}}}{8\sqrt{N}})$
then the dependency on the noise variance can be improved from $\sigma_{T}^{2}$
to $\sigma_{\text{w.o.}}^{2}$. This is obtained by modifying the
proof as follows. First, we let 
\[
\theta_{t_{\text{med}}}=\theta_{t_{\text{med}}+1}=\cdots=\theta_{t_{\text{\text{w.o.}}}}=:\tilde{\theta}_{2},
\]
and then add another possible point
\[
\theta_{t_{\text{\text{w.o.}}}+1}=\cdots=\theta_{T}=:\tilde{\theta}_{3}.
\]
We then consider a $\tilde{T}=3$ setting, but construct an hypothesis
testing problem between $\boldsymbol{\theta}^{(0)}$ and $\boldsymbol{\theta}^{(1)}$
as in $T=2$, as in Figure \ref{fig:binary hypothesis two sources},
where $\tilde{\theta}_{1}$ and $\tilde{\theta}_{2}$ are in the role
of $\theta_{1}$ and $\theta_{2}$. The third source parameter is
placed at $0$ for both hypotheses, and thus is useless for the distinction
between them. We omit the details for brevity. 
\end{rem}

\subsection{Proof of Theorem \ref{thm: lower local risk high dimension} \label{subsec:Proof-of-Theorem-second-minimax}}

We next prove Theorem \ref{thm: lower local risk high dimension},
which is a tight lower bound for the general $d$-dimensional case
and arbitrary $T$ sources. This requires a packing set of more than
a pair of points (exponential in $d$), and the use of Fano's method.

\paragraph{Proof outline}

The proof of Theorem \ref{thm: lower local risk high dimension} utilizes
Fano's method, which is based on a packing set in $\boldsymbol{\Psi}$
whose number of points is exponential in $d$. The proof begins by
a construction of such a packing set for a single parameter in $\mathbb{R}^{d}$
(Lemma \ref{lem:Construction of packing set}). Points in this packing
set belong to the unit sphere, and in addition, have a zero first
coordinate. The proof then splits between the two different regimes
$\nicefrac{d\sigma_{\text{w.o.}}^{2}}{N}\leq q_{\text{w.o.}}^{2}$
and $q_{\text{w.o.}}^{2}\leq\nicefrac{d\sigma_{\text{w.o.}}^{2}}{N}$.
For the former, Lemma \ref{lem:Construction of testing set noise regime}
uses the single-source packing set mentioned above to construct a
specific testing set $\boldsymbol{\Theta}_{\text{test}}^{*}$ with
appropriate distances between the parameters of the target and source
models, as well as appropriate separation between the models (Lemma
\ref{lem:Construction of testing set noise regime}). This specific
construction is tailored so that the optimal policy $\pi$ is non-adaptive
and uses just one source model, as shown in Lemma \ref{lem: a reduction to a single task lower bound},
which also invokes Fano's method for this construction. In fact, in
this regime, the learner knows in advance that parameters of ${\cal M}_{t}$
are at distance $q_{t}$ and have noise variance $\sigma_{t}^{2}$.
Proposition \ref{prop: Lower bound high dimensions - noise dominating}
then completes the proof of the lower bound for this regime by standard
bounding of the mutual information. The latter regime is then analyzed,
for which it is shown that the nature of the set $\boldsymbol{\Psi}_{*}(q_{\text{w.o.}})$
allows to reduce the problem to the standard case, in which only the
target model is sampled (using Claim \ref{claim: Optimal tester ignore models}).
Fano's method is then similarly invoked to establish the lower bound.

Our construction of a testing set will be based on a proper construction
of a set for a single task. Specifically, in standard single-task
estimation the constructed set is a scaled version of a packing set
of the Euclidean ball \citep[Example 15.14]{wainwright2019high}.
Here, we construct a slightly modified set, which is a packing set
of the Euclidean \emph{sphere, }and in addition, whose first coordinate
is zero. The cardinality of this set is essentially similar to the
standard packing set of the Euclidean ball, besides a worse numerical
constant.
\begin{lem}
\label{lem:Construction of packing set}Assume that $d\geq3$. Then,
there exists a set $\{v^{(j)}\}_{j\in[K]}\subset\mathbb{R}^{d}$ such
that: (1) Unit norm $\|v^{(j)}\|^{2}=1$. (2) The first coordinate
of each $v^{(j)}$ is zero. (3) The points in the set are separated
\[
\|v^{(j_{1})}-v^{(j_{2})}\|\geq e^{-\frac{d+2}{d-2}}:=c(d)\geq e^{-5}
\]
for any $j_{1}\neq j_{2}\in[K]$. (4) Cardinality $K\geq4\cdot e^{d}$. 
\end{lem}
\begin{proof}
Let $\mathbb{B}^{d}\equiv\mathbb{B}^{d}(1)$ be the unit ball in $\mathbb{R}^{d}$.
From \citep[Lemmas 5.5 and 5.7]{wainwright2019high}, for any $\epsilon\in(0,1)$
there exists a packing set $\{\tilde{v}^{(j)}\}_{j\in[K(\epsilon)]}\subset\mathbb{B}^{d-2}$
such that $\|\tilde{v}^{(j_{1})}-\tilde{v}^{(j_{2})}\|^{2}\geq\epsilon$,
and the cardinality satisfies $K(\epsilon)\geq\nicefrac{1}{\epsilon^{d-2}}$
. Choosing $\epsilon=e^{-\nicefrac{(d+2)}{(d-2)}}$, the set $\{\tilde{v}^{(j)}\}_{j\in[K]}\subset\mathbb{R}^{d}$
with $v^{(j)}=(0,\tilde{v}^{(j)},\sqrt{1-\|\tilde{v}^{(j)}\|})$ satisfies
the first three required properties, and its cardinality satisfies
$K\geq e^{\left(\nicefrac{(d+2)}{(d-2)}\right)\cdot(d-2)}\geq4\cdot e^{d}$
for any $d\geq3$. 
\end{proof}
\begin{rem}
The constant $c(d)=e^{-5}$ for $d=3$, however, improves with increasing
dimension, \eg, $c(d)=e^{-\nicefrac{3}{2}}\geq0.22$ if we restrict
to $d\geq10$, whereas for the $d\to\infty$ constant is $e^{-1}\geq0.36$. 
\end{rem}
We next prove Theorem \ref{thm: lower local risk high dimension}
separately in the two different regimes $\nicefrac{d\sigma_{\text{\text{w.o.}}}^{2}}{N}\leq q_{\text{w.o.}}^{2}$
and $q_{\text{w.o.}}^{2}\leq\nicefrac{d\sigma_{\text{\text{w.o.}}}^{2}}{N}$
in Propositions \ref{prop: Lower bound high dimensions - noise dominating}
and \ref{prop: Lower bound high dimensions - Q dominating}. We begin
with the former regime. Consider the following construction of a testing
set $\boldsymbol{\Theta}_{\text{test}}^{*}:=\{\theta^{(*,j)}\}_{j\in[K]}\subset\boldsymbol{\Theta}=\mathbb{R}^{d}$
for $d\geq3$. Let $\{v^{(j)}\}_{j\in[K]}$ be a set which satisfies
the four properties stated in Lemma \ref{lem:Construction of packing set}.
Then $\boldsymbol{\Theta}_{\text{test}}^{*}$ is constructed as follows:
\begin{enumerate}
\item Parameters for the target and ``close'' models: For any $t\in\{0\}\cup[t_{\text{w.o.}}]$
\[
\theta_{t}^{(*,j)}=\left(\sqrt{\frac{d\sigma_{\text{\text{w.o.}}}^{2}}{N}}-q_{t}\right)\cdot v^{(j)}.
\]
\item Parameters for the ``far'' models: For any $t\in[T]\backslash[t_{\text{w.o.}}]$
\[
\theta_{t}^{(*,j)}=\left(\sqrt{q_{t}^{2}-\frac{d\sigma_{\text{\text{w.o.}}}^{2}}{N}},0,\ldots,0\right)
\]
for any $j\in[K]$.
\end{enumerate}
\begin{lem}
\label{lem:Construction of testing set noise regime}Assume that $d\geq3$.
Then, the set $\boldsymbol{\Theta}_{\text{test}}^{*}:=\{\theta^{(*,j)}\}_{j\in[K]}\subset\boldsymbol{\Theta}=\mathbb{R}^{d}$
satisfies:
\begin{enumerate}
\item Distance between the task parameters: 
\begin{equation}
\|\theta_{0}^{(*,j)}-\theta_{t}^{(*,j)}\|=q_{t}^{2}\label{eq: distances between model parameters in testing set}
\end{equation}
for any $t\in[T]$ and any $j\in[K]$.
\item Separation in the target task:
\begin{equation}
\eta_{*}\equiv\eta(\boldsymbol{\Theta}_{\text{test}}^{*})=\|\theta_{0}^{(*,j_{1})}-\theta_{0}^{(*,j_{2})}\|\geq c(d)\cdot\sqrt{\frac{d\sigma_{\text{\text{w.o.}}}^{2}}{N}}\label{eq: separation in the target task}
\end{equation}
for any $j_{1},j_{2}\in[K]$, $j_{1}\neq j_{2}$. 
\end{enumerate}
\end{lem}
\begin{proof}
We first verify the distance between the model parameters. The key
point is to assure that the magnitudes $\sqrt{\nicefrac{d\sigma_{\text{\text{w.o.}}}^{2}}{N}}-q_{t}$
for close models and the magnitudes $\sqrt{q_{t}^{2}-\nicefrac{d\sigma_{\text{\text{w.o.}}}^{2}}{N}}$
used to the define the parameters are all (nonnegative) real numbers.
Indeed, for the close models, $t\in[t_{\text{\text{w.o.}}}]$, it
holds that $q_{t}^{2}\leq q_{t_{\text{w.o.}}}\leq\nicefrac{d\sigma_{\text{\text{w.o.}}}^{2}}{N}$
and then $\sqrt{\nicefrac{d\sigma_{\text{\text{w.o.}}}^{2}}{N}}-q_{t}\geq0$.
For the far models, $t\in[T]\backslash[t_{\text{\text{w.o.}}}]$,
it holds that $q_{t}^{2}\geq\nicefrac{d\sigma_{0}^{2}}{N}\geq\nicefrac{d\sigma_{\text{\text{w.o.}}}^{2}}{N}$
and then $q_{t}^{2}-\nicefrac{d\sigma_{\text{\text{w.o.}}}^{2}}{N}\geq0$.
Thus, in both cases and for all models, the parameter set $\{\theta_{t}^{(*,j)}\}_{j\in[K]}$
is well defined. It is then easy to verify that \eqref{eq: distances between model parameters in testing set}
holds. The separation in the target task follows directly from Lemma
\ref{lem:Construction of packing set} and $q_{0}=0$. 
\end{proof}
The next proposition is crucial, and is based on Claim \ref{claim: Optimal tester ignore models}
invoked for the constructed testing set $\boldsymbol{\Theta}_{\text{test}}^{*}$.
In essence, it shows that the optimal policy $\pi$ of a tester for
a randomly chosen parameter in $\boldsymbol{\Theta}_{\text{test}}^{*}$
is to take all $N$ samples from a single, optimally chosen, model.
This reduces the CL setting to a standard single-task setting, and
enables the utilization of Fano's method \citep[Section 15.3]{wainwright2019high}
to lower bound the error probability, and thus also the risk. It should
be noted that the learner knows in advance that parameters of ${\cal M}_{t}$
are at distance $Q_{t}=q_{t}$ and have noise variance $\sigma_{t}^{2}$. 
\begin{lem}
\label{lem: a reduction to a single task lower bound}An optimal CL
tester for testing a uniform random hypothesis from $\boldsymbol{\Theta}_{\text{test}}^{*}:=\{\theta^{(*,j)}\}_{j\in[K]}$
may collect all its $N$ samples from a single model ${\cal M}_{t^{*}}$,
where $t^{*}\in[t_{\text{\text{w.o.}}}]$. Given that $J=j$, the
$N$ resulting samples are equal in distribution to $N$ \IID samples
from the model 
\begin{equation}
{\cal M}_{*}:Y_{*}=\theta_{0}^{(*,j)}+\epsilon_{*}\label{eq: equivalent model}
\end{equation}
where $\epsilon_{*}\sim{\cal N}(0,\sigma_{*}^{2}\cdot I_{d})$, and
where $\sigma_{*}^{2}\geq\sigma_{\text{\text{w.o.}}}^{2}$ holds.
Furthermore, let $Y_{*}^{\otimes N}$ be $N$ \IID measurements from
\eqref{eq: equivalent model}. Then, 
\begin{equation}
\min_{{\cal A}}e_{N}(\boldsymbol{\Phi}_{\text{test}},{\cal A})\geq1-\frac{I(J;Y_{*}^{\otimes N})+\log2}{\log K}.\label{eq: lower bound on error probability via Fano}
\end{equation}
\end{lem}
\begin{proof}
For far models, $t\in[T]\backslash[t_{\text{\text{w.o.}}}]$, the
parameters $\theta_{t}^{(*,j)}=(\sqrt{q_{t}^{2}-\lambda_{\text{\text{w.o.}}}^{2}},0,\ldots,0)$
are located at the same position for all $j\in[K]$ hypotheses. Thus,
the first case of Claim \ref{claim: Optimal tester ignore models}
implies that an optimal learner will not sample from ${\cal M}_{t}$
for $t\in[T]\backslash[t_{\text{\text{w.o.}}}]$. For close models,
we exploit the property that $\theta_{t_{1}}^{(*,j)}\propto\theta_{t_{2}}^{(*,j)}$
for any $t_{1},t_{2}\in\{0\}\cup[t_{\text{\text{w.o.}}}]$, and so
the second case of Claim \ref{claim: Optimal tester ignore models}
implies that an optimal learner should only sample from one of them.
Specifically, the one with minimal noise variance after scaling. Consider
a model $t\in\{0\}\cup[t_{\text{\text{w.o.}}}]$. Since the learner
knows the norm $\|\theta_{t}^{(*,j)}\|=\lambda_{\text{\text{w.o.}}}-q_{t}$
in advance, it can scale its samples by $\frac{\sqrt{\nicefrac{d\sigma_{\text{\text{w.o.}}}^{2}}{N}}}{\sqrt{\nicefrac{d\sigma_{\text{\text{w.o.}}}^{2}}{N}}-q_{t}}>1$
to obtain an equivalent model 
\[
\tilde{{\cal M}}_{t}:\tilde{Y}_{t}=\theta_{0}^{(*,j)}+\tilde{\epsilon}_{t}
\]
where $\tilde{\theta}_{t}^{(*,j)}=\theta_{0}^{(*,j)}$ and $\tilde{\epsilon}_{t}\sim{\cal N}(0,\tilde{\sigma}_{t}^{2}\cdot I_{d})$
with 
\[
\tilde{\sigma}_{t}^{2}=\left(\frac{\sqrt{\frac{d\sigma_{\text{\text{w.o.}}}^{2}}{N}}}{\sqrt{\frac{d\sigma_{\text{\text{w.o.}}}^{2}}{N}}-q_{t}}\right)^{2}\cdot\sigma_{t}^{2}.
\]
This holds for all models $t\in\{0\}\cup[t_{\text{\text{w.o.}}}]$,
and since the parameters of the different hypotheses are located at
exactly the same positions $\{\theta_{0}^{(*,j)}\}_{j\in[K]}$, the
learner can sample only from the model with the minimal (scaled) noise
variance, that is, from
\[
t^{*}:=\argmin_{t\in\{0\}\cup[t_{\text{\text{w.o.}}}]}\left(\frac{\sqrt{\frac{d\sigma_{\text{\text{w.o.}}}^{2}}{N}}}{\sqrt{\frac{d\sigma_{\text{\text{w.o.}}}^{2}}{N}}-q_{t}}\right)^{2}\cdot\sigma_{t}^{2}.
\]
 We thus let ${\cal M}_{*}=\tilde{{\cal M}}_{t_{*}}$ and $\sigma_{*}^{2}=\tilde{\sigma}_{t^{*}}^{2}$.
It now holds that 
\[
\tilde{\sigma}_{t^{*}}^{2}=\min_{t\in\{0\}\cup[t_{\text{\text{w.o.}}}]}\left(\frac{\sqrt{\frac{d\sigma_{\text{w.o.}}^{2}}{N}}}{\sqrt{\frac{d\sigma_{\text{w.o.}}^{2}}{N}}-q_{t}}\right)^{2}\cdot\sigma_{t}^{2}\geq\sigma_{\text{\text{w.o.}}}^{2},
\]
since $\sqrt{\nicefrac{d\sigma_{\text{\text{w.o.}}}^{2}}{N}}-q_{t}>0$
and due to the noise variance ordering assumption $\sigma_{0}^{2}\geq\sigma_{1}^{2}\geq\sigma_{2}^{2}\geq\cdots\geq\sigma_{T}^{2}$.
Given that an optimal CL tester may collect $N$ samples from ${\cal M}_{*}$,
the problem is now reduced to the standard hypothesis testing setting.
By Fano's inequality \citep[Eq. (15.31)]{wainwright2019high}, the
error probability in testing $J$ based on $N$ \IID samples from
${\cal M}_{*}$ is lower bounded by $1-\nicefrac{(I(J;Y_{*}^{\otimes N})+\log2)}{\log K}$,
from which \eqref{eq: lower bound on error probability via Fano}
directly follows. 
\end{proof}
The rest of the proof in this case is a technical standard bounding
of the mutual information, leading to:
\begin{prop}
\label{prop: Lower bound high dimensions - noise dominating}Consider
the setting of Theorem \ref{thm: lower local risk high dimension}
and further assume that $\nicefrac{d\sigma_{\text{\text{w.o.}}}^{2}}{N}\geq q_{\text{w.o.}}^{2}$.
Then, 
\[
L_{N,d}(\boldsymbol{\Psi}_{=}(\boldsymbol{q}),\boldsymbol{\Gamma})\geq\frac{c^{2}(d)}{2}\cdot\frac{d\sigma_{\text{\text{w.o.}}}^{2}}{N}.
\]
\end{prop}
\begin{proof}
Let $P^{(*,j)}$ be distribution of a sample from the Gaussian location
model ${\cal M}_{*}$ in \eqref{eq: equivalent model}, and let $\overline{P}$
be distribution of $N$ \IID samples from a pure noise model ${\cal N}(0,\sigma_{*}^{2}\cdot I_{d})$.
The information-radius upper bound on the mutual information and the
tensorization property of the KL divergence imply that 
\[
I(J;Y_{*}^{\otimes N})\leq N\cdot\max_{j\in[K]}\Dkl(P^{(*,j)}\mid\mid\overline{P})=N\cdot\max_{j\in[K]}\frac{\|\theta^{(*,j)}\|^{2}}{2\sigma_{*}^{2}},
\]
where the equality follows from Fact \ref{fact: KL for Gaussians}.
The lower bound \eqref{eq: lower bound on error probability via Fano}
then reads 
\begin{align*}
1-\frac{I(J;Y_{*}^{\otimes N})+\log2}{\log K} & \geq1-\frac{N\frac{\|\theta^{(*,j)}\|^{2}}{2\sigma_{*}^{2}}+\log2}{\log K}\\
 & \trre[\geq,a]1-\frac{\frac{d\sigma_{\text{w.o.}}^{2}}{2\sigma_{*}^{2}}+\log2}{\log K}\\
 & \trre[\geq,b]1-\frac{\frac{d}{2}+\log2}{\log K}\\
 & \trre[\geq,c]\frac{1}{2},
\end{align*}
where $(a)$ follows since $\|\theta^{(*,j)}\|^{2}=\nicefrac{d\sigma_{\text{\text{w.o.}}}^{2}}{N}$,
$(b)$ follows since Lemma \ref{lem: a reduction to a single task lower bound}
states that $\sigma_{*}^{2}\geq\sigma_{\text{\text{w.o.}}}^{2}$,
and $(c)$ follows since Lemma \ref{lem:Construction of packing set}
states that $K\geq4\cdot e^{d}$. Combining this with \eqref{eq: lower bound on error probability via Fano}
(Lemma \ref{lem: a reduction to a single task lower bound}) and \eqref{eq: lower bounding risk with error probability}
(Proposition \ref{prop: variation of reduction from estimation to HT})
implies that $L_{N,d}(\boldsymbol{\Psi}_{=}(\boldsymbol{q}),\boldsymbol{\Gamma})\geq\nicefrac{\eta_{*}^{2}}{2}$.
The proof is completed by using separation in the target task bound
\eqref{eq: separation in the target task} (Lemma \ref{lem:Construction of testing set noise regime}).
\end{proof}
\begin{prop}
\label{prop: Lower bound high dimensions - Q dominating}Consider
the setting of Theorem \ref{thm: lower local risk high dimension}
and further assume that $q_{\text{w.o.}}^{2}\geq\nicefrac{d\sigma_{\text{\text{w.o.}}}^{2}}{N}$.
Then, 
\[
L_{N,d}(\boldsymbol{\Psi}_{*}(q_{\text{w.o.}}),\boldsymbol{\Gamma})\geq\frac{c^{2}(d)}{8}\cdot q_{\text{w.o.}}^{2}.
\]
\end{prop}
\begin{proof}
We construct the testing set $\boldsymbol{\Theta}_{\text{test}}^{*}:=\{\theta^{(*,j)}\}_{j\in[K]}$
in which $\theta_{0}^{(*,j)}=q_{\text{w.o.}}\cdot v^{(j)}$ and $\theta_{t}^{(*,j)}=0$
for all $t\in[T]$. Similarly to Lemma \ref{lem:Construction of testing set noise regime}
it holds that $\boldsymbol{\Theta}_{\text{test}}^{*}\subset\boldsymbol{\Psi}_{*}(q_{\text{w.o.}})$
and that 
\[
\|\theta_{0}^{(*,j_{1})}-\theta_{0}^{(*,j_{2})}\|\geq q_{\text{w.o.}}\cdot c(d)
\]
for all $j_{1}\neq j_{2}\in[K]$. By Claim \ref{claim: Optimal tester ignore models},
an optimal learner will collect $N$ samples from ${\cal M}_{0}$.
This reduces the problem to the standard estimation setting. Thus,
Fano's inequality \citep[Eq. (15.31)]{wainwright2019high} implies
that 
\[
\min_{{\cal A}}e_{N}(\boldsymbol{\Theta}_{\text{test}}^{*},{\cal A})\geq1-\frac{I(J;Y_{0}^{\otimes N})+\log2}{\log K},
\]
where $Y_{0}^{\otimes N}$ are $N$ \IID samples from ${\cal M}_{0}$.
Similarly to the proof of Proposition \ref{prop: Lower bound high dimensions - noise dominating},
this is further lower bounded as 
\begin{align*}
1-\frac{I(J;Y_{0}^{\otimes N})+\log2}{\log K} & \geq1-\frac{N\frac{\|\theta^{(*,j)}\|^{2}}{2\sigma_{0}^{2}}+\log2}{\log K}\\
 & =1-\frac{N\frac{q_{\text{w.o.}}^{2}}{2\sigma_{0}^{2}}+\log2}{\log K}\\
 & \trre[\geq,a]1-\frac{\frac{d}{2}+\log2}{\log K}\\
 & \trre[\geq,b]1-\frac{\frac{d}{2}+\log2}{\log K}\\
 & \trre[\geq,c]\frac{1}{2},
\end{align*}
where $(a)$ holds since $q_{\text{w.o.}}^{2}\leq\nicefrac{d\sigma_{0}^{2}}{N}$,
and $(b)$ since $K\geq4\cdot e^{d}$ (Lemma \ref{lem:Construction of packing set}).
The proof then follows from Proposition \ref{prop: variation of reduction from estimation to HT}.
\end{proof}
\bibliographystyle{unsrtnat}
\bibliography{CL}

\end{document}